\documentclass[11pt,a4paper]{article}  

\usepackage{graphicx}
\usepackage{amsmath}
\usepackage{amssymb}
\usepackage{ascmac}
\usepackage{authblk}
\usepackage{bm}
\usepackage{caption}
\usepackage{color}
\usepackage[T1]{fontenc}
\usepackage{hyperref}      
\hypersetup{
  colorlinks=true,
  linkcolor=blue,
  citecolor=blue
}
\usepackage{mathtools}
\usepackage[authoryear, round]{natbib} 
\usepackage{stmaryrd}
\usepackage{tgtermes}

\usepackage{tabularx}  
\usepackage{booktabs} 
\usepackage{float} 

\usepackage[linesnumbered,ruled]{algorithm2e}
\usepackage{amsmath}
\usepackage{amssymb}
\usepackage{amsthm}
\usepackage{ascmac}
\usepackage{bbm}
\usepackage{bm}
\usepackage{booktabs}
\usepackage[font=footnotesize,labelfont=bf]{caption}
\usepackage{chngcntr}
\usepackage{color}
\usepackage{eucal}
\usepackage[T1]{fontenc}
\usepackage{mathtools}
\usepackage{mathrsfs}
\usepackage{multirow}
\usepackage{pifont}
\usepackage{rotating}
\usepackage{tabularx}

\newtheorem{theorem}{Theorem}
\newtheorem*{theorem*}{Theorem}
\newtheorem*{proof*}{proof}
\newtheorem{definition}[theorem]{Definition}
\newtheorem*{definition*}{Definition}

\newtheorem*{assumption*}{Assumption}
\newtheorem{lemma}[theorem]{Lemma}
\newtheorem*{lemma*}{Lemma}
\newtheorem{proposition}[theorem]{Proposition}
\newtheorem*{proposition*}{Proposition}
\newtheorem{corollary}[theorem]{Corollary}
\newtheorem*{corollary*}{Corollary}

\theoremstyle{remark}  
\newtheorem*{remark}{Remark}

\DontPrintSemicolon
\SetArgSty{textnormal}
\SetKwInOut{Input}{Input}
\SetKwInOut{Output}{Output}
\SetKwComment{Comment}{$\triangleright$\ }{}

\renewcommand{\hat}{\widehat}

\def\R{\mathbb{R}}
\def\N{\mathbb{N}}

\def\calA{\mathcal{A}}

\def\calF{\mathcal{F}}
\def\calG{\mathcal{G}}
\def\calH{\mathcal{H}}

\def\calL{\mathcal{L}}

\def\calR{\mathcal{R}}
\def\calS{\mathcal{S}}

\def\calX{\mathcal{X}}
\def\calY{\mathcal{Y}}
\def\calZ{\mathcal{Z}}

\DeclareMathOperator*{\argmax}{arg\,max}
\DeclareMathOperator*{\argmin}{arg\,min}

\newcommand{\ind}[1]{\mathbbm{1}{\left[{#1}\right]}}

\def\x{\bm{x}}

\def\iid{\stackrel{\mathrm{i.i.d.}}{\sim}}

\newcommand{\expect}{\mathbb{E}}

\newcommand{\balpha}{\bm{\alpha}}
\newcommand{\bphi}{\bm{\phi}}
\newcommand{\btheta}{\bm{\theta}}
\newcommand{\bw}{\bm{w}}
\newcommand{\bsigma}{\bm{\sigma}}

\newcommand{\AT}{\mathrm{AT}}
\newcommand{\IT}{\mathrm{IT}}

\newcommand{\LS}{\mathrm{LS}}
\newcommand{\LU}{\mathrm{LU}}
\newcommand{\SEMI}{\mathrm{SEMI}}
\newcommand{\rand}{\mathrm{rand}}

\newcommand{\xyj}{\bm{x}^y_j}
\newcommand{\xym}{\bm{x}^y_m}
\newcommand{\xyjpr}{\bm{x}'^{y}_j}

\newcommand{\xymp}{\bm{x}^y_{m^\prime}}
\newcommand{\xuj}{\bm{x}^\mathrm{U}_j}

\newcommand{\nL}{n_\mathrm{L}}
\newcommand{\nU}{n_\mathrm{U}}

\DeclareMathOperator*{\Var}{Var}
\DeclareMathOperator*{\Cov}{Cov}
\newcommand{\SVrisk}{\calS}  
\newcommand{\empSVrisk}{\hat{\calS}}  

\newcommand{\Rluk}{\calS_{\LU}^{\backslash k}}
\newcommand{\empRluk}{\hat{\calS}_{\LU}^{\backslash k}}
\newcommand{\Rl}{\calS_{\mathrm{L}}}
\newcommand{\empRl}{\hat{\calS}_{\mathrm{L}}}
\newcommand{\Ru}{\calS_{\u}}
\newcommand{\empRu}{\hat{\calS}_{\u}}

\newcommand{\empsslrisk}{\hat{\calS}_{\SEMI\text{-}\gamma}^{\backslash k}}

\newcommand{\RiskSEMI}{\calS_{\SEMI\text{-}\gamma}^{\backslash k}}
\newcommand{\empRiskSEMI}{\hat{\calS}_{\SEMI\text{-}\gamma}^{\backslash k}}

\newcommand{\ghatk}{\hat{g}^{\backslash k}}

\newcommand{\Rad}{\mathfrak{R}}

\newcommand{\empRpsily}{\hat{\calS}_{\mathrm{L}, y}}  
\newcommand{\Rpsily}{\calS_{\mathrm{L}, y}} 

\newcommand{\RLone}{\calS_{\mathrm{L}}^{(1)}}
\newcommand{\RLtwo}{\calS_{\mathrm{L}}^{(2)}}
\newcommand{\RU}{\calS_{\mathrm{U}}}

\newcommand{\empRLone}{\hat{\calS}_{\mathrm{L}}^{(1)}}
\newcommand{\empRLtwo}{\hat{\calS}_{\mathrm{L}}^{(2)}}
\newcommand{\empRU}{\hat{\calS}_{\mathrm{U}}}

\renewcommand{\u}{\mathrm{U}}

\newcommand{\calYrmk}{\mathcal{Y}^{\backslash k}}

\newcolumntype{L}{>{$}l<{$}}
\newcolumntype{C}{>{$}c<{$}}

\usepackage{subfigure}
\newlength{\subfigwidth}
\newlength{\subfigcolsep}
\usepackage[top=30truemm,bottom=30truemm,left=25truemm,right=25truemm]{geometry}

\title{Semi-Supervised Ordinal Regression \\ Based on Empirical Risk Minimization}
\author[1,2]{Taira Tsuchiya\footnote{tsuchiya@sys.i.kyoto-u.ac.jp}\footnote{Currently at Kyoto University}}
\author[1,2]{Nontawat Charoenphakdee}
\author[1]{Issei Sato}
\author[2,1]{Masashi Sugiyama}
\affil[1]{The University of Tokyo}
\affil[2]{RIKEN AIP}
\date{}

\begin{document}
\maketitle

\begin{abstract}
Ordinal regression is aimed at predicting an ordinal class label.
In this paper, we consider its semi-supervised formulation, in which we have unlabeled data along with ordinal-labeled data to train an ordinal regressor.
There are several metrics to evaluate the performance of ordinal regression, such as the mean absolute error, mean zero-one error, and mean squared error.
However, the existing studies do not take the evaluation metric into account, have a restriction on the model choice, and have no theoretical guarantee.
To overcome these problems, we propose a novel generic framework for semi-supervised ordinal regression based on the empirical risk minimization principle that is applicable to optimizing all of the metrics mentioned above.
Besides, our framework has flexible choices of models, surrogate losses, and optimization algorithms without the common geometric assumption on unlabeled data such as the cluster assumption or manifold assumption.
We further provide an estimation error bound to show that our risk estimator is consistent.
Finally, we conduct experiments to show the usefulness of our framework.
\end{abstract}

\section{Introduction}
The goal of ordinal regression is to learn an ordinal regressor to predict a label from a discrete and ordered label set~\citep{chu2005svor, chu2005GP, gutierrez2016ordinal, pedregosa:consistency}.
For example, consider the problem of predicting the diabetes stage of a patient.
The progress of diabetes consists of five stages ranging from mild to severe conditions~\citep{weir2004five}.
The stage number is discrete, and the possible stages are total-ordered.
Ordinal regression has been employed in a variety of fields such as medical research~\citep{med1,med2}, credit rating~\citep{cred1,cred2}, and social sciences~\citep{soc1}.
In the real-world, the labeling process can be costly and time-consuming.
Hence, it is desirable to make use of unlabeled data to improve the prediction accuracy.
Although semi-supervised ordinal regression has great benefits to practical applications,
it has not been extensively explored yet~\citep{Huafu2010towards, liu2011sslor_manifold, seah2012transductive, srijith2013sslor_GP, Ortiz2016dlor}.

The main challenge of semi-supervised learning is how to incorporate unlabeled data to improve the prediction performance~\citep{chapelle2006SSL_MIT}.
To make use of unlabeled data, many assumptions on unlabeled data have been proposed, which often be violated in real-world problems. The representative one is the manifold assumption~\citep{belkin2006manifold}, in which input data is assumed to be distributed in a lower-dimensional manifold.
The manifold assumption has also been utilized in the context of ordinal regression~\citep{liu2011sslor_manifold, Ortiz2016dlor}.
Another popular assumption is the cluster assumption~\citep{seeger2000learning},
in which examples belonging to the same cluster in the input space should have the same label.
\citet{seah2012transductive} proposed a transductive ordinal regression method based on the cluster assumption. However, such a cluster assumption may no be satisfied in practice.
From a different perspective, a semi-supervised ordinal regression method based on Gaussian processes has been proposed~\citep{srijith2013sslor_GP}, which relies on the low-density separation assumption~\citep{chapelle2006SSL_MIT}.
However, the low-density separation assumption may rarely be satisfied in reality, and Gaussian processes have high computation costs and also restrictions on the choice of models.
It is known that the performance of semi-supervised learning methods based on such a geometric assumption is significantly degraded if the assumption on unlabeled data is violated~\citep{li2015towards, sakai-pnu}.
Moreover, all the existing semi-supervised ordinal regression methods mentioned above do not have a theoretical guarantee, do not take the target evaluation metric into account, and have limitations in the choice of models.

Our goal is to establish a novel generic framework that allows flexible choices of models, evaluation metrics, and optimization algorithms and does not require any geometric assumption on unlabeled data.
From the recent theoretical advances of ordinal regression in~\citet{pedregosa:consistency} and semi-supervised binary classification based on positive-unlabeled classification in~\citet{sakai-pnu}, we propose an empirical risk minimization (ERM) based framework that achieves this ambitious goal.
To theoretically justify our framework, we show that our proposed unbiased risk estimator is consistent by establishing an estimation error bound.
In addition, we conduct the analysis of variance reduction to illustrate that unlabeled data can help improve the performance.
Finally, we demonstrate the usefulness of our framework by conducting experiments for the three evaluation metrics using benchmark datasets.

\section{Preliminaries}
\label{sec:preliminaries}

In this section, we formulate the standard supervised ordinal regression problem.
Besides, we discuss its loss function, which is specific to the ordinal regression problem and largely different from the standard loss functions for regression and classification.

\subsection{Supervised Ordinal Regression}
We formulate the risk used in supervised ordinal regression.
Let $\mathcal{X} \subset \R^d$ be a $d$-dimensional input space
and $\mathcal{Y} = \{1, \ldots, K\}$ be an ordered label space, where $K$ is the number of classes.
We assume that labeled data $(\bm{x},y) \in \mathcal{X} \times \mathcal{Y}$ is drawn from the joint probability distribution with density $p$.
In ordinal regression, a function $g: \mathcal{X} \rightarrow \mathcal{Y}$,
which is called a \emph{prediction function}, is used to predict an ordinal label $y$ from input $\x$ as
\begin{align}
    g(\x;f, \btheta)
    &\coloneqq 1 + \sum_{i=1}^{K-1} \ind{f(\x) > \theta_i}
    \label{eq:pred-func}
    ,
\end{align}
where $\btheta = \left[ \theta_1, \ldots, \theta_{K-1} \right]^\top \in \R^{K-1}$ is the \emph{threshold parameters} that should be ordered $\theta_1 \leq \theta_2 \leq \cdots \leq \theta_{K-1}$,
$f: \mathcal{X} \to \R$ is the \emph{decision function},
and $\ind{\cdot}$ is the indicator function which is $1$ if the inner condition holds and otherwise is $0$~\citep{pedregosa:consistency}.
Note that $f$ is a function parameterized by some parameters that is unrelated to $\btheta$,
and that $f$ and $\btheta$ are together the parameters of $g$.
The goal of ordinal regression is to obtain a prediction function
that minimizes the \emph{task risk} defined as
\begin{align}
  \calR(g) \coloneqq \expect_{X,Y} \left[\calL(g(X), Y)\right]
  \label{eq:risk}
  ,
\end{align}
where $\mathbb{E}_{X,Y} [\cdot]$ denotes the expectation over the joint distribution over $\calX \times \calY$, and $\calL: \mathcal{Y} \times \mathcal{Y} \rightarrow \R_{\geq 0}$ is called a \emph{task loss function}.
The task loss function $\calL(g(\x), y)$ measures the performance of a prediction function based on a difference between the output of the prediction function $g(\x) \in \mathcal{Y}$ and the true class label $y \in \mathcal{Y}$.

There are a variety of task losses used in the ordinal regression problem~\citep{pedregosa:consistency}.
One of the most often used task losses is the \emph{absolute loss}, which is 
\begin{align}
    \calL(g(\x), y) = |y - g(\x)|
    .
\end{align}
The task risk~\eqref{eq:risk} with the absolute loss is expressed as
\begin{align}
  \calR(g) = \expect_{X,Y} \left[ |Y-g(X)| \right]
  \label{eq:abs-risk}
  .
\end{align}
With the property of the prediction function, the following proposition holds for the absolute loss.
\begin{proposition}[\citet{pedregosa:consistency}]
The absolute loss is equivalently expressed as
\begin{align}
    |y - g(\x; f, \btheta)|
    = \sum_{i=1}^{y-1} \ind{\alpha_i(\x) \geq 0} + \sum_{i=y}^{K-1} \ind{\alpha_i(\x) < 0}
    \label{eq:original-of-at-loss}
    ,
\end{align}
where
\begin{align}
    \balpha : \mathcal{X} \rightarrow \R^{K-1}, \; \alpha_i(\x) \coloneqq \theta_i - f(\x) \; (i=1,\ldots,K-1)
    ,
\end{align}
and $\alpha_i(\x)$ denotes the $i$-th element of $\balpha(\x)$.
\end{proposition}
Another task loss that has been popularly used is the \emph{zero-one loss}~\citep{chu2005svor}, which is 
\begin{align}
    \calL(g(\x), y) = \ind{ y \neq g(\x)}
    .
\end{align}
Using the definition of the prediction function, the following proposition holds for the zero-one loss. 
\begin{proposition}[\citet{pedregosa:consistency}]
The zero-one loss is equivalently expressed as
\begin{align}
    \ind{ y \neq g(\x)}
    =
    \begin{cases}
        \ind{\alpha_1(\x) < 0}  & (y = 1) ,  \\
        \ind{\alpha_{y-1}(\x) \ge 0} + \ind{\alpha_{y}(\x) < 0}  & (y \in \{2, \dots, K-1 \}) , \\
        \ind{\alpha_{K-1}(\x) \ge 0}  & (y = K)
        .
    \end{cases}
\end{align}
\end{proposition}
The other task loss function used that is frequently used is the \emph{squared loss}~\citep{pedregosa:consistency}, which is defined as
\begin{align}
    \calL(g(\x), y)
    =
    (y - g(\x))^2
    .
\end{align}
The above definitions of each task loss function imply that the task losses in ordinal regression are non-continuous functions, due to the discrete nature of the prediction function.
It is known that the optimization over such a non-continuous function are computationally infeasible~\citep{zeroonenphard1,zeroonenphard2}.
In practice, we do not directly minimize the task risk but relax the task risk to a \emph{surrogate} risk, which will be explained in Section~\ref{sec:surr}.

\subsection{Surrogate Losses for Ordinal Regression}
\label{sec:surr}

\begin{table}[t!]
\begin{center}
\caption{\label{table:notation}
Notations of losses that are used in this paper.
}
\begin{tabular}{cc|cc}
\hline
Task loss   & $\calL(g(\x),y)$ & Task surrogate loss & $\psi(\balpha(\x),y)$ \\
Task risk   & $\calR(g)$              & Task surrogate risk & $\calS(g)$\\
$0\text{-}1$ loss   & $\ind{m<0}$        & Binary surrogate loss & $\ell(z)$\\ \hline
\end{tabular}
\end{center}
\end{table}

We first discuss one of the \emph{task surrogate losses} to the absolute loss called the \emph{all threshold} (AT) loss~\citep{pedregosa:consistency}.
Table~\ref{table:notation} shows the notations used throughout this paper.
The AT loss can be obtained by replacing the $0\text{-}1$ loss $\ind{m<0}$ in~\eqref{eq:original-of-at-loss} with a \emph{binary surrogate loss} $\ell: \mathbb{R} \to \mathbb{R}_{\geq 0}$.
The main motivation to use a surrogate loss instead of the $0\text{-}1$ loss is to make optimization easier.
The binary surrogate loss should be an optimization-friendly function and also satisfy the minimum requirement to be \emph{Fisher-consistent}, as described in~\citet{pedregosa:consistency}.
For example, the well-known squared loss and logistic loss are valid to be employed in the AT losses in ordinal regression.
Specifically, the task surrogate loss $\psi_{\mathrm{AT}} : \R^{K-1} \times \mathcal{Y} \rightarrow \R_{\geq 0}$ based on~\eqref{eq:original-of-at-loss} can be expressed as
\begin{align}
    \psi_\AT (\balpha(\x), y) \coloneqq \sum_{i=1}^{y-1} \ell(-\alpha_i(\x)) + \sum_{i=y}^{K-1} \ell(\alpha_i(\x))
    \label{eq:at-loss}
    .
\end{align}
It is known that the AT loss has good empirical performance compared with other surrogate losses in most cases~\citep{Rennie05lossfunctions}.
In a similar manner, we can define the task surrogate loss to the zero-one loss called the \emph{immediate threshold} (IT) loss as
\begin{align}
    \psi_\IT (\balpha(\x), y) \coloneqq
    \begin{cases}
        \ell(\alpha_1(\x))  & (y = 1)  , \\
        \ell(-\alpha_{y-1}(\x)) + \ell(\alpha_{y}(\x))  & (y \in \{2, \dots, K-1 \})  , \\
        \ell(-\alpha_{K-1}(\x))  & (y = K)
        .
    \end{cases}
    \label{eq:it-loss}
\end{align}
Defining $\alpha_0(\x) = \alpha_K(\x) = 0$ for all $\x\in\calX$, we may write $\psi_\IT (\balpha(\x), y) = \ell(-\alpha_{y-1}(\x)) + \ell(\alpha_{y}(\x))$ for the notational simplicity.
The task surrogate loss to the squared loss called the \emph{least-squares} (LS) loss does not rely on the binary surrogate loss, and is defined as
\begin{align}
    \psi_\LS(\balpha(\x), y)
    \coloneqq
    \left( y + \alpha_1(\x) - \frac32 \right)^2
    .
\end{align}

In general, the task surrogate risk guided by a task surrogate loss $\psi : \R^{K-1} \times \mathcal{Y} \rightarrow \R_{\geq 0}$ can be written as
\begin{align}
  \calS(g) \coloneqq \expect_{X,Y} \left[ \psi(\balpha(X), Y) \right]
  \label{eq:task-sur-risk}
  .
\end{align}
At glance, the notation of~\eqref{eq:task-sur-risk} can be confusing since $g$ does not appear in the right-hand side.
However, $g$ contains $f$ and $\btheta$ as parameters, and $\balpha$ is defined by $f$ and $\btheta$.
Thus, we adopt this notation in this paper.
In supervised ordinal regression,
we are given labeled data drawn independently from the joint density~$p$. 
Then, the risk~\eqref{eq:abs-risk} can be naively minimized by the ERM framework.
Note that the minimization is run over the decision function $f$ and the thresholds $\btheta$ for the AT and IT losses,
while it is run only over $f$ for the LS loss.
Note also that regularization schemes can be applied if necessary.

In the real world, collecting labeled training data can be costly.
Thus, it is preferable to incorporate unlabeled data to train an ordinal regressor.
We can see that direct empirical estimation of the risk term in~\eqref{eq:abs-risk} cannot utilize unlabeled data.
In this paper, we mitigate this problem by extending the ERM framework to semi-supervised ordinal regression.

\section{Proposed Framework}
\label{sec:proposed_framework}

In this section, we introduce our new formulation for semi-supervised ordinal regression.

\noindent \textbf{Unbiased Risk Estimator.}
We first derive a risk estimator within the proposed ERM framework.
Then, we theoretically investigate the behavior of our risk estimator.
Here, we assume that we are given the following data:
\begin{align}
    \calX_\mathrm{L} \coloneqq \{(\x^{\mathrm{L}}_j, y_{j})\}_{j=1}^{\nL} \iid p(X,Y), \quad \calX_\u \coloneqq \{\x^\u_j\}^{\nU}_{j=1} \iid
    p(X) = \sum_{y=1}^{K} \pi_y \, p(X | Y = y)
    \label{eq:joint-assumption}
    ,
\end{align}
where $\nL$ and $\nU$ denote the number of labeled data and unlabeled data, respectively. Let $n_y$ denote the number of labeled data in class $y$, i.e.,  $\nL = \sum_{y=1}^K n_y$, and $\pi_y \coloneqq p(Y=y)$ denote the class-prior probability of the class $y$ such that $\sum_{y=1}^K \pi_y = 1$.

As discussed in Section~\ref{sec:surr}, it is not straightforward to incorporate unlabeled data into the task surrogate risk~\eqref{eq:task-sur-risk}.
To handle this problem, we propose to find an equivalent expression of the task surrogate risk~\eqref{eq:task-sur-risk} so that we can obtain an unbiased risk estimator that uses both unlabeled and labeled data.
Our following lemma states that
we can rewrite the task surrogate risk to contain the expectations over $K-1$ classes of labeled data and the expectation over the marginal density $p(X)$.
\begin{lemma}
    \label{thm:ssl-or-surrogate-risk}
    For any $k \in \{1,\ldots,K\}$, the task surrogate risk~\eqref{eq:task-sur-risk} is equivalently expressed as
    \begin{align}
        \Rluk(g)
        =
        \sum_{y\in\calYrmk} \pi_y \expect_{X|Y=y} \left[ \psi(\balpha(X), y) \right]
        + 
        \expect_\u \left[ \psi(\balpha(X), k) \right]
        -  
        \sum_{y\in\calYrmk} \pi_y \expect_{X|Y=y} \left[ \psi(\balpha(X), k) \right]
        ,
        \label{eq:ssl-or-surrogate-risk}
    \end{align}
    where $\expect_{\mathrm{U}}[\cdot]$ denotes the expectation over unlabeled data, $\calYrmk \coloneqq \calY \backslash \{k\}$,
    and \emph{LU} stands for "\emph{Labeled-Unlabeled}".
\end{lemma}
The proof of Lemma~\ref{thm:ssl-or-surrogate-risk} is given in Appendix~\ref{sec:proof-of-ssl-or-surrogate-risk}.
With the risk obtained from Lemma~\ref{thm:ssl-or-surrogate-risk},
we can derive the following unbiased risk estimator for the task surrogate risk~\eqref{eq:task-sur-risk}:
\begin{align}
    \empRluk(g)
    \coloneqq
    \underbrace{
        \sum_{y\in\calYrmk} \frac{{\pi}_y}{n_y} \sum_{j=1}^{n_y}  \psi(\balpha(\xyj), y)
    }_{\eqqcolon \empRLone(g)}
    +  
    \underbrace{
        \frac{1}{n_\u} \sum_{j=1}^{n_\u} \psi(\balpha(\xuj), k)
    }_{\eqqcolon \empRU(g)}
    - 
    \underbrace{
        \sum_{y\in\calYrmk} \frac{{\pi}_y}{n_y} \sum_{j=1}^{n_y} \psi(\balpha(\xyj), k) 
    }_{\eqqcolon \empRLtwo(g)}.
    \label{eq:ssl-or-surrogate-estimator}
\end{align}
We can interpret the risk estimator in~\eqref{eq:ssl-or-surrogate-estimator} as follows.
The first term on the right-hand side of~\eqref{eq:ssl-or-surrogate-estimator} indicates that labeled data that are not from class $k$ should be predicted correctly.
The second term indicates that unlabeled data should be predicted as $k$.
The purpose of the third term is to cancel the bias from the second term by subtracting the risk of all labeled data that are not from class $k$ to be predicted as class~$k$.
Lemma~\ref{thm:ssl-or-surrogate-risk} is inspired by the technique used in weakly-supervised learning, such as binary classification from positive and unlabeled data and semi-supervised binary classification, which have been shown to be effective~\citep{du2015convex, sakai-pnu}.

Although we can obtain a risk estimator that can utilize unlabeled data,
we still cannot make full use of all given data since our risk estimator in~\eqref{eq:ssl-or-surrogate-estimator} ignores labeled data from class $k$.
To mitigate this problem, inspired by the work on semi-supervised learning for binary classification~\citep{sakai-pnu}, we propose to combine the risk estimator of supervised ordinal regression with our risk estimator in~\eqref{eq:ssl-or-surrogate-estimator} by the convex combination as follows.
\begin{theorem}
  \label{thm:ssl-cvx-surrogate-risk}
  For any $k\in\mathcal{Y}$ and $\gamma \in [0,1]$, the task surrogate risk~\eqref{eq:task-sur-risk} is equivalently expressed as
    \begin{align}
      \RiskSEMI(g) 
      \coloneqq 
      \gamma\Rluk(g)  + (1-\gamma) \calS(g)
      .
      \label{eq:convex-sum}
    \end{align}
\end{theorem}
Theorem~\ref{thm:ssl-cvx-surrogate-risk} directly follows from Lemma~\ref{thm:ssl-or-surrogate-risk}.
It is worth noting that our risk~\eqref{eq:convex-sum} is equivalent to the ordinary surrogate risk in~\eqref{eq:task-sur-risk}. Therefore, the theory of the Fisher-consistency of surrogate losses and excess risk bounds in the ordinary ordinal regression~\citep{pedregosa:consistency} are directly applicable to our framework, which will be discussed in the next section.

\noindent \textbf{Roles of Unlabeled Data.} 
The proposed risk does not use geometric information of unlabeled data; thus, it is applicable even when a specific geometric assumption does not hold.
Then, an important question to be answered is, "\emph{how does unlabeled data help us obtain a good ordinal regressor?}"
In the training phase, by introducing unlabeled data, we can expect that the variance of the empirical risk is decreased, resulting in more accurate risk estimation.
Also, in the validation phase, the variance of the validation risk estimator is reduced, helping us to select good hyper-parameters.

The removed class $k$ can be determined arbitrarily as shown in Lemma~\ref{thm:ssl-or-surrogate-risk}, and the performance can be depending on the choice of $k$.
We will investigate strategies to select the class $k$ in Section~\ref{subsec:what-class-to-pick-out} based on the theory discussed in the next section.

\section{Theoretical Analysis}
\label{sec:theory}

This section establishes the Fisher-consistency and estimation error bounds to elucidate that our risk estimator $\empRiskSEMI(g)$ has theoretical guarantees.
Besides, we provide a theoretical analysis of variance reduction, which shows that the variance of the empirical semi-supervised risk can be smaller than that of the empirical supervised risk.

\subsection{Fisher-Consistency}
To ensure that the minimizer of our task surrogate risk $\RiskSEMI$ can give the optimal solution to the task risk~\eqref{eq:risk}, we give the following proposition.

\begin{proposition}[Fisher-Consistency]
\label{thm:consistency}
    Let the task surrogate loss $\psi$ be the AT loss, IT loss, or LS loss.
    Assume for the AT and IT losses that the binary surrogate loss $\ell$ is convex, differentiable at 0, and $\ell'(0) < 0$.
    Then, every minimizer $g$ of $\RiskSEMI$ reaches
    Bayes optimal risk $\inf \{ \calR(g) : \text{ all measurable functions } g \}$.
\end{proposition}
The proof of Proposition~\ref{thm:consistency} is given in Appendix~\ref{sec:proof-of-consistency}.
Although we avoid going into the details, Proposition~\ref{thm:consistency} also holds for the cumulative link and least absolute deviation losses, 
we can prove Fisher-consistency for them in a similar manner (see~\citet{pedregosa:consistency} for the definitions of theses losses).
With Proposition~\ref{thm:consistency}, we can clarify that the following estimation error bound is applicable to a wide range of task surrogate losses.

\subsection{Estimation Error Bound}
To ensure that our empirical risk estimator $\empRiskSEMI(g)$ is consistent to the risk $\RiskSEMI(g)$ and the empirical risk minimizer approaches the true risk minimizer, we establish estimation error bounds here.  
Let $\mathcal{F} \subset \R^\mathcal{X}$ be a class of specified decision functions. Then, we will consider a distribution-dependent complexity measure of functions, called the \emph{expected Rademacher complexity}~\citep{bartlett2002rademacher}.
\begin{definition}[Expected Rademacher complexity]
    Let $n$ be a positive integer,
    $Z_1,\dots,Z_n$ be i.i.d. random variables drawn from a probability distribution with density $p$ over a set $\mathcal{Z}$,
    $\mathcal{F} \subset \R^\mathcal{Z}$ be a family of functions from $\mathcal{Z}$ to $\R$,
    and $\bm{\sigma}=(\sigma_1,\dots,\sigma_n)$ be random variables, which take $+1$ and $-1$ with equal probabilities.
    Then the expected Rademacher complexity of $\mathcal{F}$ is defined as
    \begin{align*}
        \mathfrak{R}(\mathcal{F};n, p)
        \coloneqq
        \expect_{Z_1,\dots,Z_n} \expect_{\bm{\sigma}} \left[ \sup_{f \in \mathcal{F}} \frac{1}{n}\sum_{i=1}^n \sigma_i f(Z_i) \right]
        .
    \end{align*}
\end{definition}
Intuitively, the Rademacher complexity quantifies how much our decision function class $\mathcal{F}$ can correlate to the random noise.
Thus, a large Rademacher complexity indicates that a decision function is highly flexible to fit the noise. This complexity term is an important tool to derive an estimation error bound (more details about the measures of the complexity of a hypothesis class can be found in~\citet{bartlett2002rademacher, shalev2014understanding, mohri2018foundations}).

To formally define the Rademacher complexity in the context of ordinal regression, we introduce the following definitions.
\begin{align*}
    \Theta &\coloneqq \{ \btheta \in \R^{K-1} : \theta_1 \le \theta_2 \le \dots \le \theta_{K-1},\, \|\btheta\|_2 \le C_{\btheta}\},  \\
    \calG &\coloneqq \{(f, \btheta) : f\in\calF,\, \btheta\in\Theta \},  \\
    \calA_i &\coloneqq \{\alpha_i(\x) = \theta_i - f(\x) : f\in\calF,\, \theta\in\Theta \} \text{ for }  i \in \{0, 1,\dots,K\}.
\end{align*}
Note that $\Rad(\calA_0) = \Rad(\calA_K) = 0$ from the extended definition of $\balpha$.
Let
\begin{align}
    g^* \coloneqq \argmin_{f \in \mathcal{F}, \btheta\in\Theta} \RiskSEMI(g) (= \argmin_{f \in \mathcal{F}, \btheta\in\Theta}  \calS(g))
\end{align}
be the true risk minimizer,
and
\begin{align}
    \hat{g}^{\backslash k} \coloneqq \argmin_{f \in \mathcal{F}, \btheta\in\Theta} \empRiskSEMI (\not\equiv \argmin_{f \in \mathcal{F}, \btheta\in\Theta}  \hat{\calS}_\psi(g))
\end{align}
be the empirical risk minimizer.
Note that $g^*, \ghatk \in \calG$ are pairs of some $f\in\calF$ and $\btheta\in\Theta$, respectively.
We also define the hypothesis classes guided by each task surrogate loss as
\begin{align}
\begin{aligned}
    \calH_\AT &\coloneqq \{ (\x, y) \mapsto \psi_\AT(\balpha(\x), y) : f\in\calF, \btheta\in\Theta \},   \\
    \calH_\IT &\coloneqq \{ (\x, y) \mapsto \psi_\IT(\balpha(\x), y) : f\in\calF, \btheta\in\Theta \},   \\
    \calH_\LS &\coloneqq \{ (\x, y) \mapsto \psi_\LS(\balpha(\x), y) : f\in\calF \}.  
\end{aligned}
\end{align}
In a similar manner, we define the the hypothesis classes guided by each task surrogate loss for fixed label $y\in\calY$ as
\begin{align}
\begin{aligned}
    \calH_\AT^{(y)} &\coloneqq \{ \x \mapsto \psi_\AT(\balpha(\x), y) : f\in\calF, \btheta\in\Theta \},   \\
    \calH_\IT^{(y)} &\coloneqq \{ \x \mapsto \psi_\IT(\balpha(\x), y) : f\in\calF, \btheta\in\Theta \},   \\
    \calH_\LS^{(y)} &\coloneqq \{ \x \mapsto \psi_\LS(\balpha(\x), y) : f\in\calF \}.  
\end{aligned}
\end{align}
We use $\calH$ (reps. $\calH^{(y)}$) instead of $\calH_\AT, \calH_\IT$, and $\calH_\LS$ (reps. $\calH_{\AT}^{(y)}, \calH_{\IT}^{(y)}$, and $\calH_{\LS}^{(y)}$), when a statement holds without depending on the used task surrogate loss.
At a glance, some definitions may look redundant, but they are needed to investigate properties specific to each of the AT, IT, and LS losses.

To the best of our knowledge, the Rademacher complexity in the context of the ordinal regression problem has never been investigated.
The next two theorems establish upper bounds for the guided hypothesis classes.

\begin{theorem}
\label{thm:rad-bound-OR-general-y-dependent}
    Fix $y\in\calY$. Let $n \in \N$ and assume for the AT and IT losses that the binary surrogate loss $\ell$ is $\rho$-Lipschitz.
    Then, the expected Rademacher complexities of  $\calH_\AT^{(y)}, \calH_\IT^{(y)}$, and $\calH_\LS^{(y)}$ are bounded as
    \begin{align}
        \begin{aligned}
        \Rad(\calH_\AT^{(y)} ; n) 
        &\le 
        \rho \sum_{j=1}^{K-1} \Rad(\calA_j; n)   
        ,  \\
        \Rad(\calH_\IT^{(y)} ; n) 
        &\le 
        \rho \left( \Rad(\calA_{y-1}; n) + \Rad(\calA_{y}; n) \right)
        , \\
        \Rad(\calH_\LS^{(y)} ; n) 
        &\le 
        2 |y + \theta_1 - 3/2 | \Rad(\calF; n) + \Rad(\mathrm{sq} \circ \calF; n)
        ,
        \end{aligned}
    \end{align}
    where $\mathrm{sq} : \mathrm{Im}\,f \ni z \mapsto z^2 \in \R$.
\end{theorem}

\begin{theorem}[Upper bounds of Rademacher complexities of task surrogate losses]
\label{thm:rad-bound-OR-general}
    Let $n \in \N$ and
    assume for the AT and IT losses that the binary surrogate loss $\ell$ is $\rho$-Lipschitz.
    Then, the expected Rademacher complexities of $\calH_\AT, \calH_\IT$, and $\calH_\LS$ are bounded as
    \begin{align}
        \begin{aligned}
        \Rad(\calH_\AT ; n) 
        &\le 
        \rho \, K \sum_{j=1}^{K-1} \Rad(\calA_j; n)   
        ,  
        \\
        \Rad(\calH_\IT ; n) 
        &\le 
        \rho \, \sum_{y\in\calY} 
        \left( \Rad(\calA_{y-1}; n) + \Rad(\calA_{y}; n) \right), 
        \\
        \Rad(\calH_\LS ; n) 
        &\le 
        2 (K + |\theta_1 - 3/2|) \Rad(\calF; n) + \Rad(\mathrm{sq} \circ \calF; n)
        ,
        \end{aligned}
    \end{align}
    where $\mathrm{sq} : \mathrm{Im}\,f \ni z \mapsto z^2 \in \R$.
\end{theorem}
\begin{remark}
    The bound on $\Rad(\calH_\LS ; n)$ is given using $\Rad(\calF ; n)$ instead of $\Rad(\calA_j ; n) $.
    This is because, in the LS loss, the threshold is fixed.
\end{remark}
The proofs of Theorems~\ref{thm:rad-bound-OR-general-y-dependent} and~\ref{thm:rad-bound-OR-general} are given in Appendix~\ref{sec:proof-of-Rademacher-OR-linear}.
Both theorems state that the Rademacher complexities with the guided hypothesis classes can be bounded by the Rademacher complexity of the underlying decision functions class $\calF$,
and show how $\Rad(\calH; n_y)$ of each task surrogate loss depends on the problem-dependent parameters.
Then, the next theorem establishes estimation error bounds with a class of general decision functions based on Theorems~\ref{thm:rad-bound-OR-general-y-dependent} and~\ref{thm:rad-bound-OR-general}.
\begin{theorem}[Estimation error bounds with general decision functions]
\label{thm:estimation-error-bound-general}
    Assume for the AT and IT losses that the binary surrogate loss $\ell$ is $\rho$-Lipschitz,
    and that there exists a constant $C_\psi > 0$ such that $\psi(y, \balpha) \leq C_\psi$ for any $y \in \calY,\ \balpha \in \R^{K-1}$.
    Then, for any $k \in \mathcal{Y}$ and $\delta \in (0, 1]$, with probability at least $1 - \delta$,
    \begin{align}
        &
        \calS(\hat{g}^{\backslash k}) - \calS(g^*)   \nonumber \\
        &\le
        4 \left(
            \gamma
            \left(
                \sum_{y\in\calYrmk} \pi_y \left( \Rad_{\mathrm{UB}}^{(y)}(n_y) + \Rad_{\mathrm{UB}}^{(k)}(n_y) \right)
                +
                \Rad_{\mathrm{UB}}^{(k)}(n_\mathrm{U})
            \right)
            +
            (1-\gamma) \Rad_{\mathrm{UB}}(n_\mathrm{L}) 
        \right)   \nonumber \\
        &\qquad+
        2\sqrt{2} C_\psi \left( \gamma \left(2\sum_{y\in\calYrmk} \frac{\pi_y}{\sqrt{n_y}} + \frac{1}{\sqrt{n_\u}} \right) + (1-\gamma) \frac{1}{\sqrt{n_{\mathrm{L}}}} \right) \sqrt{\ln\frac{2(K+1)}{\delta}}
        ,
        \label{eq:estimation-error-bound-general}
    \end{align}
    where $\Rad_{\mathrm{UB}}^{(y)}(n)$ for $y\in\calY$ and $\Rad_{\mathrm{UB}}(n)$ are defined as
    \begin{align*}
        \Rad_{\mathrm{UB}}^{(y)}(n)
        &\coloneqq
        \begin{cases}
            \rho \sum_{j=1}^{K-1} \Rad(\calA_j; n)    & (\text{AT loss}) , \\
            \rho \, \left( \Rad(\calA_{y-1}; n) + \Rad(\calA_{y}; n) \right)    & (\text{IT loss}) , \\
            2 \left|y + \theta_1 - 3/2 \right| \Rad(\calF; n) + \Rad(\mathrm{sq} \circ \calF; n) & (\text{LS loss})
            ,
        \end{cases}  
        \\
        \Rad_{\mathrm{UB}}(n) 
        &\coloneqq 
        \begin{cases}
            \rho K \sum_{j=1}^{K-1} \Rad(\calA_j; n)    & (\text{AT loss}) , \\
            \rho \, \sum_{y\in\calY} \left( \Rad(\calA_{y-1}; n) + \Rad(\calA_{y}; n) \right) & (\text{IT loss}) , \\
            2 (K + |\theta_1 - 3/2|) \Rad(\calF; n) + \Rad(\mathrm{sq} \circ \calF; n) & (\text{LS loss})
            .
        \end{cases}
    \end{align*}
\end{theorem}
The proof of Theorem~\ref{thm:estimation-error-bound-general} is given in Appendix~\ref{sec:proof-of-estimation-error-bound}.
Note that the Rademacher complexity with a class of underlying decision functions $\calF$ can be bounded in many decision function classes as discussed in~\citet{niu2016,baosu,mohri2018foundations},
and will be given for a linear-in-parameter models (Lemma~\ref{lem:rad-bound-linear} in Appendix~\ref{sec:proof-of-Rademacher-OR-linear}).

Next, we restrict the decision function class $\mathcal{F}$ to a class of \emph{linear-in-parameter} models and see more in detail the behavior of the Rademacher complexities with hypothesis classes guided by each task surrogate loss.
The class of linear-in-parameter models is defined as
\begin{align}
    \mathcal{F} = \{f(\x) = \bw^\top\bm{\phi}(\x) \colon \|\bw\| \le C_{\bw},\; \|\bm{\phi}(\x)\|_2 \le C_{\bm{\phi}} \}
\end{align}
for positive constants $C_{\bm{w}}$ and $C_{\bm{\phi}}$,
where $\bw \in \R^b$ is a parameter and $\bphi: \R^{d+1} \rightarrow \R^b$ is a basis function.
We assume that the bias parameter is included in $\bw$.
We can prove the following lemma using the theorem on the Rademacher complexity bound for ordinal regression (Theorem~\ref{thm:rad-bound-OR-general})
and a well-known upper bound for the Rademacher complexity with linear-in-parameter models (Lemma~\ref{lem:rad-bound-linear} in Appendix~\ref{sec:proof-of-Rademacher-OR-linear}).
\begin{lemma}
\label{lem:OR-rad-bound-linear}
    Let $n \in \N$ and
    assume for the AT and IT losses that the binary surrogate loss $\ell$ is $\rho$-Lipschitz.
    Then, the expected Rademacher complexities of $\calH_\AT^{(y)}$, $\calH_\IT^{(y)}$, $\calH_\LS^{(y)}$, $\calH_\AT$, $\calH_\IT$, and $\calH_\LS$ with the class of linear-in-parameter models are bounded as
\begin{align}
\begin{aligned}
    \Rad(\calH_\AT^{(y)} ; n)
    &\le
    \frac{(C_{\btheta} + C_{\bw}) \, C_{\bphi} \, \rho \, K}{\sqrt{n}}
    ,   \\
    \Rad(\calH_\IT^{(y)} ; n)
    &\le
    \frac{2 (C_{\btheta} + C_{\bw}) \, C_{\bphi} \, \rho  }{\sqrt{n}}
    ,  \\
    \Rad(\calH_\LS^{(y)} ; n)
    &\le
    \frac{2 (|y + \theta_1 - 3/2| + C_{\bw} C_{\bphi} ) (C_{\btheta} + C_{\bw}) \, C_{\bphi}  K  }{\sqrt{n}} 
    , \\
    \Rad(\calH_\AT ; n)
    &\le
    \frac{(C_{\btheta} + C_{\bw}) \, C_{\bphi} \, \rho \, K (K-1)}{\sqrt{n}}
    ,   \\
    \Rad(\calH_\IT ; n)
    &\le
    \frac{2 (C_{\btheta} + C_{\bw}) \, C_{\bphi} \, \rho \, K  }{\sqrt{n}}
    ,  \\
    \Rad(\calH_\LS ; n)
    &\le
    \frac{2 ( K + |\theta_1 - 3/2| + C_{\bw} C_{\bphi} ) (C_{\btheta} + C_{\bw}) \, C_{\bphi}  K  }{\sqrt{n}}
    .  
\end{aligned}
\end{align}

\end{lemma}
The proof of Lemma~\ref{lem:OR-rad-bound-linear} is given in Appendix~\ref{sec:proof-of-Rademacher-OR-linear}.
Combining Theorem~\ref{thm:estimation-error-bound-general} and Lemma~\ref{lem:OR-rad-bound-linear} establishes the following estimation error bounds for the linear-in-parameter models.
\begin{corollary}[Estimation error bounds with linear-in-parameter models]
\label{cor:estimation-error-bound}
    Assume for the AT and IT losses that the binary surrogate loss $\ell$ is $\rho$-Lipschitz and that there exists a constant $C_\psi > 0$ such that $\psi(y, \balpha) \leq C_\psi$ for any $y \in \calY,\ \balpha \in \R^{K-1}$.
    Then, for any $k \in \mathcal{Y}$ and $\delta \in (0, 1]$, with probability at least $1 - \delta$,
    \begin{align}
        &
        \calS(\hat{g}^{\backslash k}) - \calS(g^*)   \nonumber \\
        &\le
        4 \left(
            \gamma
            \left(
                \sum_{y\in\calYrmk}  \pi_y
                \bigg(
                \frac{C^{(y)}}{\sqrt{n_y}} 
                + 
                \frac{C^{(k)}}{\sqrt{n_y}}
                \bigg)
                +
                \frac{C^{(k)}}{\sqrt{n_{\mathrm{U}}}}
            \right)
            +
            (1-\gamma) \frac{C}{\sqrt{n_{\mathrm{L}}}} 
        \right)   \nonumber \\
        &\qquad+
        2\sqrt{2} C_\psi \left( \gamma \left(2\sum_{y\in\calYrmk} \frac{\pi_y}{\sqrt{n_y}} + \frac{1}{\sqrt{n_\u}} \right) + (1-\gamma) \frac{1}{\sqrt{n_{\mathrm{L}}}} \right) \sqrt{\ln\frac{2(K+1)}{\delta}}
        ,
        \label{eq:estimation-error-bound-linear}
    \end{align}
    where $C^{(y)}$ for $y\in\calY$ and $C$ are constants defined as
    \begin{align}
        C^{(y)}
        &\coloneqq
        \begin{cases}
            (C_{\btheta} + C_{\bw}) \, C_{\bphi} \, \rho \, K & (\text{AT loss}) , \\
            2(C_{\btheta} + C_{\bw}) \, C_{\bphi} \, \rho   & (\text{IT loss})  , \\
            2 (|y + \theta_1 - 3/2| + C_{\bw} C_{\bphi} ) (C_{\btheta} + C_{\bw}) \, C_{\bphi} K  & (\text{LS loss})
            ,
        \end{cases}
        \\
        C 
        &\coloneqq
        \begin{cases}
            (C_{\btheta} + C_{\bw}) \, C_{\bphi} \, \rho \, K (K-1) & (\text{AT loss}) , \\
            2(C_{\btheta} + C_{\bw}) \, C_{\bphi} \, \rho \, K  & (\text{IT loss})  , \\
            2(K + |\theta_1 - 3/2| + C_{\bw} C_{\bphi} ) (C_{\btheta} + C_{\bw}) \, C_{\bphi} K  & (\text{LS loss})
            .
        \end{cases}
    \end{align}
\end{corollary}
Corollary~\ref{cor:estimation-error-bound} shows that our proposed risk estimator is consistent,
i.e., $\calS(\hat{g}^{\backslash k}) \rightarrow \calS(g^*)$ as $n_y \rightarrow \infty\; (y=1,\ldots,K)$ and $n_\u \rightarrow \infty$.
The convergence rate is
\begin{align*}
    \mathcal{O}_p \left( \sum_{y\in\calYrmk} \frac{1}{\sqrt{n_y}} + \frac{1}{\sqrt{n_\u}} \right)
    ,
\end{align*}
where $\mathcal{O}_p$ denotes the order in probability.
This order is the optimal parametric rate for empirical risk minimization without any additional assumption~\citep{mendelson2008lower}.

\subsection{Variance Reduction}
Here, we conduct an analysis of variance reduction to emphasize why incorporating unlabeled data to our risk estimator improves the performance.
The following theorem indicates that the variance of the empirical semi-supervised risk can be smaller than that of the empirical supervised risk.

\begin{theorem}[Variance reduction]
\label{thm:variance-reduction}
    Fix any $g \in \calG$ and assume that $\gamma \in [0,1]$ satisfies
    \begin{align*}
        0 < \gamma
        <
        \frac{
            2 \left( \Var[\empSVrisk(g)] - \Cov(\empRluk(g),\empSVrisk(g)) \right)
        }{
            \Var[\empRluk(g)] + \Var[\empSVrisk(g)] - 2 \Cov(\empRluk(g), \empSVrisk(g))
        }.
    \end{align*}
    Then,
    \begin{align*}
        \Var[\empsslrisk(g)] < \Var[\empSVrisk(g)].
    \end{align*}
\end{theorem}
The proof of Theorem~\ref{thm:variance-reduction} is given in Appendix~\ref{sec:proof-of-variance-reduction}.
This theorem implies that if we select $\gamma$ properly,
the variance of the empirical semi-supervised risk is strictly smaller than that of the empirical supervised risk.
Since the empirical semi-supervised risk is unbiased and has a smaller variance than the standard supervised risk, the former risk estimator is expected to be more accurate and stable, as we will see experimentally in Section~\ref{sec:experiments}.

\section{Practical Implementation}
\label{sec:practical}

With Theorem~\ref{thm:ssl-cvx-surrogate-risk}, we can obtain a risk estimator that can fully use both labeled and unlabeled data.
It is also straightforward to see that our risk estimator based on Theorem~\ref{thm:ssl-cvx-surrogate-risk} is unbiased. 
However, to use our risk estimator effectively in practice, one important question is: \emph{how can we decide the class $k$ to calculate $\empRluk(g)$?}
We discuss strategies to handle this problem theoretically.

\subsection{Strategies to Remove One Class for \texorpdfstring{$\empRluk(g)$}{Lg}}
\label{subsec:what-class-to-pick-out}
If the choice of $k$ is highly sensitive, it could be cumbersome to tune this hyperparameter as the number of classes $K$ increases. 
Here, we consider two strategies for selecting a class $k$ to removed in $\empRluk(g)$ motivated by the property of the estimator.
We provide two strategies.

The first strategy is based on finite sample estimation error.
More specifically, when we approximate the expectation
term by a limited number of samples, the variance of the estimator can be large.
Then a naive strategy would be to remove the class that contains the smallest number of labeled data as
\begin{equation}
    \label{eq:strategy1}
    k = \argmin_{y\in \mathcal{Y}} n_y.
\end{equation}

The other strategy is based on the estimation error bound.
As discussed in Theorem~\ref{thm:estimation-error-bound-general}, the convergence rate of the estimation error for $\empRluk(g)$ is
$\mathcal{O}_p(\sum_{y\in\calYrmk} {\pi_y}/{\sqrt{n_y}})$ under the assumptions that we have enough unlabeled data $(n_\u \rightarrow \infty)$.
This implies the following proposition which provides a strategy to decide the removed class $k$.
\begin{proposition}
\label{cor:what-class-to-pick-out-joint}
  Assume that labeled data are obtained under the assumption in~\eqref{eq:joint-assumption},
  and the above assumptions are satisfied.
  Then, 
  \begin{align}
      \label{eq:strategy2}
      k &= \argmax_{y\in \mathcal{Y}} {n_y}
  \end{align}
  gives the lowest upper bound of the estimation error as $n_{\mathrm{L}} \rightarrow \infty$.
\end{proposition}
It is interesting to see that the above two strategies give completely opposite solutions as the removed class.
We will experimentally compare these strategies in Section~\ref{sec:experiments}, 
where we find that both strategies performed similarly although they are completely different, indicating that the choice of $k$ may not be critical to the performance.

\subsection{Order Constraints}\label{sec:thr-con}
In ordinal regression problems, the threshold parameters should be ordered, i.e., $\theta_1 \leq \theta_2 \leq \cdots \leq \theta_{K-1}$~\citep{pedregosa:consistency}.
Here, we introduce a simple trick to constrain the threshold parameters $\btheta$ by adding the term
\begin{align}
\label{eq:order-constraints}
  \Omega(\btheta) \coloneqq 
  \mu \sum_{i=1}^{K-2} \max \left\{0,\ -\log (\theta_{i+1} - \theta_i) \right\},
\end{align}
where $\mu \geq 0$ is a regularization parameter for the order constraints.
We show that this simple trick works well in the experiments.
In fact, even without any regularization term on $\btheta$,
we empirically observed that the values of threshold constraints $\btheta$ are usually ordered.
This observation suggests that the order constraints are not difficult to satisfy when optimizing $\empRiskSEMI$.
This threshold regularization scheme is based on the idea of the log-barrier method~\citep{boyd2004convex},
in which we impose a high cost $-\log(-c(\btheta))$ to the objective function for an inequality constraint $c(\btheta) \le 0$ for $c:\Theta \rightarrow \R$.
However, only using $-\log (\theta_{i+1} - \theta_i)$ may lead to the following two problems. 
First, the algorithm may make $\theta_{i+1} - \theta_i$ large more than necessary while $\theta_{i+1} - \theta_i \ge 0$ is only required.
Second, the function $-\log(\theta_{i+1} - \theta_i)$ becomes negative if $\theta_{i+1} - \theta_i > 1$.
This may cause an overall risk to be arbitrarily negative by making $\theta_{i+1} - \theta_i$ larger, and cause instability in the optimization procedure when combining it with the empirical version of~\eqref{eq:ssl-or-surrogate-risk}.
As a result, we further introduce $\max\{0,\cdot\}$ in~\eqref{eq:order-constraints} to mitigate these problems.

Here, we investigate the objective function when the linear-in-parameter model $f(\x) = \bw^\top \bphi(\x)$ is employed as $f$.
Our next theorem states a sufficient condition to guarantee that, for a certain task surrogate loss function, the optimization problem is convex with respect to both the model and order constraints parameters.
\begin{theorem}
  \label{thm:convex-ssl-or}
  Let $C_\ell$ be a positive constant.
  If the AT or LS losses are adopted as the task surrogate loss,
  and the binary surrogate loss $\ell(z)$ is convex and satisfies
  \begin{align}
    \ell(z) - \ell(-z) = -C_\ell z
    \label{eq:margin-loss-cond}
    ,
  \end{align}
  then the objective function $\hat{J}_\ell(\bw, \btheta) \coloneqq \empRiskSEMI(g) + \Omega(\btheta)$ is convex with respect to $\bw$ and $\btheta$.
\end{theorem}
The proof of Theorem~\ref{thm:convex-ssl-or} is given in Appendix~\ref{sec:proof-of-convex-ssl-or}.
The condition~\eqref{eq:margin-loss-cond} is known as the linear-odd condition~\citep{patrini2016loss}.
Examples of binary surrogate losses are shown in Figure~\ref{fig:binary-surr} and Table~\ref{tab:binary-surr}.
In our experiments, we used the logistic loss as the binary surrogate loss.
Note that the objective function with the IT loss is not always convex.

\begin{figure}[t]
    \centering
    \includegraphics[width=0.67\columnwidth]{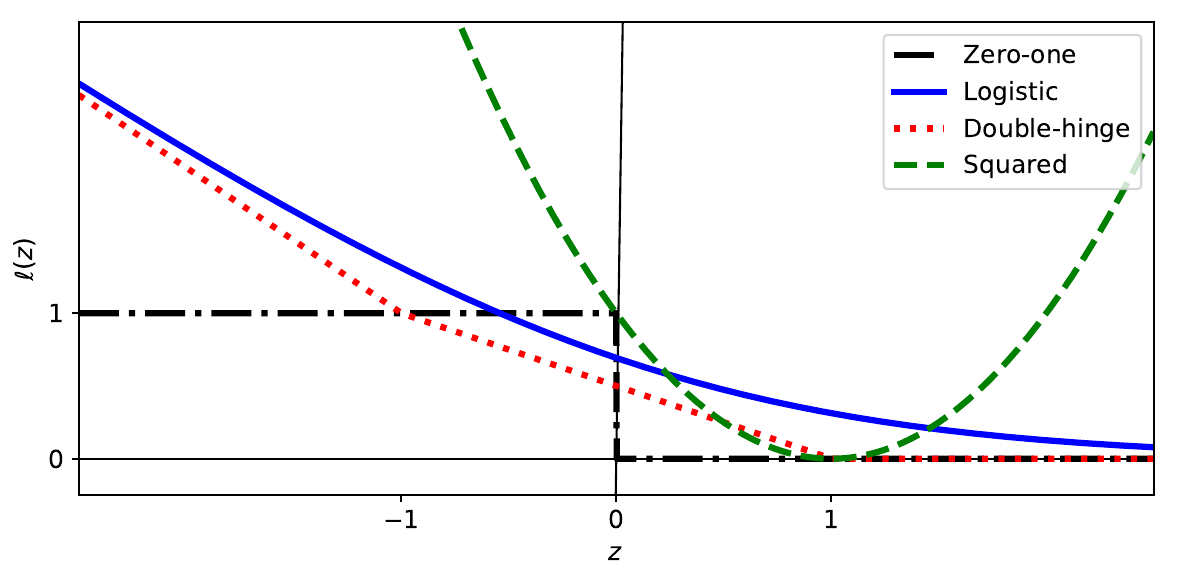}
    \caption{
        Examples of the binary surrogate losses satisfying~\eqref{eq:margin-loss-cond}.
        The zero-one function is depicted for reference.
    }
    \label{fig:binary-surr}

    \centering
    \captionof{table}{Examples of binary surrogate losses and their satisfiability of the sufficient condition in~\eqref{eq:margin-loss-cond}.}
    \label{tab:binary-surr}
    \begin{tabular}{CCC  C  C  C  C}
        \hline
        \text{Loss name} & \ell(z)  & \text{Eq.~\eqref{eq:margin-loss-cond}}\\ \hline
        \text{Hinge} & \max(0, 1-z)  & \text{No} \\
        \text{Exponential} & \exp(-z) & \text{No} \\
        \text{Logistic} & \log(1+e^{-z}) & \text{Yes} \\
        \text{Double-hinge} & \max(-z, \max(0, \frac{1}{2} - \frac{1}{2}z)) & \text{Yes} \\
        \text{Squared} & (1-z)^{2}  & \text{Yes} \\ \hline
    \end{tabular}

\end{figure}

\subsection{Non-Negative Risk Estimator}
\label{subsec:nn-risk}
We discuss a risk modification technique called a \emph{non-negative risk estimator}.
This technique was originally proposed in~\citet{kiryoNNPU2017} for classification from positive and unlabeled data based on an unbiased risk estimator.
We can see that the sum of the second and third terms of our empirical risk estimator~\eqref{eq:ssl-or-surrogate-estimator}, $\empRU(g) - \empRLtwo(g)$, can be negative while the corresponding expected risk $\RU(g) - \RLtwo(g) = \pi_k \expect_{X|Y=k}\left[ \psi(\balpha(X), k) \right]$ is always non-negative.
This observation suggests that the model can excessively reduce the empirical risk by maximizing $\empRLtwo(g)$.
To prevent this problem, following~\citet{kiryoNNPU2017}, we modify the empirical risk estimator~\eqref{eq:ssl-or-surrogate-estimator} as
\begin{align}
  \empRLone(g)
  + \max \left\{0,\empRU(g) - \empRLtwo(g) \right\}
  \label{eq:nn-ssl-or-surrogate-estimator}
  .
\end{align}
This approach is later generalized in~\citet{Lu2020mitigating} in the context of unlabeled-unlabeled classification.
Specifically, they considered the LeakyReLU-like regularization function instead of the max operator used in~\eqref{eq:nn-ssl-or-surrogate-estimator}.
They empirically showed that this regularization technique makes trained models generalize better.
Following this observation, we proposed to modify our empirical risk estimator~\eqref{eq:ssl-or-surrogate-estimator} as
\begin{align}
  \empRLone(g)
  + \mathrm{GLReLU} \left(\empRU(g) - \empRLtwo(g) \right)
  \label{eq:generalized-lrelu-ssl-or-surrogate-estimator}
  ,
\end{align}
where the Generalized Leaky ReLU function $\mathrm{GLReLU}: \R \rightarrow \R$ is defined as 
$\mathrm{GLReLU}(x) \coloneqq x \ind{x\ge0} + \lambda x \ind{x<0}$ for $\lambda \le 0$.
We will employ this regularization technique in experiments.

\section{Experiments}
\label{sec:experiments}

In this section, we begin by numerically investigating that unlabeled data helps reduce the variance of the risk estimator as theoretically discussed in Section~\ref{sec:theory}.
Besides, we experimentally investigate which strategy is better to remove one class for $\empRluk(g)$ as discussed in Section~\ref{subsec:what-class-to-pick-out}.
Finally, we present the experimental comparison results of semi-supervised ordinal regression on benchmark datasets.

\noindent \textbf{Common Setup.}
As evaluation metrics, we adopted the~\emph{mean absolute error},~\emph{mean zero-one error}, and~\emph{mean squared error}.
Note that each metric coincides with the task surrogate risk, which adopts the AT, IT, and LS losses as the task surrogate loss, respectively.
We did experiments on the AT, IT, and LS losses as the task surrogate loss and the logistic loss as the binary surrogate loss.
For validation of the hyperparameters,
we used the hold-out method by splitting the training data set with the ratio of $2:1$.
For both models, 
we fixed the hyper-parameters $\gamma$ and $\mu$ to $0.5$ and $10$, respectively.
Note that we empirically observed that if $\mu$ is large enough, $\mu$ is insensitive to the performance.
We ran the experiments $20$ times to calculate the mean and standard error of the performances.
Also, we used the non-negative risk estimator described in Section~\ref{subsec:nn-risk} with $\lambda = -0.2$ to prevent over-fitting.
We used Chainer~\citep{tokui2015chainer} to implement our models.
We obtained datasets from 
a survey paper on ordinal regression~\citep{gutierrez2016ordinal},
and the website on ordinal regression benchmark data\footnote{\url{https://www.gagolewski.com/resources/data/ordinal-regression/}}.
The detail of the dataset description is given in the following.

\noindent \textbf{Baselines.}
We compared our proposed methods (SEMI) against the following baselines.

\begin{enumerate}
    \item 
    \noindent \textbf{Supervised Ordinal Regression (SV):} 
    Supervised ordinal regression here is based on empirical risk minimization of the task surrogate risk, which is described in~\citet{pedregosa:consistency} and Section~\ref{sec:preliminaries}.
    Recall that the different task surrogate losses are used for each objective function.
    In other words, the AT, IT, and LS losses are used when the mean absolute error, mean zero-one error, and mean squared error are used as evaluation metrics, respectively.
    \item
    \noindent \textbf{Transductive Ordinal Regression (TOR):} 
    Transductive Ordinal Regression~\citep{seah2012transductive}
    is a method based on pseudo-labeling, and TOR assumes that the data satisfies the cluster assumption.
    TOR tries to minimize objective function by repeatedly minimizing the sum of risks on labeled data and pseudo-labeled data.
    Note that the classifier obtained by the training can be used as an inductive classifier.
    We noticed that the TOR is essentially a pseudo-labeling method based on supervised empirical risk minimization with AT loss,
    while it is not mentioned in the paper, implying that we can extend the use IT and LS loss instead of AT loss.
    Thus, we used IT and LS loss in our experiments when adopting mean zero-one error and mean squared error, respectively.
    \item 
    \noindent \textbf{Semi-Supervised Manifold Ordinal Regression (SSMOR):} 
    Semi-Supervised Manifold Ordinal Regression~\citep{liu2011sslor_manifold} is a method based on the manifold assumption.
    SSMOR first minimizes the objective function, 
    which is the sum of smoothness constraint (nearby points should have the same label) and fitting constraint (fitting on the given labels).
    After optimizing the objective function, 
    SSMOR tries to obtain labels for new data points by calculating thresholds in a heuristic manner.
    However, we found that the threshold calculated via this approach can be not ordered, and its performance is poor.
    Thus, we first sorted the data points based on the prediction score and then allocated labels with the class prior.
\end{enumerate}

\subsection{Variance Reduction and Performance Comparison on Small Datasets}
Here, we empirically investigate the effect of the variance reduction discussed in Section~\ref{sec:theory}, and compare the empirical performance of proposed method on small datasets.
The class size was fixed to $K=3$ by merging some classes into one class for each dataset.
We sub-sampled labeled data with $\nL = 20$.
while the number of unlabeled data depends on the dataset size. Full dataset statistics is given in Appendix~\ref{subsec:dataset-stats}.
We used a linear-in-input model $f(\x)=\bw^\top \x$.
For training our proposed method, we trained the model for $1000$ epochs using Adam with $\alpha=5.0\times 10^{-2}$ (full batch size).
The candidates of the weight decay parameter were $\{10^{-6}, 10^{-4}, 10^{-2}\}$.

\begin{figure*}[t]
    \begin{center}
    \setlength{\subfigwidth}{.32\linewidth}
    \addtolength{\subfigwidth}{-.32\subfigcolsep}
    \begin{minipage}[t]{\subfigwidth}
        \centering
         \subfigure[{\tt bank1-5}, AT]{\includegraphics[scale=0.40]{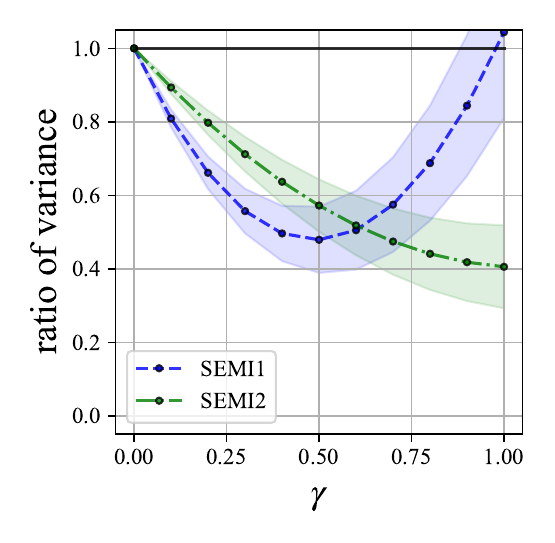}}
    \end{minipage}\hfill
    \begin{minipage}[t]{\subfigwidth}
        \centering
         \subfigure[{\tt bank1-5}, IT]{\includegraphics[scale=0.40]{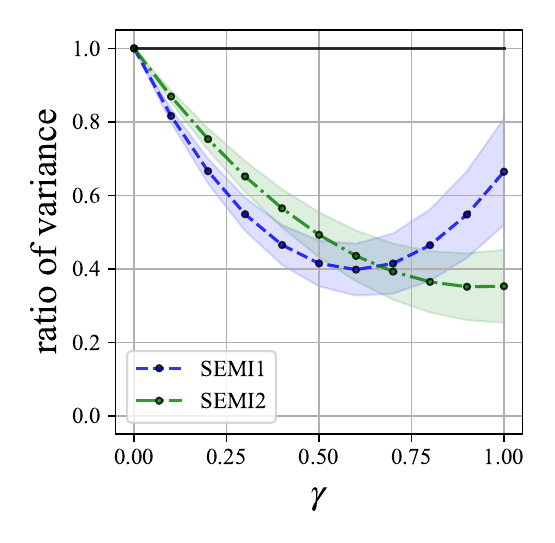}}
    \end{minipage}\hfill
    \begin{minipage}[t]{\subfigwidth}
        \centering
         \subfigure[{\tt bank1-5}, LS]{\includegraphics[scale=0.40]{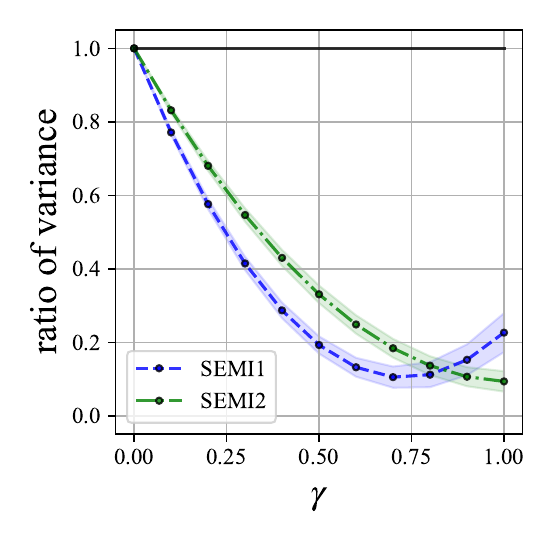}}
    \end{minipage}\hfill
    \begin{minipage}[t]{\subfigwidth}
        \centering
         \subfigure[{\tt census2-5}, AT]{\includegraphics[scale=0.40]{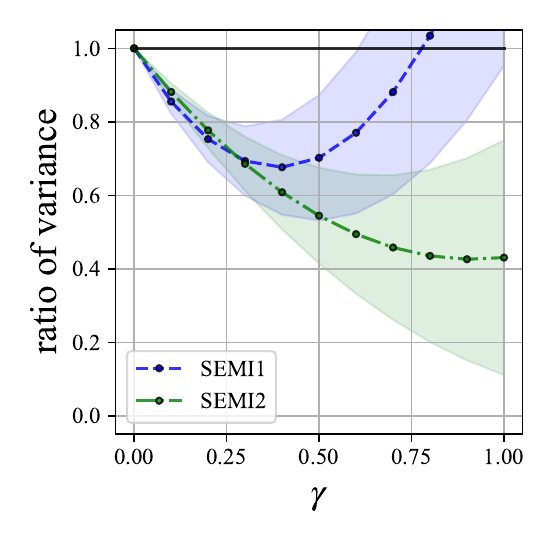}}
    \end{minipage}\hfill
    \begin{minipage}[t]{\subfigwidth}
        \centering
         \subfigure[{\tt census2-5}, IT]{\includegraphics[scale=0.40]{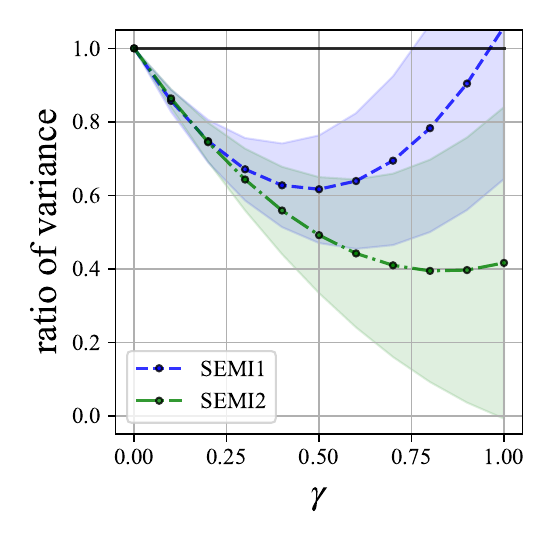}}
    \end{minipage}\hfill
    \begin{minipage}[t]{\subfigwidth}
        \centering
         \subfigure[{\tt census2-5}, LS]{\includegraphics[scale=0.40]{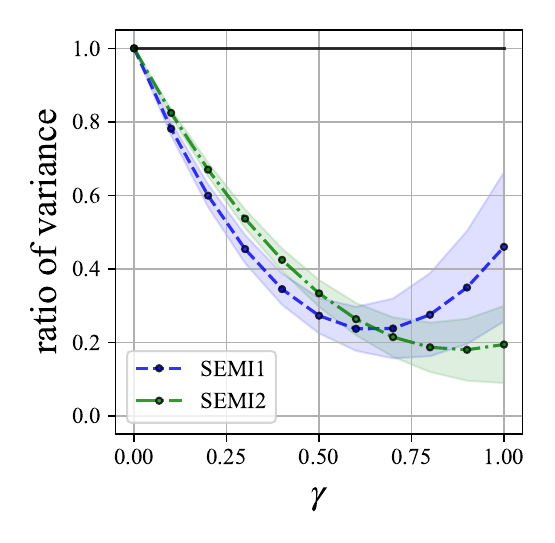}}
    \end{minipage}\hfill
    \end{center}
    \caption{The ratio between the variance of the empirical semi-supervised risk to the supervised risk when we adopted AT, IT, and LS as the task surrogate loss and used linear-in-input model.
    The ratio below 1.0 implies that the variance of the semi-supervised risk is lower than that of the supervised risk.
    }
    \label{fig:variance_reduction_selected}
\end{figure*}

\noindent \textbf{Variance Reduction.}
Here, before going to the performance comparison,
we investigate the effect of variance reduction by using unlabeled data when estimating the training risk.
We adopted the randomly initialized ordinal regressor $g_\rand$, 
which is the linear-in-input model, 
to compute the variance of the empirical supervised risk $\Var[\empSVrisk(g_\rand)]$ and the empirical semi-supervised risk $\Var[\empRiskSEMI(g_\rand)]$.
We selected the removed class $k$ using the strategies described in Section~\ref{subsec:what-class-to-pick-out}.
We denote SEMI1 as our proposed semi-supervised methods, where the strategy of removing a class is based on the finite sample approximation~\eqref{eq:strategy1}, while SEMI2 denotes the strategy based on the estimation error bound~\eqref{eq:strategy2}.
Then, we calculated the ratio between the variance of the empirical semi-supervised risk and supervised risk
\[\Var[\empRiskSEMI(g_\rand)] / \Var[\empSVrisk(g_\rand)]. \]

Figure~\ref{fig:variance_reduction_selected}
shows the ratio between the variance of the empirical semi-supervised risk and supervised risk,
which is evaluated on different $\gamma \in \{ 0.0, 0.1, 0.2, \ldots, 1.0 \}$.
Additional results of experiments are given in Appendix~\ref{subsec:additional_experiments}.
It can be observed that the ratios were less than $1$ in most cases.
This indicates that the variance of our empirical semi-supervised risk is smaller than that of the empirical supervised risk, as suggested by Theorem~\ref{thm:variance-reduction}.

\noindent \textbf{Performance Comparison.}
Tables~\ref{tab:real-result-linear} shows the performances of the proposed framework against the supervised method.
We can see that there is no clear difference between the simple strategy (SEMI1) and the strategy based on the estimation error bound (SEMI2).
As for the performance comparison, 
We can see that the proposed methods (SEMI1, SEMI2) perform better than the standard supervised learning (SV).
It is worth noting that our framework performs well in all task losses, indicating the broad applicability of our proposed framework.

\begin{table}[t]
\centering
\caption{
The mean and standard error of the mean absolute error, mean zero-one error, and mean squared error on small benchmark datasets. 
The experiments were conducted 20 times.
Outperforming methods in each evaluation metric an dataset are highlighted in boldface using the Wilcoxon signed-rank test with a significance level of 5\%.
}
\label{tab:real-result-linear}
\scalebox{0.65}[0.65]{  
\begin{tabular}{l|ccc|ccc|ccc}
\toprule
{}  & \multicolumn{3}{|c|}{Mean Absolute Error} & \multicolumn{3}{c|}{Mean Zero-one Error} & \multicolumn{3}{c}{Mean Squared Error} \\
\midrule
{} &     SV  &    SEMI1 &    SEMI2 &     SV  &    SEMI1 &    SEMI2  &     SV  &    SEMI1 &    SEMI2  \\
\midrule
{\tt abalone }         &       0.381\;(0.02) &  {\bf 0.370\;(0.04)} &  {\bf 0.356\;(0.03)} &       0.372\;(0.02) &  {\bf 0.368\;(0.03)} &  {\bf 0.353\;(0.03)}  &  0.447\;(0.06) &        0.392\;(0.05) &  {\bf 0.365\;(0.03)}   \\
{\tt bank1-5 }         &       0.312\;(0.06) &  {\bf 0.280\;(0.06)} &  {\bf 0.285\;(0.05)} &       0.309\;(0.06) &  {\bf 0.277\;(0.06)} &  {\bf 0.283\;(0.05)}  &  0.353\;(0.09) &  {\bf 0.281\;(0.05)} &  {\bf 0.290\;(0.06)}   \\
{\tt bank2-5 }         & {\bf 0.518\;(0.04)} &  {\bf 0.518\;(0.04)} &        0.616\;(0.04) &       0.484\;(0.03) &  {\bf 0.475\;(0.03)} &        0.552\;(0.03)  &  1.110\;(0.08) &  {\bf 0.701\;(0.07)} &  {\bf 0.710\;(0.09)}   \\
{\tt census1-5 }       &       0.453\;(0.07) &  {\bf 0.414\;(0.04)} &  {\bf 0.416\;(0.04)} &       0.425\;(0.06) &  {\bf 0.400\;(0.03)} &  {\bf 0.397\;(0.03)}  &  0.544\;(0.07) &        0.468\;(0.04) &  {\bf 0.431\;(0.03)}   \\
{\tt census2-5 }       &       0.534\;(0.07) &  {\bf 0.496\;(0.06)} &  {\bf 0.500\;(0.05)} &       0.488\;(0.06) &  {\bf 0.453\;(0.05)} &  {\bf 0.457\;(0.04)}  &  0.755\;(0.11) &        0.627\;(0.10) &  {\bf 0.530\;(0.06)}   \\
{\tt computer1-5 }     &       0.308\;(0.05) &  {\bf 0.295\;(0.05)} &  {\bf 0.302\;(0.04)} & {\bf 0.299\;(0.05)} &  {\bf 0.291\;(0.04)} &  {\bf 0.300\;(0.04)}  &  0.461\;(0.10) &        0.403\;(0.06) &  {\bf 0.329\;(0.06)}   \\
{\tt computer2-5 }     &       0.312\;(0.04) &  {\bf 0.286\;(0.03)} &        0.317\;(0.04) &       0.306\;(0.04) &  {\bf 0.283\;(0.03)} &        0.310\;(0.04)  &  0.511\;(0.08) &        0.432\;(0.06) &  {\bf 0.356\;(0.05)}   \\
{\tt fireman } & {\bf 0.556\;(0.08)} &  {\bf 0.556\;(0.08)} &  {\bf 0.554\;(0.07)} & {\bf 0.471\;(0.06)} &  {\bf 0.464\;(0.06)} &  {\bf 0.461\;(0.04)}  &  0.637\;(0.12) &        0.626\;(0.11) &  {\bf 0.591\;(0.08)}   \\
{\tt kinematics }      &       0.635\;(0.06) &  {\bf 0.606\;(0.05)} &  {\bf 0.608\;(0.06)} & {\bf 0.536\;(0.03)} &  {\bf 0.526\;(0.04)} &  {\bf 0.528\;(0.03)}  &  0.779\;(0.10) &  {\bf 0.702\;(0.07)} &  {\bf 0.743\;(0.12)}   \\
{\tt lev }             &       0.339\;(0.07) &  {\bf 0.308\;(0.04)} &  {\bf 0.301\;(0.04)} &       0.335\;(0.06) &  {\bf 0.308\;(0.03)} &  {\bf 0.299\;(0.04)}  &  0.354\;(0.07) &  {\bf 0.345\;(0.13)} &  {\bf 0.304\;(0.05)}   \\
{\tt swd }             &       0.701\;(0.05) &        0.700\;(0.12) &  {\bf 0.620\;(0.05)} &       0.579\;(0.03) &        0.566\;(0.03) &  {\bf 0.537\;(0.04)}  &  0.873\;(0.12) &        0.765\;(0.10) &  {\bf 0.683\;(0.08)}   \\
{\tt toy }             &       0.266\;(0.07) &  {\bf 0.215\;(0.05)} &  {\bf 0.226\;(0.05)} &       0.235\;(0.05) &  {\bf 0.212\;(0.04)} &  {\bf 0.215\;(0.04)}  &  0.361\;(0.10) &  {\bf 0.250\;(0.08)} &        0.293\;(0.08)   \\
\bottomrule
\end{tabular}
} 
\end{table}

\noindent \textbf{The Effect of the Number of Unlabeled Data.}
Here, we investigate the effect of the number of unlabeled data on the performance improvement against the supervised empirical risk minimization approach. 
We first trained the ordinal regressors with the linear-in-input model for a different number of unlabeled data.
Then, we calculated the ratio between the error of the semi-supervised method to the supervised method.
Note again that the different evaluation metrics are used for the different task surrogate losses.

Figure~\ref{fig:nu_change_selected} shows the ratio between the error of the empirical semi-supervised risk and supervised risk, which is evaluated on the number of unlabeled data.
Additional results of experiments are given in Appendix~\ref{subsec:additional_experiments}.
It can be observed that SEMI1 and SEMI2 outperform the supervised method. 
However, it is also observed that the performance gain is saturated at a certain point depending on the dataset, i.e., increasing more unlabeled data may no longer improve the performance of both SEMI1 and SEMI2. 
Note that this is also a common phenomenon observed in the literature of semi-supervised learning~\citep{sakai-pnu, Oliver2018realistic, Sakai2018AUC}.  
Nevertheless, it is a challenging future work to consider a way to improve the way to utilize unlabeled data, for example, by developing a new regularization technique for the unbiased risk estimator approach in the context of semi-supervised learning.


\begin{figure*}[t]
    \begin{center}
    \setlength{\subfigwidth}{.30\linewidth}
    \addtolength{\subfigwidth}{-.30\subfigcolsep}
    \begin{minipage}[t]{\subfigwidth}
        \centering
         \subfigure[{\tt bank1-5}, AT]{\includegraphics[scale=0.32]{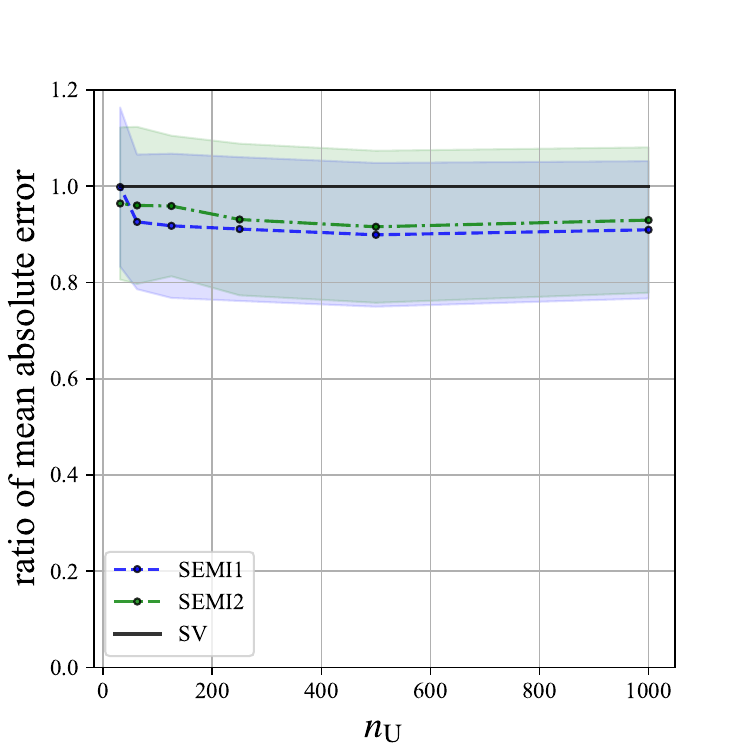}}
    \end{minipage}\hfill
    \begin{minipage}[t]{\subfigwidth}
        \centering
         \subfigure[{\tt bank1-5}, IT]{\includegraphics[scale=0.32]{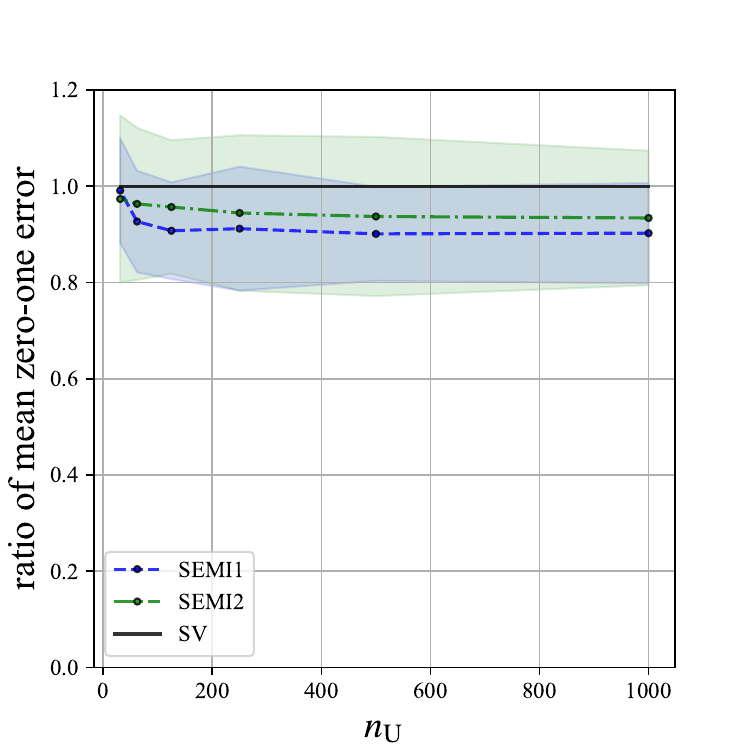}}
    \end{minipage}\hfill
    \begin{minipage}[t]{\subfigwidth}
        \centering
         \subfigure[{\tt bank1-5}, LS]{\includegraphics[scale=0.32]{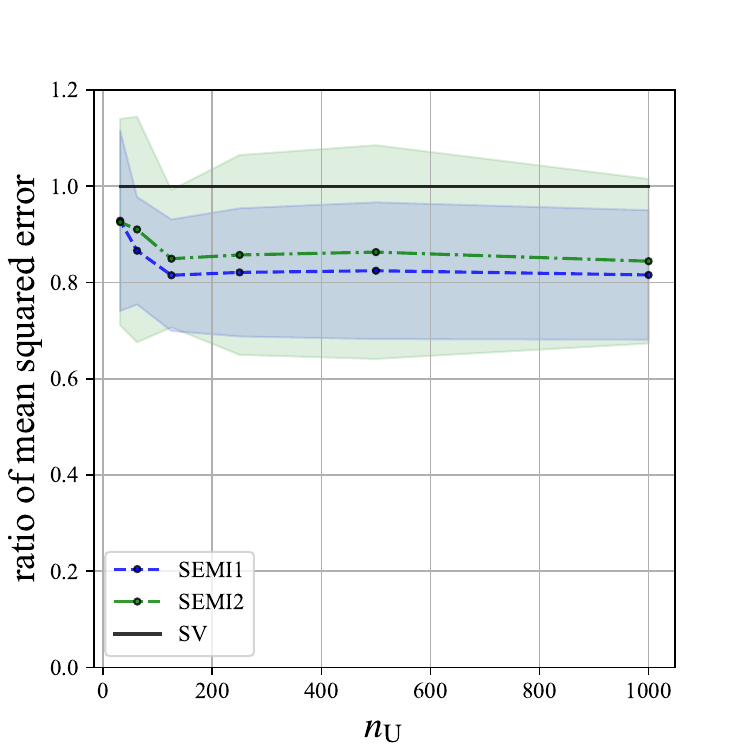}}
    \end{minipage}\hfill
    \vspace{-7pt}
    \begin{minipage}[t]{\subfigwidth}
        \centering
         \subfigure[{\tt census2-5}, AT]{\includegraphics[scale=0.32]{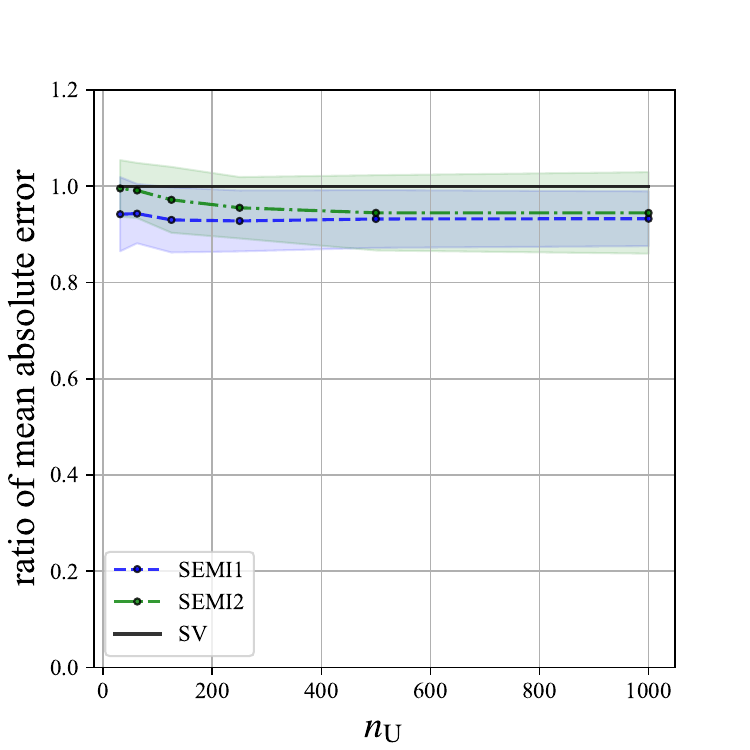}}
    \end{minipage}\hfill
    \begin{minipage}[t]{\subfigwidth}
        \centering
         \subfigure[{\tt census2-5}, IT]{\includegraphics[scale=0.32]{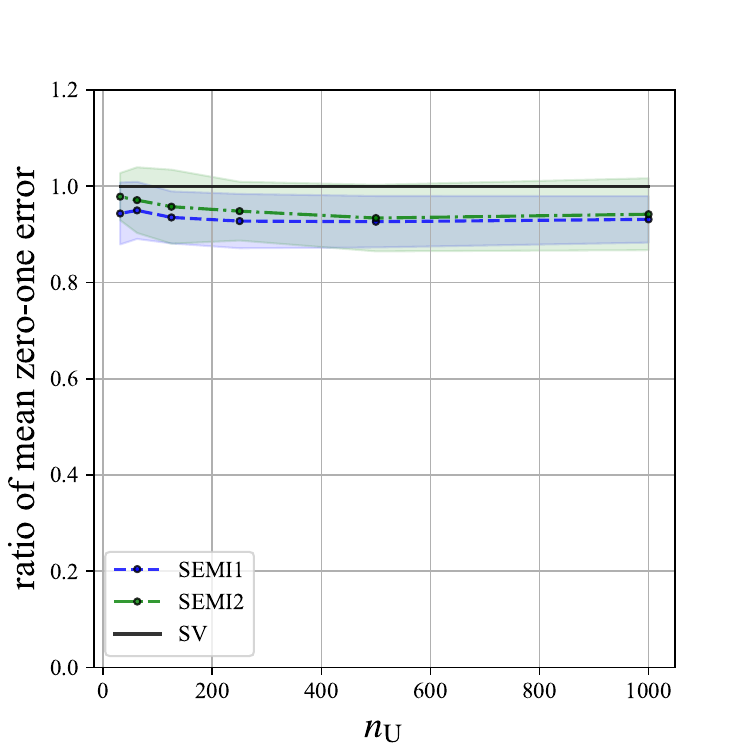}}
    \end{minipage}\hfill
    \begin{minipage}[t]{\subfigwidth}
        \centering
         \subfigure[{\tt census2-5}, LS]{\includegraphics[scale=0.32]{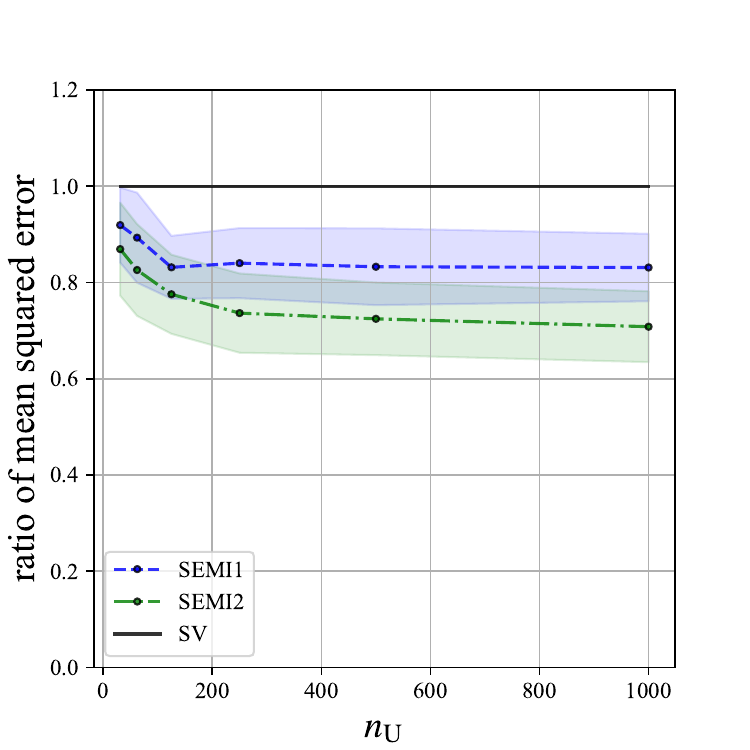}}
    \end{minipage}\hfill
    \end{center}
    \caption{The ratio between the error of the empirical semi-supervised risk to the supervised risk when we adopted AT, IT, and LS as the task surrogate loss and used a linear-in-input model.
    The ratio below 1.0 implies that the error of the semi-supervised method is better than that of the supervised risk.
    }
    \label{fig:nu_change_selected}
\end{figure*}

\subsection{Performance Comparison on Large Datasets}
Here, we compare the empirical performance of proposed framework on large datasets.
We sub-sampled labeled data with $\nL = 500$, and $\nU = 2000$, and the number of test size was fixed to $2000$. The statistics of datasets used in the experiments is given in Appendix~\ref{sec:additional_experiments}.
We used linear-in-parameter models with Gaussian kernel 
$f(\x)=\sum_{i=1}^{\nL} w_{i} \exp (-\left\|\x-\x_{i}\right\|^{2} /\left(2 \sigma^{2}\right))$,
and neural networks with one hidden layer (of size $256$) and ReLU as an activation function.
For the Gaussian kernel, the bandwidth candidates were $\{0.5, 1.0, 2.0 \} \cdot \textrm{median} (\left\|\x_{i}-\x_{j}\right\|_{i, j=1}^{\nL} )$ where $\{\x_i\}_{i=1}^{\nL} = \mathcal{X}_\mathrm{L}$ .
For training our proposed method, we trained the model for $1000$ epochs using Adam with $\alpha=5.0\times 10^{-3}$ (full batch size).
The weight decay parameter was fixed to $10^{-4}$ for both models.

\noindent \textbf{Performance Comparison.}
Tables~\ref{tab:real-result-AT-large} to~\ref{tab:real-result-LS-large} 
show the performances of the proposed methods against the baseline methods.
We can see that the proposed methods (SEMI1-NN and SEMI2-NN) perform best than the baseline methods on most benchmark datasets.
Specifically, we can observe that the method based on the manifold assumption (SSMOR) performs poorly
while the performance method based on the cluster assumption (TOR) is somewhat comparable with the proposed methods.
This is because the datasets used in the experiments do not satisfy the manifold assumption as the SSMOR is original proposed in the context of image ranking,
while TOR does not heavily rely on the geometric assumption.
As well as the case for the small datasets,
our framework performs well in all task losses, indicating the broad applicability of our proposed framework.
However, unlike neural networks, when a linear-in-parameter model with Gaussian kernel is used (SEMI1-Ker and SEMI2-Ker), we found that our method failed to improve the performance over the supervised method.
This could be due to the risk modification in~\eqref{eq:generalized-lrelu-ssl-or-surrogate-estimator} might not be suitable for a kernel model since it was originally designed for using with neural networks~\citep{Lu2020mitigating}. 
We hypothesize that the poor performance of the kernel method when using~\eqref{eq:generalized-lrelu-ssl-or-surrogate-estimator} can be attributed to the fact that the modification makes the risk formulation to be no longer convex, which causes difficulty in optimization.
Note that the supervised baseline does not use the risk modification in~\eqref{eq:generalized-lrelu-ssl-or-surrogate-estimator}. 
Hence, the supervised baseline has a convex formulation and therefore is easier to train.
We note that in the case of neural networks, the optimization problem is non-convex regardless of whether we use a risk modification or not. 
Thus, the modification does not affect the convexity of the problem
To improve the performance of the kernel model with our proposed method, it is an important future work to explore a risk modification that can retain the convexity of the formulation.


\begin{table}[t]
\centering
\caption{
The mean and standard error of the mean absolute error on benchmark datasets. 
Outperforming methods in each dataset are highlighted in boldface using the Wilcoxon signed-rank test with a significance level of 5\%.
The experiments were conducted 20 times.
}
\label{tab:real-result-AT-large}
\scalebox{0.65}[0.65]{  
\begin{tabular}{lcccccccc}
\toprule
{} & SSMOR &  TOR &   SV-Ker & SEMI1-Ker & SEMI2-Ker &   SV-NN & SEMI1-NN & SEMI2-NN \\
\midrule
{\tt bank1-5 }         &                 0.935\;(0.16) &  {\bf 0.250\;(0.01)} &  {\bf 0.249\;(0.01)} &         0.294\;(0.05) &         0.339\;(0.06) &        0.267\;(0.01) &         0.251\;(0.01) &   {\bf 0.248\;(0.01)} \\
{\tt bank2-5 }         &                 1.184\;(0.07) &        0.907\;(0.03) &        0.883\;(0.02) &         0.887\;(0.03) &         0.915\;(0.03) &        0.834\;(0.02) &         0.839\;(0.02) &   {\bf 0.828\;(0.02)} \\
{\tt census1-5 }       &                 1.190\;(0.11) &        0.803\;(0.03) &        0.785\;(0.03) &         0.792\;(0.02) &         0.803\;(0.02) &        0.785\;(0.03) &   {\bf 0.686\;(0.02)} &         0.697\;(0.02) \\
{\tt census2-5 }       &                 1.206\;(0.14) &        0.784\;(0.02) &        0.768\;(0.02) &         0.779\;(0.02) &         0.798\;(0.03) &        0.745\;(0.02) &   {\bf 0.693\;(0.02)} &   {\bf 0.689\;(0.02)} \\
{\tt computer1-5 }     &                 0.671\;(0.05) &        0.444\;(0.02) &        0.441\;(0.02) &         0.445\;(0.02) &         0.459\;(0.02) &        0.474\;(0.03) &   {\bf 0.413\;(0.02)} &   {\bf 0.409\;(0.02)} \\
{\tt computer2-5 }     &                 0.668\;(0.07) &        0.379\;(0.01) &        0.376\;(0.01) &         0.387\;(0.01) &         0.397\;(0.02) &        0.457\;(0.01) &   {\bf 0.367\;(0.01)} &   {\bf 0.370\;(0.02)} \\
{\tt fireman } &                 2.999\;(0.44) &        3.690\;(0.07) &        3.649\;(0.06) &         3.732\;(0.13) &         3.798\;(0.18) &  {\bf 1.298\;(0.05)} &         1.602\;(0.08) &         1.593\;(0.09) \\
{\tt kinematics }      &                 1.797\;(0.13) &        1.442\;(0.02) &        1.454\;(0.02) &         1.474\;(0.02) &         1.474\;(0.02) &  {\bf 1.155\;(0.03)} &         1.175\;(0.04) &         1.175\;(0.04) \\
\bottomrule
\end{tabular}
} 
\end{table}

\bigskip

\begin{table}[H]
\centering
\caption{
The mean and standard error of the zero-one error on benchmark datasets. 
Outperforming methods in each dataset are highlighted in boldface using the Wilcoxon signed-rank test with a significance level of 5\%.
The experiments were conducted 20 times.
}
\label{tab:real-result-IT-large}
\scalebox{0.65}[0.65]{  
\begin{tabular}{lcccccccc}
\toprule
{} & SSMOR &  TOR &   SV-Ker & SEMI1-Ker & SEMI2-Ker &   SV-NN & SEMI1-NN & SEMI2-NN \\
\midrule
{\tt bank1-5 }         &                 0.608\;(0.04) &  {\bf 0.238\;(0.01)} &  {\bf 0.238\;(0.01)} &         0.244\;(0.01) &         0.246\;(0.01) &        0.273\;(0.01) &         0.255\;(0.01) &         0.259\;(0.01) \\
{\tt bank2-5 }         &                 0.706\;(0.02) &        0.646\;(0.01) &        0.651\;(0.02) &         0.641\;(0.01) &         0.647\;(0.01) &  {\bf 0.602\;(0.01)} &   {\bf 0.601\;(0.01)} &         0.605\;(0.01) \\
{\tt census1-5 }       &                 0.697\;(0.02) &        0.569\;(0.01) &        0.579\;(0.01) &         0.584\;(0.02) &         0.585\;(0.01) &        0.574\;(0.01) &   {\bf 0.556\;(0.01)} &   {\bf 0.560\;(0.01)} \\
{\tt census2-5 }       &                 0.689\;(0.03) &        0.568\;(0.01) &        0.572\;(0.01) &         0.579\;(0.01) &         0.576\;(0.01) &        0.549\;(0.01) &         0.541\;(0.01) &   {\bf 0.538\;(0.01)} \\
{\tt computer1-5 }     &                 0.544\;(0.04) &        0.386\;(0.01) &        0.386\;(0.01) &         0.386\;(0.01) &         0.386\;(0.01) &        0.416\;(0.02) &   {\bf 0.376\;(0.02)} &         0.383\;(0.01) \\
{\tt computer2-5 }     &                 0.537\;(0.04) &  {\bf 0.339\;(0.01)} &  {\bf 0.339\;(0.01)} &   {\bf 0.341\;(0.01)} &         0.344\;(0.02) &        0.398\;(0.01) &   {\bf 0.343\;(0.01)} &         0.357\;(0.02) \\
{\tt fireman } &                 0.863\;(0.02) &        0.927\;(0.01) &        0.926\;(0.00) &         0.927\;(0.01) &         0.926\;(0.01) &  {\bf 0.682\;(0.01)} &         0.718\;(0.01) &         0.718\;(0.02) \\
{\tt kinematics }      &                 0.785\;(0.01) &        0.768\;(0.01) &        0.779\;(0.01) &         0.781\;(0.01) &         0.781\;(0.01) &  {\bf 0.705\;(0.01)} &         0.711\;(0.01) &         0.711\;(0.01) \\
\bottomrule
\end{tabular}
} 
\end{table}

\bigskip

\begin{table}[H]
\centering
\caption{
The mean and the standard error of the mean squared error on benchmark datasets. 
Outperforming methods in each dataset are highlighted in boldface using the Wilcoxon signed-rank test with a significance level of 5\%.
The experiments were conducted 20 times.
}
\label{tab:real-result-LS-large}
\scalebox{0.65}[0.65]{  
\begin{tabular}{lcccccccc}
\toprule
{} & SSMOR &  TOR &   SV-Ker & SEMI1-Ker & SEMI2-Ker &   SV-NN & SEMI1-NN & SEMI2-NN \\
\midrule
{\tt bank1-5 }         &                 1.719\;(0.47) &       0.364\;(0.02) &        0.363\;(0.02) &         0.415\;(0.08) &         0.432\;(0.06) &        0.317\;(0.02) &         0.313\;(0.02) &   {\bf 0.306\;(0.02)} \\
{\tt bank2-5 }         &                 2.448\;(0.23) &       1.282\;(0.05) &  {\bf 1.256\;(0.04)} &         1.595\;(0.07) &         1.380\;(0.08) &        1.464\;(0.07) &         1.481\;(0.07) &         1.476\;(0.07) \\
{\tt census1-5 }       &                 2.524\;(0.44) &       1.233\;(0.08) &        1.183\;(0.06) &         1.188\;(0.07) &         1.222\;(0.06) &        1.264\;(0.08) &   {\bf 1.053\;(0.07)} &   {\bf 1.068\;(0.06)} \\
{\tt census2-5 }       &                 2.493\;(0.51) &       1.205\;(0.07) &        1.172\;(0.06) &         1.140\;(0.04) &         1.198\;(0.06) &        1.189\;(0.06) &         1.100\;(0.05) &   {\bf 1.079\;(0.04)} \\
{\tt computer1-5 }     &                 0.958\;(0.09) &       0.550\;(0.03) &        0.538\;(0.03) &         0.542\;(0.03) &         0.569\;(0.05) &        0.618\;(0.05) &         0.536\;(0.03) &   {\bf 0.515\;(0.03)} \\
{\tt computer2-5 }     &                 0.944\;(0.16) &       0.457\;(0.02) &  {\bf 0.453\;(0.02)} &   {\bf 0.456\;(0.02)} &         0.484\;(0.02) &        0.604\;(0.05) &         0.502\;(0.05) &         0.480\;(0.03) \\
{\tt fireman } &                15.753\;(4.34) &       6.472\;(0.24) &        6.519\;(0.24) &         7.553\;(1.32) &         7.226\;(1.01) &  {\bf 2.766\;(0.19)} &         3.319\;(0.33) &         3.168\;(0.28) \\
{\tt kinematics }      &                 5.633\;(0.74) &       2.878\;(0.14) &        2.751\;(0.08) &         2.994\;(0.14) &         2.994\;(0.14) &        2.597\;(0.11) &   {\bf 2.362\;(0.13)} &   {\bf 2.362\;(0.13)} \\
\bottomrule
\end{tabular}
} 
\end{table}

\section{Conclusions}
\label{sec:conclusion}
We presented a novel framework to incorporate unlabeled data in ordinal regression based on empirical risk minimization.
We proposed an unbiased risk estimator that is applicable to all well-known task losses,
such as the absolute loss, squared loss, and zero-one loss.
We also elucidated the property of the proposed unbiased risk estimator
through the analysis of the estimation error bound to guarantee that the proposed risk estimator is consistent.
Experimental results showed that our proposed framework could effectively make use of unlabeled data,
resulting in better scores compared to the standard supervised learning method in three task losses in terms of the mean absolute error, mean zero-one error, and mean squared error.
In the future work, we plan to consider a new regularization technique for the unbiased risk estimator approach in the context of semi-supervised learning so that the method can much more utilize the information of unlabeled data.

\section*{Acknowledgements}
The authors would like to thank Han Bao, Yusuke Konno, and Tomoya Sakai for helpful discussions,
and Kento Nozawa, Ikko Yamane, and Soma Yokoi for maintaining servers for our experiments.
TT was supported by JST AIP Challenge Program, UTokyo Toyota-Dwango AI Scholarship, and Softbank AI Scholarship. NC was supported by MEXT Scholarship, JST AIP Challenge Program, and Google PhD Fellowship Program. IS was supported by JST CREST Grant Number JPMJCR17A1, and MS was supported by JST CREST JPMJCR18A2.


\newpage
\appendix
\section{Proof of Lemma~\ref{thm:ssl-or-surrogate-risk}}
\label{sec:proof-of-ssl-or-surrogate-risk}

\begin{proof}
We can rewrite the task surrogate risk as follows:
\begin{align}\label{eq:proof-for-lem-1}
\calS(g) &= \expect_{X,Y} \left[ \psi(\balpha(X), Y) \right] \nonumber \\
&= \sum_{y=1}^{K}\pi_y \expect_{X|Y=y} \left[ \psi(\balpha(X), y) \right] \nonumber\\
&= \sum_{y\in\calYrmk} \pi_y \expect_{X|Y=y} \left[ \psi(\balpha(X), y) \right]
+\pi_k \expect_{X|Y=k}\left[ \psi(\balpha(X), k) \right].
\end{align}
By expanding the marginal distribution, we have
\begin{align}\label{eq:expand-unlabeled}
\expect_\u \left[ \psi(\balpha(X), k) \right] = \sum_{y=1}^{K}\pi_y \expect_{X|Y=y} \left[ \psi(\balpha(X), k) \right],
\end{align}
Then, we can express 
$\pi_k \expect_{X|Y=k}\left[ \psi(\balpha(X), k) \right]$ 
in terms of the expectation of unlabeled data and the class-conditional expectations of all classes except class $k$ as
\begin{align}
    \pi_k \expect_{X|Y=k}\left[ \psi(\balpha(X), k) \right] 
    = 
    \expect_\u \left[ \psi(\balpha(X), k) \right]
    - 
    \sum_{y\in\calYrmk} \pi_y \expect_{X|Y=y} \left[ \psi(\balpha(X), k) \right] .
    \label{eq:express-k-by-u}
\end{align}
Replacing the last term of the right-hand side of~\eqref{eq:proof-for-lem-1} with the right-hand side of~\eqref{eq:express-k-by-u}, we complete the proof.
\end{proof}

\section{Proof of Proposition~\ref{thm:consistency}}
\label{sec:proof-of-consistency}

\begin{proof}
  From Theorem~\ref{thm:ssl-cvx-surrogate-risk}, we have $\RiskSEMI(g) = \calS(g)$.
  Also it is known that $\calS(g)$ is Fisher-consistent to $\calR(g)$,
  if we adopt the all threshold, cumulative link, least absolute deviation, immediate threshold, or least squares as the task surrogate loss $\psi$~\citep{pedregosa:consistency}.
  By combining all of the results mentioned above, the proof is completed.
\end{proof}

\section{Technical Lemmas}
\label{sec:tech-lems}
The following lemma is used when deriving the estimation error bounds in Appendix~\ref{sec:proof-of-estimation-error-bound}.

\begin{lemma}[McDiarmid's inequality (Theorem D.8 in~\citet{mohri2018foundations})]
\label{lem:McDiarmid-ineq}
    Let $X_1, \dots, X_n$ be a set of independent random variables taking values in $\calX$ and assume that there exists $c_1, \dots, c_n > 0$ such that $ f : \calX^n \rightarrow \R$ satisfies
    \begin{align*}
        |f(x_1,\dots,x_i,\dots,x_n) - f(x_1,\dots,x'_i,\dots,x_n)| \le c_i
    \end{align*}
    for all $i \in \{1,\dots,n\}$ and any $x_1, \dots, x_n, x'_i \in \calX$.
    Then,
    \begin{align*}
        \mathbb{P} \left\{ {f(X_1, \dots, X_n) - \expect[f(X_1, \dots, X_n)] \ge \epsilon} \right\} &\le \exp\left( \frac{-2 \epsilon^2}{\sum_{i=1}^n c_i} \right)  \\
        \mathbb{P} \left\{ {f(X_1, \dots, X_n) - \expect[f(X_1, \dots, X_n)] \le -\epsilon} \right\} &\le \exp\left( \frac{-2 \epsilon^2}{\sum_{i=1}^n c_i} \right).
    \end{align*}
\end{lemma}

The following lemma is used when bounding the Rademacher complexities of hypothesis classes in Appendix~\ref{sec:proof-of-Rademacher-OR-linear}.

\begin{lemma}[Talagrand's lemma (Lemma 5.7 in~\citet{mohri2018foundations})]
\label{lem:Talagrand-lem}
    Let $\Phi : \R \rightarrow \R$ be $\rho$-Lipschitz function. Then, for any set $\calH \subset \R^\calX$,
    \begin{align*}
        \Rad(\Phi \circ \calH)
        \le
        \rho \, \Rad(\calH)
        .
    \end{align*}
\end{lemma}

\section{Analysis of Estimation Error Bounds (Proof of Theorem~\ref{thm:estimation-error-bound-general})}
\label{sec:proof-of-estimation-error-bound}

Before proving Theorem~\ref{thm:estimation-error-bound-general}, we define and recall notations and give some lemmas.
Defining
\begin{align*}
  \Rl(g) 
  \coloneqq  
  \RLone(g) - \RLtwo(g) 
  = 
  \sum_{y\in\calYrmk} \pi_y 
  \underbrace{\expect_{X|Y=y} \left( \psi(\balpha(X), y) -  \psi(\balpha(X), k) \right)}_{\eqqcolon \Rpsily(g)}, 
\end{align*}
we have
\begin{align*}
  \Rluk(g) 
  = 
  \Rl(g) + \Ru(g),
\end{align*}
where we recall that
\begin{align*}
    \calS_{\u}(g) 
    = 
    \expect_\u \left[ \psi(\balpha(X), k) \right]
    .   
\end{align*}
In a similar manner, we define the empirical version of each risk.
Defining
\begin{align*}
  \empRl(g) 
  \coloneqq
  \empRLone(g) - \empRLtwo(g) 
  = 
  \sum_{y\in\calYrmk} \pi_y
  \underbrace{\frac{1}{n_y} \sum_{j=1}^{n_y} (\psi(\balpha(\xyj), y) - \psi(\balpha(\xyj), k) )}_{\eqqcolon \empRpsily(g)},
\end{align*}
we have
\begin{align*}
  \empRluk(g) 
  = 
  \empRl(g) + \empRu(g),
\end{align*}
where we recall that
\begin{align*}
    \empRu(g) 
    = 
    \frac{1}{n_\u} \sum_{j=1}^{n_\u} \psi(\balpha(\xuj), k).
\end{align*}
First, we prove three lemmas in the following.
\begin{lemma}
  \label{lem:labeled-except-k-bound}
  Assume that there exists a constant $C_\psi > 0$ such that $\psi(y, \balpha) \leq C_\psi$ for any $y \in \R,\ \balpha \in \R^{K-1}$.
  Then, for any $\delta > 0$, with probability at least $1 - \frac{K-1}{K+1}\delta$,
  \begin{align}
    &
    \sup_{g \in \mathcal{G}} \left| \Rl(g) - \empRl(g) \right|  \nonumber \\
    &\quad\leq 
    2 \sum_{y\in\calYrmk} \pi_y \left( \Rad(\mathcal{H}^{(y)}; n_y) + \Rad(\mathcal{H}^{(k)}; n_y) \right) 
    + 
    2\sqrt{2} (K-1) C_\psi \sqrt{\ln{\frac{2(K+1)}{\delta}}} \sum_{y\in\calYrmk} \frac{\pi_y}{\sqrt{n_y}}
    .
  \end{align}
\end{lemma}

\begin{proof}
Recalling that
  \begin{align*}
    \empRpsily(g) 
    &=
    \frac{1}{n_y} \sum_{j=1}^{n_y} ( \psi(\balpha(\xyj), y) - \psi(\balpha(\xyj), k)),
  \end{align*}
we have
\begin{align}
    &\left| \sup_{g \in \mathcal{G}} \left( \Rpsily(g) - \empRpsily(g) \right) -  \sup_{g \in \mathcal{G}} \left( \Rpsily(g) - \empRpsily^m(g) \right) \right|  \nonumber \\
    &\leq 
    \sup_{g \in \mathcal{G}} \left| \empRpsily^m(g) - \empRpsily(g) \right| \nonumber \\
    &\leq 
    \sup_{g \in \mathcal{G}} \left| \frac{1}{n_y} \left[ \left( \psi(\balpha(\xym), y) - \psi(\balpha(\xym), k) \right) - \left( \psi(\balpha(\xymp), y) - \psi(\balpha(\xymp), k) \right) \right] \right| \nonumber \\
    &\leq 
    \frac{4 C_\psi}{n_y},
\end{align}
  where $\empRpsily^m(g)$ is the slight modification of $\empRpsily(g)$ where $\xym$ is changed to $\xymp$.
  Thus, by McDiarmid's inequality (Lemma~\ref{lem:McDiarmid-ineq} in Appendix~\ref{sec:tech-lems}),
  we have
  \begin{equation}
      \mathbb{P}\left\{\sup_g ( \Rpsily(g) - \empRpsily(g) ) - \expect_{S_y|Y=y} \left[ \sup_g ( \Rpsily(g) - \empRpsily(g)) \right] \geq \epsilon \right\}
      \leq \exp{\left( -\frac{\epsilon^2 n_y}{8 C_\psi^2} \right)}
      ,
  \end{equation}
  where $S_y | Y = y $ denotes the distribution over $\calX^{n_y}$ when conditioned as $Y=y$.
  Therefore, with probability at least $1-\frac{\delta}{2(K+1)}$, it holds that
  \begin{equation}
    \sup_g ( \Rpsily(g) - \empRpsily(g) )
    \leq
    \expect_{S_y|Y=y} \left[ \sup_g ( \Rpsily(g) - \empRpsily(g)) \right] + 2\sqrt{2} C_\psi \sqrt{\ln{\frac{2(K+1)}{\delta}}} \frac{1}{\sqrt{n_y}}
    .
    \label{eq:estim_y}
  \end{equation}
  The first term of the right-hand side of the above inequality can be bounded as
  \begin{align}
    &
    \expect_{S_y|Y=y} \left[ \sup_g ( \Rpsily(g) - \empRpsily(g)) \right]  \\
    &=
    \expect_{S_y|Y=y} \left[ \sup_g \left( \expect_{S'_y|Y=y} \left [ \empRpsily'(g) - \empRpsily(g) \right] \right) \right]  \nonumber \\ 
    &\leq
    \expect_{S_y, S'_y |Y=y} \left[ \sup_g \left( \empRpsily'(g) - \empRpsily(g) \right) \right]  
    \label{eq:k-rm-1}
    \\ 
    &=
    \expect_{S_y, S'_y |Y=y} \left[ \sup_g \frac{1}{n_y}  \sum_{j=1}^{n_y} \left( 
        ( \psi(\balpha(\xyjpr), y) - \psi(\balpha(\xyjpr), k) ) 
        - 
        ( \psi(\balpha(\xyj), y) - \psi(\balpha(\xyj), k) )
    \right) \right]  
    \nonumber \\ 
    &=
    \expect_{S_y, S'_y |Y=y} \expect_{\bsigma} \left[ \sup_g \frac{1}{n_y} \sum_{j=1}^{n_y} \sigma_j
    \left( 
        ( \psi(\balpha(\xyjpr), y) - \psi(\balpha(\xyjpr), k) ) 
        - 
        ( \psi(\balpha(\xyj), y) - \psi(\balpha(\xyj), k) )
    \right) \right]  
    \label{eq:k-rm-2}
    \\ 
    &\leq
    \expect_{S'_y|Y=y} \expect_{\bsigma}  \left[ \sup_g \frac{1}{n_y} \sum_{j=1}^{n_y} 
        \sigma_j \psi(\balpha(\xyjpr), y) 
     \right] 
    +
    \expect_{S'_y|Y=y} \expect_{\bsigma} \left[ \sup_g \frac{1}{n_y} \sum_{j=1}^{n_y} 
        (-\sigma_j) \psi(\balpha(\xyjpr), k) 
    \right]  \nonumber  \\
    &\qquad+
    \expect_{S_y |Y=y} \expect_{\bsigma} \left[ \sup_g \frac{1}{n_y} \sum_{j=1}^{n_y}
        (-\sigma_j)  \psi(\balpha(\xyj), y) 
    \right] 
    +
    \expect_{S_y |Y=y} \expect_{\bsigma} \left[ \sup_g \frac{1}{n_y} \sum_{j=1}^{n_y}  
        \sigma_j \psi(\balpha(\xyj), k) 
    \right]
    \label{eq:k-rm-3}
    \\ 
    &\leq 
    2 ( \Rad(\calH^{(y)}; n_y) +  \Rad(\calH^{(k)}; n_y) ) 
    ,
  \end{align}
  where~\eqref{eq:k-rm-1} and~\eqref{eq:k-rm-3} follow since the subadditivity of the supremum,~\eqref{eq:k-rm-2} follows from the fact that the distribution $S_y | Y=y$ and $S'_y | Y=y$ are the same,
  and in the last inequality we the fact that $-\sigma_j$ and $\sigma_j$ follow the same distribution.
  Repeating the same argument and together with~\eqref{eq:estim_y}, de Morgan's Laws and the union bound, wit probability at least $1-\frac{\delta}{K+1}$, we have
  \begin{equation}
      \sup_g \left| \Rpsily(g) - \empRpsily(g) \right|
      \leq
      2 ( \Rad(\calH^{(y)}; n_y) +  \Rad(\calH^{(k)}; n_y) ) + 2\sqrt{2} C_\psi \sqrt{\ln{\frac{2(K+1)}{\delta}}} \frac{1}{\sqrt{n_y}}.
  \end{equation}
   Therefore, Lemma~\ref{lem:labeled-except-k-bound} holds as a consequence of the inequality
   \begin{align*}
       \sup_{g \in \mathcal{G}} \left| \Rl(g) - \empRl(g) \right|
       =
       \sup_g \left| \sum_{y\in\calYrmk} \pi_y (\Rpsily(g) - \empRpsily(g)) \right| 
       \leq
       \sum_{y\in\calYrmk} \pi_y \sup_g \left| \Rpsily(g) - \empRpsily(g) \right|
       ,
   \end{align*}
   de Morgan's Laws, and the union bound.
\end{proof}

\begin{lemma}
  \label{lem:unlabeled-bound}
  Assume that there exists a constant $C_\psi > 0$ such that $\psi(y, \balpha) \leq C_\psi$ for any $y \in \calY,\ \balpha \in \R^{K-1}$.
  Then, for any $\delta > 0$, with probability at least $1 - \frac{\delta}{K+1}$,
  \begin{align*}
    \sup_{g \in \mathcal{G}} \left| \Ru(g) - \empRu(g) \right|
    \leq 2 \mathfrak{R}(\mathcal{H}^{(k)}; \nU) + \sqrt{2} C_\psi \sqrt{\ln{\frac{2(K+1)}{\delta}}} \frac{1}{\sqrt{\nU}}
    .
  \end{align*}
\end{lemma}

\begin{proof}
  This lemma can be proven similarly to Lemma~\ref{lem:labeled-except-k-bound}.
\end{proof}

\begin{lemma}
  \label{lem:sv-bound}
  Assume that there exists a constant $C_\psi > 0$ such that $\psi(y, \balpha) \leq C_\psi$ for any $y \in \calY,\ \balpha \in \R^{K-1}$.
  Then, for any $\delta > 0$, with probability at least $1 - \frac{\delta}{K+1}$,
  \begin{align*}
    \sup_{g \in \mathcal{G}} \left| \SVrisk(g) - \empSVrisk(g) \right|
    \leq 2 \mathfrak{R}(\mathcal{H}; n_{\mathrm{L}}) + \sqrt{2} C_\psi \sqrt{\ln{\frac{2(K+1)}{\delta}}} \frac{1}{\sqrt{\nL}}
    .
  \end{align*}
\end{lemma}

\begin{proof}
    This lemma can be proven similarly to Lemma~\ref{lem:labeled-except-k-bound}.
\end{proof}

\begin{proof}[Proof of Theorem~\ref{thm:estimation-error-bound-general}]
Let us bound the estimation error $\calS(\ghatk) - \calS(g^*)$.
Note that the equality $\RiskSEMI(g) = \calS(g)$ holds for any $g \in \mathcal{G}$.
Then, we have
\begin{align}
  \calS(\ghatk) - \calS(g^*)
  &=
  \RiskSEMI(\ghatk) - \RiskSEMI(g^*)
  \nonumber
  \\
  &= 
  \left( \RiskSEMI(\ghatk) - \empRiskSEMI(\ghatk) \right) 
  + 
  \left( \empRiskSEMI(\ghatk) - \empRiskSEMI(g^*)) \right) \nonumber \\
  & \qquad + \left(\empRiskSEMI(g^*) - \RiskSEMI(g^*) \right)
  \nonumber
  \\
  & \le \left(\RiskSEMI(\ghatk) - \empRiskSEMI(\ghatk) \right) + 0 + \left(\empRiskSEMI(g^*) - \RiskSEMI(g^*) \right)
  &&
  \label{eq:est-bound-1}
  \\
  & \le 2\sup_{g \in \mathcal{G}} \left| \RiskSEMI(g) - \empRiskSEMI(g) \right| \nonumber
  \\
  & \le 
  2 \gamma
  \left(
    \sup_{g \in \mathcal{G}} \left| \Rl(g) - \empRl(g) \right| 
  + 
    \sup_{g \in \mathcal{G}} \left| \Ru(g) - \empRu(g) \right|  
  \right)  \nonumber \\
  &\qquad+
  2 (1-\gamma) \sup_{g \in \mathcal{G}} \left| \calS(g) - \hat{\calS}(g) \right|
  ,
  \label{eq:su-risk-decomposition}
\end{align}
where~\eqref{eq:est-bound-1} follows since $\ghatk$ is the minimizer of $\empRiskSEMI$,
and the last inequality follows from the subadditivity of the supremum.
Hence, combining Eq.~\eqref{eq:su-risk-decomposition} and Lemmas~\ref{lem:labeled-except-k-bound} to~\ref{lem:sv-bound},
with probability at least $1 - \delta$, we have
\begin{align}
    &
    \calS(\ghatk) - \calS(g^*)   \nonumber \\
    &\le
    4 \left(
        \gamma
        \left(
            \sum_{y\in\calYrmk} \pi_y
            \left( \Rad(\mathcal{H}^{(y)}; n_y) + \Rad(\mathcal{H}^{(k)}; n_y) \right) 
            + 
            \Rad( \calH^{(k)}; n_{\mathrm{U}} )
        \right)
        +
        (1-\gamma) \Rad(\calH; n_{\mathrm{L}})
    \right)   \nonumber \\
    &\qquad+
    2\sqrt{2} C_\psi \sqrt{\ln\frac{2(K+1)}{\delta}}
    \left( \gamma \left(2\sum_{y\in\calYrmk} \frac{\pi_y}{\sqrt{n_y}} + \frac{1}{\sqrt{n_\u}} \right) + (1-\gamma) \frac{1}{\sqrt{n_{\mathrm{L}}}} \right) 
    .
\end{align}
Combining the above inequality with Theorems~\ref{thm:rad-bound-OR-general-y-dependent} and~\ref{thm:rad-bound-OR-general}, we completed the proof of Theorem~\ref{thm:estimation-error-bound-general}.
\end{proof}

\section{Upper Bounds for Rademacher Complexities in Ordinal Regression (Proof of Lemma~\ref{lem:OR-rad-bound-linear})}
\label{sec:proof-of-Rademacher-OR-linear}
Before proving Lemma~\ref{lem:OR-rad-bound-linear},
we prepare the following lemma and two theorems.
In this section, we write $[n] \coloneqq \{1, \dots, n\}$ for $n\in\N$.
Recall that 
\begin{align}
\begin{aligned}
    \calH_\AT &= \{ (\x, y) \mapsto \psi_\AT(\balpha(\x), y) : f\in\calF, \btheta\in\Theta \},   \\
    \calH_\IT &= \{ (\x, y) \mapsto \psi_\IT(\balpha(\x), y) : f\in\calF, \btheta\in\Theta \},   \\
    \calH_\LS &= \{ (\x, y) \mapsto \psi_\LS(\balpha(\x), y) : f\in\calF \}.  
\end{aligned}
\end{align}
and
\begin{align}
\begin{aligned}
    \calH_\AT^{(y)} &= \{ \x \mapsto \psi_\AT(\balpha(\x), y) : f\in\calF, \btheta\in\Theta \},   \\
    \calH_\IT^{(y)} &= \{ \x \mapsto \psi_\IT(\balpha(\x), y) : f\in\calF, \btheta\in\Theta \},   \\
    \calH_\LS^{(y)} &= \{ \x \mapsto \psi_\LS(\balpha(\x), y) : f\in\calF \}.  
\end{aligned}
\end{align}

\begin{lemma}
\label{lem:decomposing-hypothesis-class}
Let $n\in\N$ and $\calZ \subset \calY$, and assume that 
\begin{align}
    \Rad(\calH^{(y)}; n) = 0 
\end{align}
for $y \in \calZ$.
Then,
\begin{align}
    \Rad(\calH; n) \le \sum_{y\in\mathcal{Z}} \Rad(\calH^{(y)}; n)
    .
\end{align}
\end{lemma}

\begin{proof}
Using the identity $\ind{y=y_i} = \frac12 + \frac{2 \, \ind{y=y_i} - 1}{2}$, we have
\begin{align}
    \Rad(\calH ; n)
    &=
    \expect
    \left[  
        \sup_{h \in \calH} \frac1n \sum_{i\in[n]} \sigma_i h(\x_i, y_i)
    \right]  \nonumber  \\ 
    &=
    \expect
    \left[  
        \sup_{h \in \calH} \frac1n \sum_{i\in[n]} \sum_{y\in\calY} \sigma_i h(\x_i, y) \ind{y=y_i}
    \right]  \nonumber  \\ 
    &\le
    \expect
    \left[  
        \sup_{h \in \calH} \frac1n \sum_{i\in[n]} \sum_{y\in\calY} \sigma_i h(\x_i, y) \ind{y=y_i}
    \right]  
    \label{eq:decompse-1}
    \\ 
    &=
    \sum_{y\in\calY}
    \expect
    \left[  
        \sup_{h \in \calH} \frac1n \sum_{i\in[n]} \sigma_i h(\x_i, y) 
        \left( \frac12 + \frac{2 \, \ind{y=y_i} - 1}{2}  \right)
    \right]  \nonumber  \\ 
    &\le 
    \frac12 \sum_{y\in\calY}
    \expect
    \left[  
        \sup_{h \in \calH} \frac1n \sum_{i\in[n]} \sigma_i h(\x_i, y) 
    \right] 
    +
    \frac12 \sum_{y\in\calY}
    \expect
    \left[  
        \sup_{h \in \calH} \frac1n \sum_{i\in[n]} \sigma_i ( 2 \, \ind{y=y_i} - 1 ) h(\x_i, y) 
    \right] 
    \label{eq:decompse-2}
    \\
    &=
    \frac12 \sum_{y\in\calY}
    \expect
    \left[  
        \sup_{h \in \calH} \frac1n \sum_{i\in[n]} \sigma_i h(\x_i, y) 
    \right] 
    +
    \frac12 \sum_{y\in\calY}
    \expect
    \left[  
        \sup_{h \in \calH} \frac1n \sum_{i\in[n]} \sigma_i h(\x_i, y) 
    \right] 
    \label{eq:decompse-3}
    \\
    &=
    \sum_{y\in\calY} \Rad(\calH^{(y)}; n) 
    \nonumber \\
    &=
    \sum_{y\in\calZ} \Rad(\calH^{(y)}; n)
    ,
\end{align}
where~\eqref{eq:decompse-1} and~\eqref{eq:decompse-2} follow from the subadditivity of the supremum,
\eqref{eq:decompse-3} follows from the fact that $\sigma_i ( 2 \, \ind{y=y_i} - 1 ) $  and $\sigma_i$ follow the same distribution,
and the last inequality follows from the assumption.
\end{proof}

Next, we give the upper bounds for the Rademacher complexities with $\calH_\AT^{(y)}$, $\calH_\IT^{(y)}$, and $\calH_\LS^{(y)}$ without any assumptions on the decision functions class $\calF$.
Let us recall the theorem for the convenience.
\begingroup
\def\thetheorem{\ref{thm:rad-bound-OR-general-y-dependent}}
\begin{theorem}
    Fix $y\in\calY$. Let $n \in \N$ assume for the AT and IT losses that the binary surrogate loss $\ell$ is $\rho$-Lipschitz.
    Then, the expected Rademacher complexities of $\calH_\AT^{(y)}, \calH_\IT^{(y)}$, and $\calH_\LS^{(y)}$ are bounded as
    \begin{align}
        \begin{aligned}
        \Rad(\calH_\AT^{(y)} ; n) 
        &\le 
        \rho \sum_{j=1}^{K-1} \Rad(\calA_j; n)   
        ,  \\
        \Rad(\calH_\IT^{(y)} ; n) 
        &\le 
        \rho \left( \Rad(\calA_{y-1}; n) + \Rad(\calA_{y}; n) \right)
        , \\
        \Rad(\calH_\LS^{(y)} ; n) 
        &\le 
        2 |y + \theta_1 - 3/2| \Rad(\calF; n) + \Rad(\mathrm{sq} \circ \calF; n)
        ,
        \end{aligned}
    \end{align}
    where $\mathrm{sq} : \mathrm{Im}\,f \ni z \mapsto z^2 \in \R$.
\end{theorem}
\addtocounter{theorem}{-1}
\endgroup

\begin{proof}

We prove the theorem for each loss.

\noindent \textbf{AT loss.}
We have
\begin{align}
    \Rad(\calH_\AT^{(y)} ; n)
    &=
    \expect
    \left[  
        \sup_{h \in \calH_\AT^{(y)}} \frac1n \sum_{i\in[n]} \sigma_i h(\x_i, y)
    \right]  \nonumber  \\ 
    &=
    \expect
    \left[  
        \sup_{f\in\calF, \btheta\in\Theta} \frac1n \sum_{i\in[n]} \sigma_i 
        \left(
            \sum_{j=1}^{y-1} \ell(-\alpha_j(\x_i)) + \sum_{j=y}^{K-1} \ell(\alpha_j(\x_i)) 
        \right)
    \right] \nonumber  \\ 
    &\le
    \expect
    \left[  
        \sup_{f\in\calF, \btheta\in\Theta} \frac1n \sum_{i\in[n]} \sigma_i 
        \left(
            \sum_{j=1}^{y-1} \ell(-\alpha_j(\x_i)) 
        \right)
    \right]
    +
    \expect
    \left[  
        \sup_{f\in\calF, \btheta\in\Theta} \frac1n \sum_{i\in[n]} \sigma_i 
        \left(
            \sum_{j=y}^{K-1} \ell(\alpha_j(\x_i)) 
        \right)
    \right]
    ,  
    \label{eq:Rad-AT-1-y-dep}
    \\
    &\le
    \sum_{j=1}^{y-1} 
    \expect
    \left[  
        \sup_{f\in\calF, \btheta\in\Theta} \frac1n \sum_{i\in[n]} \sigma_i 
        \left(
            \ell(-\alpha_j(\x_i)) 
        \right)
    \right]
    +
    \sum_{j=y}^{K-1} 
    \expect
    \left[  
        \sup_{f\in\calF, \btheta\in\Theta} \frac1n \sum_{i\in[n]} \sigma_i 
        \left(
            \ell(\alpha_j(\x_i)) 
        \right)
    \right]
    ,  
    \label{eq:Rad-AT-2-y-dep} 
    \\
    &\le
    \sum_{j=1}^{K-1}
    \Rad(\ell \circ \calA_j; n)  
    \label{eq:Rad-AT-3-y-dep} 
    \\
    &\le
    \rho
    \sum_{j=1}^{K-1}
    \Rad(\calA_j; n)   
    ,
\end{align}
where~\eqref{eq:Rad-AT-1-y-dep} and~\eqref{eq:Rad-AT-2-y-dep} follow from the subadditivity of the supremum,
\eqref{eq:Rad-AT-3-y-dep} follow from the fact that $-\sigma_i$ and $\sigma_i$ follows the same distribution,
and the last inequality follows from the Talagrand's lemma (Lemma~\ref{lem:Talagrand-lem} in Appendix~\ref{sec:tech-lems}).

\noindent \textbf{IT loss.}
The proof for the case of the IT loss follows a very similar argument to that of the AT loss,
and we include the proof for completeness.
We have
\begin{align}
    \Rad(\calH_\IT^{(y)} ; n)
    &=
    \expect
    \left[  
        \sup_{h\in\calH_\IT^{(y)}} \frac1n \sum_{i\in[n]} \sigma_i 
        \left(
            \ell(-\alpha_{y-1}(\x_i)) + \ell(\alpha_{y}(\x_i)) 
        \right)
    \right] \nonumber  \\ 
    &\le
    \expect
    \left[  
        \sup_{f\in\calF, \btheta\in\Theta} \frac1n \sum_{i\in[n]} \sigma_i 
            \ell(-\alpha_{y-1}(\x_i)) 
    \right]  
    +
    \expect
    \left[  
        \sup_{f\in\calF, \btheta\in\Theta} \frac1n \sum_{i\in[n]} \sigma_i 
        \ell(\alpha_{y}(\x_i)) 
    \right]
    \label{eq:Rad-IT-1-y-dep}
    ,
    \\
    &\le
    \Rad(\ell \circ (-\calA_{y-1}); n)  
    +
    \Rad(\ell \circ \calA_{y}; n)  
    \label{eq:Rad-IT-2-y-dep} 
    \\
    &\le
    \rho ( \Rad(\calA_{y-1}; n) + \Rad(\calA_y; n) ) 
    ,
\end{align}
where~\eqref{eq:Rad-IT-1-y-dep} and~\eqref{eq:Rad-IT-2-y-dep} follows from the subadditivity of the supremum,
the last inequality follows from the fact that $-\sigma_i$ and $\sigma_i$ follows the same distribution and the Talagrand's lemma.

\noindent \textbf{LS loss.}
We give the upper bound for the case of the LS loss.
We have
\begin{align}
    \Rad(\calH_\LS^{(y)} ; n)
    &=
    \expect
    \left[  
        \sup_{h\in\calH_\LS^{(y)}} \frac1n \sum_{i\in[n]} \sigma_i 
        h(\alpha(\x_i), y)
    \right] \nonumber  \\ 
    &=
    \expect
    \left[  
        \sup_{f\in\calF} \frac1n \sum_{i\in[n]} \sigma_i 
        \left(
            y - f(\x_i) + \theta_1 - \frac32
        \right)^2
    \right] \nonumber  \\ 
    &\le
    \expect
    \left[  
        \frac1n \sum_{i\in[n]} \sigma_i 
        \left(
            y + \theta_1 - \frac32
        \right)
    \right] 
    +
    \expect
    \left[  
        \sup_{f\in\calF} \frac1n \sum_{i\in[n]} \sigma_i f(\x_i)^2
    \right]  \nonumber \\
    &\qquad+
    2
    \expect 
    \left[  
        \sup_{f\in\calF} \frac1n \sum_{i\in[n]} (-\sigma_i)
        f(\x_i) (y + \theta_1 - 3/2)
    \right] 
    .
    \label{eq:Rad-LS-1-dep-y}
\end{align}
The first term in~\eqref{eq:Rad-LS-1-dep-y} is 0.
The third term in~\eqref{eq:Rad-LS-1-dep-y} can be rewritten as
\begin{align}
    &
    2
    \expect
    \left[  
        \sup_{f\in\calF} \frac1n \sum_{i\in[n]} (-\sigma_i)
        f(\x_i) (y + \theta_1 - 3/2)
    \right]  \nonumber \\
    &=
    2
    \expect
    \left[  
        \sup_{f\in\calF} \frac1n \sum_{i\in[n]} (-\sigma_i)\, \mathrm{sign}(y + \theta_1 - 3/2)
        |y + \theta_1 - 3/2|   f(\x_i)
    \right]  \nonumber  \\ 
    &=
    2 |y + \theta_1 - 3/2| 
    \expect
    \left[  
        \sup_{f\in\calF} \frac1n \sum_{i\in[n]} (-\sigma_i)\, \mathrm{sign}(y + \theta_1 - 3/2)
         f(\x_i)
    \right]  \nonumber  \\ 
    &=
    2 |y + \theta_1 - 3/2| 
    \expect
    \left[  
        \sup_{f\in\calF} \frac1n \sum_{i\in[n]} \sigma_i
         f(\x_i)
    \right]  \label{eq:Rad-LS-2-y-dep} \\
    &=
    2 |y + \theta_1 - 3/2| \; \Rad(\calF; n)
    ,
\end{align}
where in~\eqref{eq:Rad-LS-2-y-dep} we used the fact that $-\sigma_i \, \mathrm{sign}(y + \theta_1 - 3/2)$ and  $\sigma_i$ follow the same distribution.
Summing up the above argument, the proof is completed.
\end{proof}

Next, we give the upper bounds for the Rademacher complexities with $\calH_\AT$, $\calH_\IT$, and $\calH_\LS$ without any assumptions on a class of decision functions $\calF$
using Lemma~\ref{lem:decomposing-hypothesis-class} and Theorem~\ref{thm:rad-bound-OR-general-y-dependent}
Let us recall the theorem for the convenience.
\begingroup
\def\thetheorem{\ref{thm:rad-bound-OR-general}}
\begin{theorem}[Upper bounds of Rademacher complexities with general decision functions class]
    Let $n \in \N$ and
    assume for the AT and IT losses that the binary surrogate loss $\ell$ is $\rho$-Lipschitz.
    Then, the expected Rademacher complexities of $\calH_\AT, \calH_\IT$, and $\calH_\LS$ are bounded as
    \begin{align}
        \begin{aligned}
        \Rad(\calH_\AT ; n) 
        &\le 
        \rho \, K \sum_{j=1}^{K-1} \Rad(\calA_j; n)   
        ,  
        \\
        \Rad(\calH_\IT ; n) 
        &\le 
        \rho \, \sum_{y\in\calY} 
        \left( \Rad(\calA_{y-1}; n) + \Rad(\calA_{y}; n) \right), 
        \\
        \Rad(\calH_\LS ; n) 
        &\le 
        2 (K + |\theta_1 - 3/2|) \Rad(\calF; n) + \Rad(\mathrm{sq} \circ \calF; n)
        ,
        \end{aligned}
    \end{align}
    where $\mathrm{sq} : \mathrm{Im}\,f \ni z \mapsto z^2 \in \R$.
\end{theorem}
\addtocounter{theorem}{-1}
\endgroup

\begin{proof}[Proof of Theorem~\ref{thm:rad-bound-OR-general}]
For the AT and IT losses, the result of Theorem~\ref{thm:rad-bound-OR-general} directly follows from Lemma~\ref{lem:decomposing-hypothesis-class} and Theorem~\ref{thm:rad-bound-OR-general-y-dependent}.

\noindent \textbf{AT loss.}
For the case of the AT loss, we have
\begin{align}
    \Rad(\calH_\AT ; n) 
    \le 
    \sum_{y\in\calY} \Rad(\calH_\AT^{(y)}; n) 
    \le 
    \rho \,
    \sum_{y\in\calY} 
    \sum_{j=1}^{K-1} \Rad(\calA_j; n)   
    &=
    \rho \, K \sum_{j=1}^{K-1} \Rad(\calA_j; n) 
    ,
\end{align}
where in the first inequality we used Lemma~\ref{lem:decomposing-hypothesis-class},
and in the last inequality we used Theorem~\ref{thm:rad-bound-OR-general-y-dependent}.

\noindent \textbf{IT loss.}
In a similar manner, we can bound for the case of IT loss as
\begin{align}
    \Rad(\calH_\IT ; n) 
    \le 
    \sum_{y\in\calY} \Rad(\calH_\IT^{(y)}; n) 
    \le 
    \rho \,
    \sum_{y\in\calY} 
    \left( \Rad(\calA_{y-1}; n) + \Rad(\calA_{y}; n) \right)
    .
\end{align}

\noindent \textbf{LS loss.}
Next, we give the upper bound for the case of the LS loss, where we will not rely on Lemma~\ref{lem:decomposing-hypothesis-class} and Theorem~\ref{thm:rad-bound-OR-general-y-dependent}.
We have
\begin{align}
    \Rad(\calH_\LS ; n)
    &=
    \expect
    \left[  
        \sup_{h\in\calH_\LS} \frac1n \sum_{i\in[n]} \sigma_i 
        h(\alpha(\x_i), y_i)
    \right] \nonumber  \\ 
    &=
    \expect
    \left[  
        \sup_{f\in\calF} \frac1n \sum_{i\in[n]} \sigma_i 
        \left(
            y_i - f(\x_i) + \theta_1 - \frac32
        \right)^2
    \right] \nonumber  \\ 
    &\le
    \expect
    \left[  
        \frac1n \sum_{i\in[n]} \sigma_i 
        \left(
            y_i^2 + 2 y_i(\theta_1 - 3/2) + (\theta_1 - 3/2)^2
        \right)
    \right] 
    +
    2
    \expect
    \left[  
        \sup_{f\in\calF} \frac1n \sum_{i\in[n]} (- \sigma_i)
        y_i f(\x_i)
    \right] 
    \nonumber  \\ 
    &\qquad+ 
    \expect
    \left[  
        \sup_{f\in\calF} \frac1n \sum_{i\in[n]} \sigma_i f(\x_i)^2
    \right] 
    +
    2
    \expect
    \left[  
        \sup_{f\in\calF} \frac1n \sum_{i\in[n]} (-\sigma_i)
        f(\x_i) (\theta_1 - 3/2)
    \right] 
    .
    \label{eq:Rad-LS-1}
\end{align}
The first term in~\eqref{eq:Rad-LS-1} is 0.
The second term in~\eqref{eq:Rad-LS-1} is bounded as 
\begin{align}
    2
    \expect
    \left[  
        \sup_{f\in\calF} \frac1n \sum_{i\in[n]} (- \sigma_i)
        y_i f(\x_i)
    \right] 
    \le
    2 K
    \expect
    \left[  
        \sup_{f\in\calF} \frac1n \sum_{i\in[n]} \sigma_i f(\x_i)
    \right] 
    =
    2 K \Rad(\calF; n)
    ,
\end{align}
since $\sigma_i$ and $-\sigma_i$ follow the same distribution.
The forth term in~\eqref{eq:Rad-LS-1} can be rewritten as
\begin{align}
    2
    \expect
    \left[  
        \sup_{f\in\calF} \frac1n \sum_{i\in[n]} (-\sigma_i)
        f(\x_i) (\theta_1 - 3/2)
    \right] 
    &=
    2
    \expect
    \left[  
        \sup_{f\in\calF} \frac1n \sum_{i\in[n]} (-\sigma_i)\, \mathrm{sign}(\theta_1 - 3/2)
        |\theta_1 - 3/2|  f(\x_i)
    \right]  \nonumber  \\ 
    &=
    2 |\theta_1 - 3/2|
    \expect
    \left[  
        \sup_{f\in\calF} \frac1n \sum_{i\in[n]} (-\sigma_i)\, \mathrm{sign}(\theta_1 - 3/2)
         f(\x_i)
    \right]  \nonumber  \\ 
    &=
    2 |\theta_1 - 3/2|
    \expect
    \left[  
        \sup_{f\in\calF} \frac1n \sum_{i\in[n]} \sigma_i
         f(\x_i)
    \right]  \label{eq:Rad-LS-2} \\
    &=
    2 |\theta_1 - 3/2| \; \Rad(\calF; n)
    ,
\end{align}
where in~\eqref{eq:Rad-LS-2} we used the fact that $-\sigma_i \, \mathrm{sign}(\theta_1 - 3/2)$ and $\sigma_i$ follow the same distribution.
Summing up the above argument, the proof is completed.
\end{proof}

\begin{remark}
The upper bound for the case of the LS loss can be similarly obtained as the cases of the AT and IT losses.
However, it loosen the bound by the factor of $K$ compared to the analysis given above.
\end{remark}

The next lemma is the well-known bound for the Rademacher complexity of the linear-in-parameter models.
\begin{lemma}[Theorem 5.10 in \citet{mohri2018foundations}]
\label{lem:rad-bound-linear}
    Let $\calF$ be the class of linear-in-parameter models, that is,
    \begin{align}
        \mathcal{F} = \{f(\x) = \bw^\top\bm{\phi}(\x) \colon \|\bw\|_2 \le C_{\bw},\; \|\bm{\phi}(\x)\|_2 \le C_{\bm{\phi}} \}
    \end{align}
    for positive constants $C_{\bw}$ and $C_{\bm{\phi}}$.
    Then, 
    \begin{align}
        \Rad(\mathcal{F};n) 
        \le 
        \frac{C_{\bw} C_{\bphi}}{\sqrt{n}}
        .
        \label{eq:model-assumption}
    \end{align}
\end{lemma}

Finally, we prove the Lemma~\ref{lem:OR-rad-bound-linear} in the following.
\begin{proof}[Proof of Lemma~\ref{lem:OR-rad-bound-linear}]
The result of Lemma~\ref{lem:OR-rad-bound-linear} directly follows from Theorems~\ref{thm:rad-bound-OR-general-y-dependent},~\ref{thm:rad-bound-OR-general} and Lemma~\ref{lem:rad-bound-linear} as follows.

\noindent \textbf{AT and IT losses.}
Fix $j \in [K]$.
Being aware that we have
\begin{align}
    \alpha_j(\x) = \theta_j - w^\top \phi(\x) 
    = 
    (\theta_j, \bw^\top)
    \begin{pmatrix}
        1  \\
        -\x
    \end{pmatrix}
    ,
\end{align}
and 
\begin{align}
    \| (\theta_j, \bw^\top) \|_2 
    \le 
    \|\btheta\|_2 + \|\bw\|_2
    \le
    C_{\btheta} + C_{\bw}
    ,
\end{align}
we can use Lemma~\ref{lem:rad-bound-linear} and have
\begin{align}
    \Rad(\calA_j; n) 
    \le
    \frac{(C_{\btheta} + C_{\bw}) \, C_{\bphi}}{\sqrt{n}}
    .
\end{align}
In a similar manner, using Theorem~\ref{thm:rad-bound-OR-general}, we have
\begin{align}
    \Rad(\calH_\AT^{(y)} ; n) 
    \le
    \rho\,K \Rad(\calA_j; n)
    \le
    \frac{(C_{\btheta} + C_{\bw}) \, C_{\bphi} \rho K}{\sqrt{n}}
\end{align}
and
\begin{align}
    \Rad(\calH_\IT^{(y)} ; n) 
    \le
    \rho (\Rad(\calA_{y-1}; n) + \Rad(\calA_y; n))
    \le
    \frac{2 (C_{\btheta} + C_{\bw}) \, C_{\bphi} \rho }{\sqrt{n}}
    ,
\end{align}
which completes the proof for the cases of the AT and IT losses.

\noindent \textbf{LS loss.}
Since $|f(\x)| \le \|w\|_2 \|\bphi(\x)\|_2 \le C_{\bw} C_{\bphi}$ by the Cauchy–Schwarz inequality,
the map $\mathrm{sq} : \mathrm{Im}\,f \rightarrow \R$ is $2 C_{\bw} C_{\bphi}$-Lipschitz.
Therefore, using Theorem~\ref{thm:rad-bound-OR-general}  and the Talagrand's lemma,
we have
\begin{align}
    \Rad(\calH_\LS^{(y)} ; n) 
    &\le 
    2 |y + \theta_1 - 3/2| \Rad(\calF; n) + \Rad(\mathrm{sq} \circ \calF; n)  \nonumber \\
    &\le
    2 |y + \theta_1 - 3/2|\Rad(\calF; n) + 2 C_{\bw} C_{\bphi} \Rad(\calF; n)  \nonumber \\
    &\le
    \frac{2 (|y + \theta_1 - 3/2|+ C_{\bw} C_{\bphi} ) (C_{\btheta} + C_{\bw}) \, C_{\bphi}  K  }{\sqrt{n}}
    ,
\end{align}
where in the last inequality we used Lemma~\ref{lem:rad-bound-linear}.

Using similar manners as the above arguments,
we can obtain the upper bounds for the cases with $\calH_\AT$, $\calH_\IT$, and $\calH_\LS$.
\end{proof}

\section{Proof of Theorem~\ref{thm:variance-reduction}}
\label{sec:proof-of-variance-reduction}

\begin{proof}
    Take any $g \in \calG$.
    Recalling that $\empsslrisk(g) = \gamma \empRluk + (1-\gamma) \empSVrisk$,
    we can rewrite the variance of the empirical semi-supervised risk $\Var[\empsslrisk(g)]$ as
    \begin{align}
        \Var[\empsslrisk(g)]
        &=
        \gamma^2 \Var[\empRluk(g)] + (1-\gamma)^2 \Var[\empSVrisk(g)] + 2\gamma (1-\gamma) \Cov(\empRluk(g), \empSVrisk(g))
        .
    \end{align}
    Hence, 
    the condition 
    $\Var[\empsslrisk(g)] < \Var[\empSVrisk(g)]$
    is equivalent to 
    \begin{align}
        (\Var[\empRluk(g)] &+ \Var[\empSVrisk(g)] - 2 \Cov(\empRluk(g), \empSVrisk(g)) ) \gamma^2 \nonumber \\
        &
        - 2 ( \Var[\empRluk(g)] -  \Cov(\empRluk(g), \empSVrisk(g)) ) \gamma < 0
        .
    \end{align}
    This condition is indeed satisfied by taking $\gamma$ satisfying
    \begin{align}
        0 < \gamma < \frac{2 \left (\Var[\empRluk(g)] - \Cov(\empRluk(g), \empSVrisk(g)) \right) }{\Var[\empRluk(g)] + \Var[\empSVrisk(g)] - 2 \Cov(\empRluk(g), \empSVrisk(g)) }
        ,
    \end{align}
    and the proof is completed.
\end{proof}

\section{Proof of Theorem~\ref{thm:convex-ssl-or}} 
\label{sec:proof-of-convex-ssl-or}
\begin{proof}
If $\empRl(g)$ is shown to be convex, the objective function 
$\hat{J}_\ell(\bw, \btheta) = \hat{\calS}_{\SEMI\text{-}\gamma}^{\backslash k}(g) + \Omega(\btheta)$
is convex with respect to $\bw$ and $\btheta$
since the other terms are convex.

\noindent \textbf{AT loss.}
Recalling that the AT loss is defined as
\begin{align}
    \psi_\AT (\balpha(\x), y) 
    = 
    \sum_{i=1}^{y-1} \ell(-\alpha_i(\x)) + \sum_{i=y}^{K-1} \ell(\alpha_i(\x))
    ,
\end{align}
we have
\begin{align}
  \empRl(g)
  &= \sum_{y\in\calYrmk} \frac{{\pi}_y}{n_y} \sum_{j=1}^{n_y} \left( \psi(\balpha(\xyj), y) - \psi(\balpha(\xyj), k) \right) \nonumber \\
  &= \sum_{y\in\calYrmk} \frac{{\pi}_y}{n_y} \sum_{j=1}^{n_y} \left( \sum_{i=1}^{y-1} \ell(-\alpha_i(\xyj)) + \sum_{i=y}^{K-1} \ell(\alpha_i(\xyj)) - \sum_{i=1}^{k-1} \ell(-\alpha_i(\xyj)) - \sum_{i=k}^{K-1} \ell(\alpha_i(\xyj))\right)
  .
\end{align}
Then, the term inside the sum can be rewritten as
\begin{align}
&
  \sum_{i=1}^{y-1} \ell(-\alpha_i(\xyj)) + \sum_{i=y}^{K-1} \ell(\alpha_i(\xyj)) - \sum_{i=1}^{k-1} \ell(-\alpha_i(\xyj)) - \sum_{i=k}^{K-1} \ell(\alpha_i(\xyj)) \nonumber \\
  &= \begin{cases}
       \sum_{i=y}^{k-1} \left( -\ell(-\alpha_i(\xyj)) + \ell(\alpha_i(\xyj)) \right) & (y < k) \\
       - \sum_{i=k}^{y-1} \left( -\ell(-\alpha_i(\xyj)) + \ell(\alpha_i(\xyj)) \right) & (y > k) \\
       0 & (y = k)
     \end{cases} \nonumber \\
  &= \begin{cases}
       -C_\ell \sum_{i=y}^{k-1} \alpha_i(\xyj)  & (y < k) \\
       C_\ell \sum_{i=k}^{y-1} \alpha_i(\xyj)  & (y > k) \\
       0 & (y = k).
     \end{cases}
\end{align}
Therefore, $\empRl(g)$ is convex and the objective function $\hat{J}_{\ell}(\bw, \btheta)$ is convex.

\noindent \textbf{LS loss.}
Recalling that the LS loss is defined as 
\begin{align*}
    \psi_\LS (\balpha(\x), y) = \left(y + \alpha_1(\x) - \frac{3}{2} \right)^2
    ,
\end{align*}
we have
\begin{align}
  \empRl(g)
  &= \sum_{y\in\calYrmk} \frac{{\pi}_y}{n_y} \sum_{j=1}^{n_y} \left( \psi(\balpha(\xyj), y) - \psi(\balpha(\xyj), k) \right) \nonumber \\
  &= \sum_{y\in\calYrmk} \frac{{\pi}_y}{n_y} \sum_{j=1}^{n_y} \left\{  \left(y + \alpha_1(\xyj) - \frac{3}{2} \right)^2 -  \left(k + \alpha_1 - \frac{3}{2} \right)^2 \right\} \nonumber \\
  &= \sum_{y\in\calYrmk} \frac{{\pi}_y}{n_y} \sum_{j=1}^{n_y} \left\{  y^2 + k^2 + 2(y+k) \left(\alpha_1(\xyj) - \frac{3}{2}\right) \right\}
  .
\end{align}
This is convex and subsequently the objective function $\hat{J}_{\ell}(\bw, \btheta)$ is convex, and the proof is completed.
\end{proof}

\section{Dataset Statistics and Additional Results of Experiments}
\label{sec:additional_experiments}

\subsection{Benchmark Dataset Statistics}
\label{subsec:dataset-stats}
The detailed dataset statistics used in the small and large scale experiments are given in Tables~\ref{tab:dataset-stats-small} and~\ref{tab:dataset-stats-large}.
The other statistics is given in Section~\ref{sec:experiments}.

\begin{table}[ht]
\centering
\caption{
Dataset statistics of small scale experiments.
}
\label{tab:dataset-stats-small}
\scalebox{0.9}[0.9]{
\begin{tabular}{lrrrl}
\toprule
{} &  dim. &  \# of unlabeled data &  \# of test data &            class prior \\
\midrule
{\tt abalone}         &   10 &                  871 &            1253 &     [0.2, 0.6, 0.2] \\
{\tt bank1-5}         &    8 &                 1851 &            2000 &     [0.2, 0.6, 0.2] \\
{\tt bank2-5}         &   32 &                 1851 &            2000 &     [0.2, 0.6, 0.2] \\
{\tt census1-5}       &    8 &                 2000 &            2000 &     [0.2, 0.6, 0.2] \\
{\tt census2-5}       &   16 &                 2000 &            2000 &     [0.2, 0.6, 0.2] \\
{\tt computer1-5}     &   12 &                 1851 &            2000 &     [0.2, 0.6, 0.2] \\
{\tt computer2-5}     &   21 &                 1851 &            2000 &     [0.2, 0.6, 0.2] \\
{\tt fireman-example} &    9 &                 2000 &            2000 &   [0.5, 0.25, 0.25] \\
{\tt kinematics}      &    7 &                 1851 &            2000 &  [0.38, 0.38, 0.25] \\
{\tt lev}             &    4 &                  204 &             300 &  [0.09, 0.68, 0.22] \\
{\tt swd}             &   10 &                  204 &             300 &   [0.38, 0.4, 0.22] \\
{\tt toy}             &    2 &                   57 &              90 &   [0.12, 0.78, 0.1] \\
\bottomrule
\end{tabular}
} 
\end{table}

\begin{table}[ht]
\centering
\caption{
Dataset statistics of large scale experiments.
}
\label{tab:dataset-stats-large}
\scalebox{0.9}[0.9]{
\begin{tabular}{lrl}
\toprule
{} &  \# of class &                                  class prior \\
\midrule
{\tt bank1-5}         &                 5 &                          [0.2, 0.2, 0.2, 0.2, 0.2] \\
{\tt bank2-5}         &                 5 &                          [0.2, 0.2, 0.2, 0.2, 0.2] \\
{\tt census1-5}       &                 5 &                          [0.2, 0.2, 0.2, 0.2, 0.2] \\
{\tt census2-5}       &                 5 &                          [0.2, 0.2, 0.2, 0.2, 0.2] \\
{\tt computer1-5}     &                 5 &                          [0.2, 0.2, 0.2, 0.2, 0.2] \\
{\tt computer2-5}     &                 5 &                          [0.2, 0.2, 0.2, 0.2, 0.2] \\
{\tt fireman-example} &                16 &   around 0.0625 for all classes \\
{\tt kinematics}      &                 8 &   around 0.125 for all classes \\
\bottomrule
\end{tabular}
} 
\end{table}

\subsection{Additional Results of Experiments}
\label{subsec:additional_experiments}

Here we show the extended experimental results for the variance reduction experiments discussed in Section~\ref{sec:experiments}. 
Figures~\ref{fig:variance_reduction_AT_linear-in-input} to~\ref{fig:variance_reduction_LS_linear-in-input}
show the results of the variance comparison.
We can see that the variances of the semi-supervised risks are smaller than that of the supervised risk.
Figures~\ref{fig:nu_change_AT_linear-in-input} to~\ref{fig:nu_change_LS_linear-in-input} show that the effect of the number of unlabeled on the performance improvement.
Note that the experiments on {\tt toy} dataset is not conducted because its number of unlabeled data is too small.
We can see that the performance of the proposed method is indeed improved, but the improvement by adding a large amount of data is limited except when using LS as the task surrogate loss.


\begin{figure*}[htb]
    \begin{center}
    \setlength{\subfigwidth}{.32\linewidth}
    \addtolength{\subfigwidth}{-.32\subfigcolsep}
    \begin{minipage}[t]{\subfigwidth}
        \centering
         \subfigure[{\tt abalone}, AT]{\includegraphics[scale=0.40]{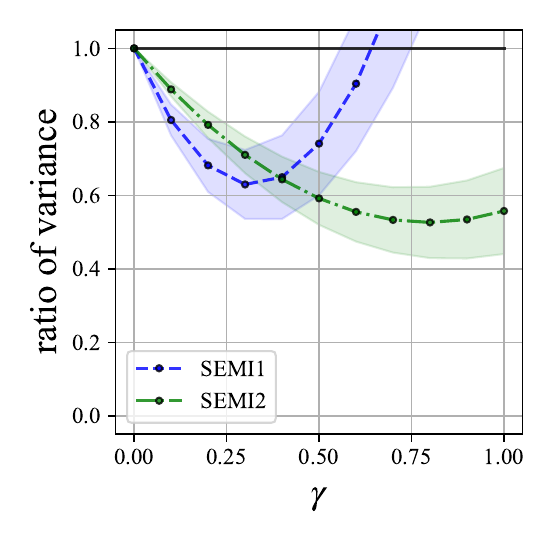}}
    \end{minipage}\hfill
    \begin{minipage}[t]{\subfigwidth}
        \centering
         \subfigure[{\tt bank1-5}, AT]{\includegraphics[scale=0.40]{fig/variance_reduction/sv_AT_bank1-5.pdf}}
    \end{minipage}\hfill
    \begin{minipage}[t]{\subfigwidth}
        \centering
         \subfigure[{\tt bank2-5}, AT]{\includegraphics[scale=0.40]{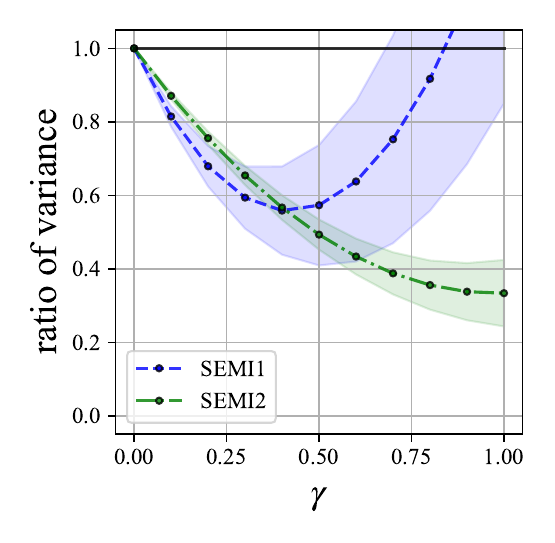}}
    \end{minipage}\hfill
    \begin{minipage}[t]{\subfigwidth}
        \centering
         \subfigure[{\tt census1-5}, AT]{\includegraphics[scale=0.40]{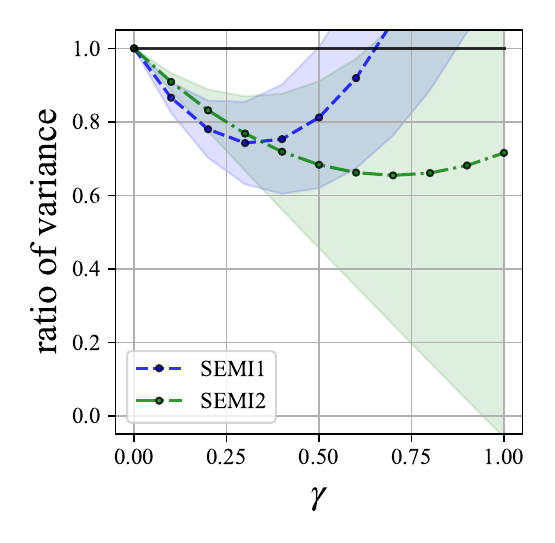}}
    \end{minipage}\hfill
    \begin{minipage}[t]{\subfigwidth}
        \centering
         \subfigure[{\tt census2-5}, AT]{\includegraphics[scale=0.40]{fig/variance_reduction/sv_AT_census2-5.pdf}}
    \end{minipage}\hfill
    \begin{minipage}[t]{\subfigwidth}
        \centering
         \subfigure[{\tt computer1-5}, AT]{\includegraphics[scale=0.40]{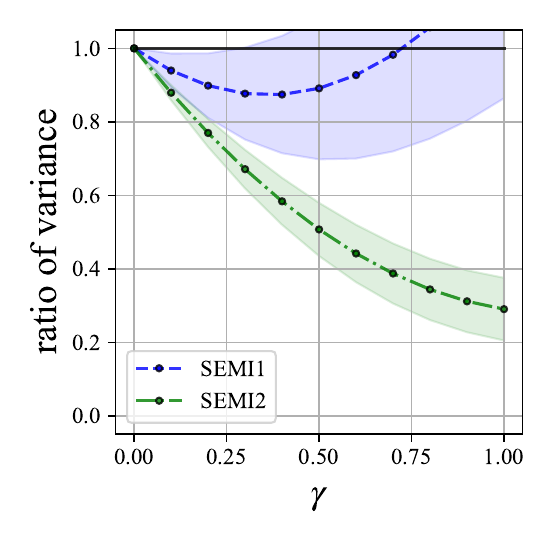}}
    \end{minipage}\hfill
    \begin{minipage}[t]{\subfigwidth}
        \centering
         \subfigure[{\tt computer2-5}, AT]{\includegraphics[scale=0.40]{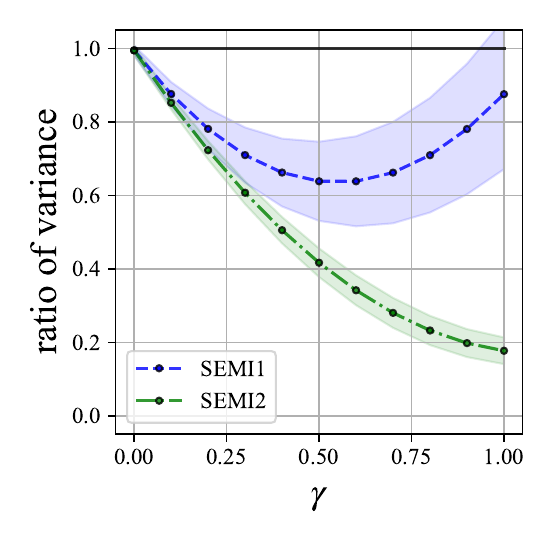}}
    \end{minipage}\hfill
    \begin{minipage}[t]{\subfigwidth}
        \centering
         \subfigure[{\tt fireman-example}, AT]{\includegraphics[scale=0.40]{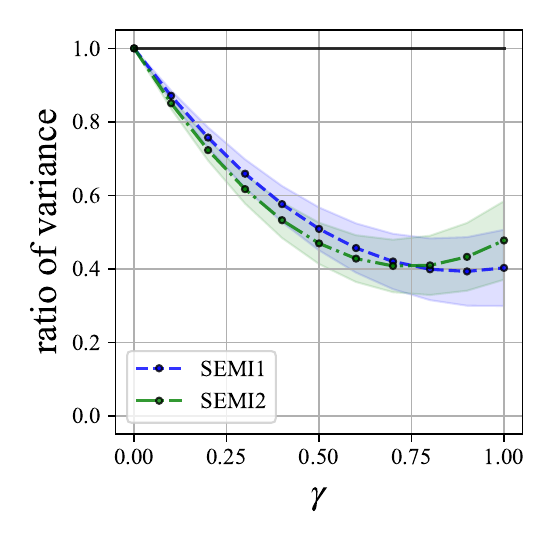}}
    \end{minipage}\hfill
    \begin{minipage}[t]{\subfigwidth}
        \centering
         \subfigure[{\tt kinematics}, AT]{\includegraphics[scale=0.40]{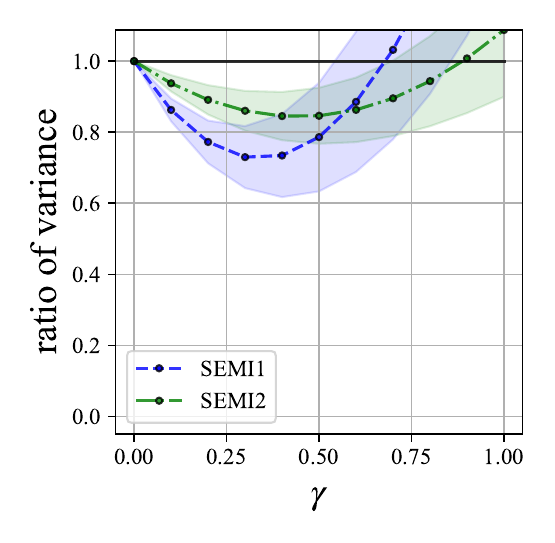}}
    \end{minipage}\hfill
    \begin{minipage}[t]{\subfigwidth}
        \centering
         \subfigure[{\tt lev}, AT]{\includegraphics[scale=0.40]{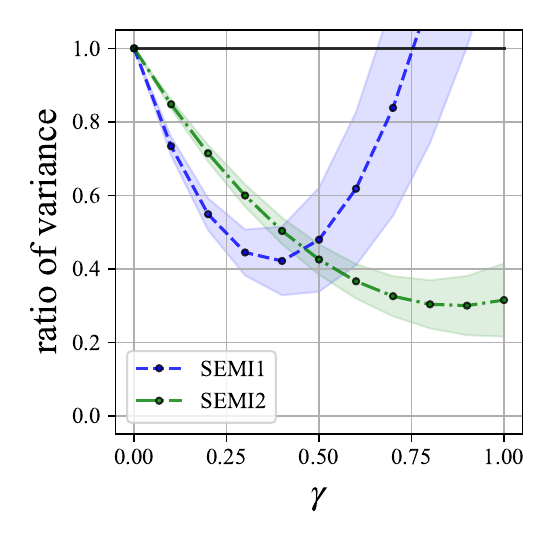}}
    \end{minipage}\hfill
    \begin{minipage}[t]{\subfigwidth}
        \centering
         \subfigure[{\tt swd}, AT]{\includegraphics[scale=0.40]{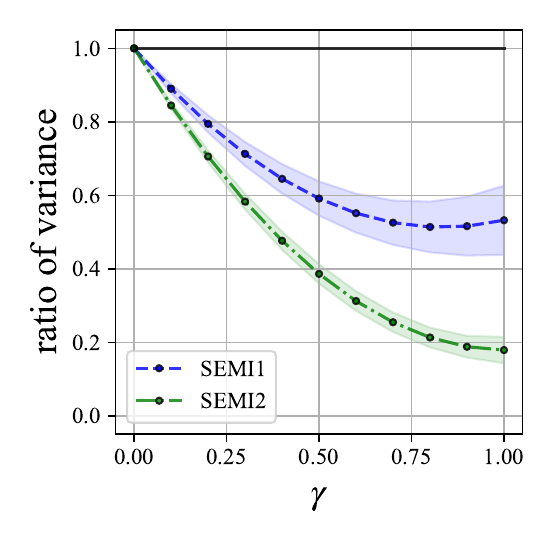}}
    \end{minipage}\hfill
    \begin{minipage}[t]{\subfigwidth}
        \centering
         \subfigure[{\tt toy}, AT]{\includegraphics[scale=0.40]{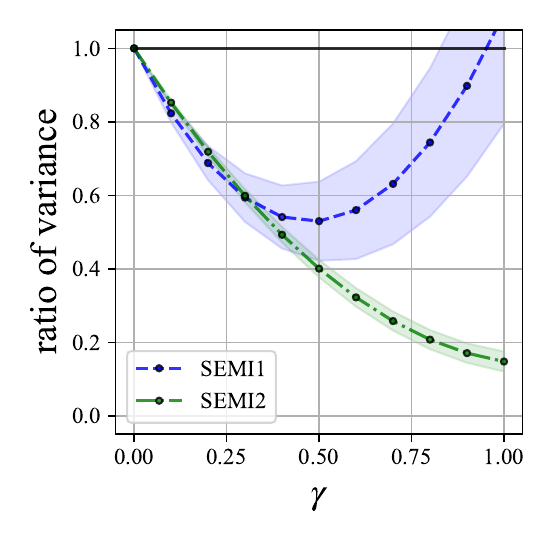}}
    \end{minipage}\hfill
    \end{center}
    \caption{The ratio between the variance of the empirical semi-supervised risk to the supervised risk when we adopted AT as the task surrogate loss and used a linear-in-input model.}
    \label{fig:variance_reduction_AT_linear-in-input}
\end{figure*}

\begin{figure*}[htb]
    \begin{center}
    \setlength{\subfigwidth}{.32\linewidth}
    \addtolength{\subfigwidth}{-.32\subfigcolsep}
    \begin{minipage}[t]{\subfigwidth}
        \centering
         \subfigure[{\tt abalone}, IT]{\includegraphics[scale=0.40]{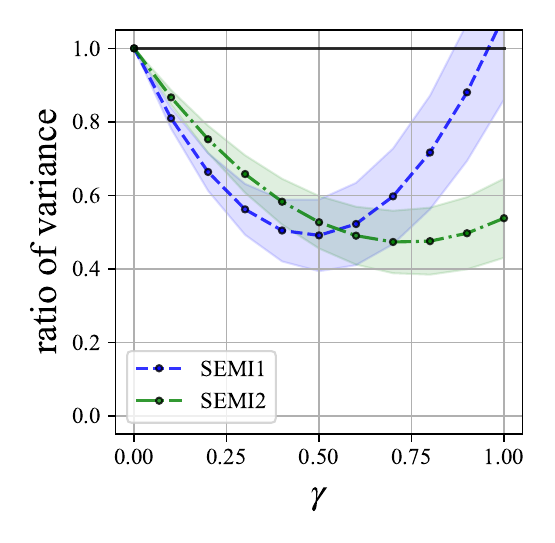}}
    \end{minipage}\hfill
    \begin{minipage}[t]{\subfigwidth}
        \centering
         \subfigure[{\tt bank1-5}, IT]{\includegraphics[scale=0.40]{fig/variance_reduction/sv_IT_bank1-5.pdf}}
    \end{minipage}\hfill
    \begin{minipage}[t]{\subfigwidth}
        \centering
         \subfigure[{\tt bank2-5}, IT]{\includegraphics[scale=0.40]{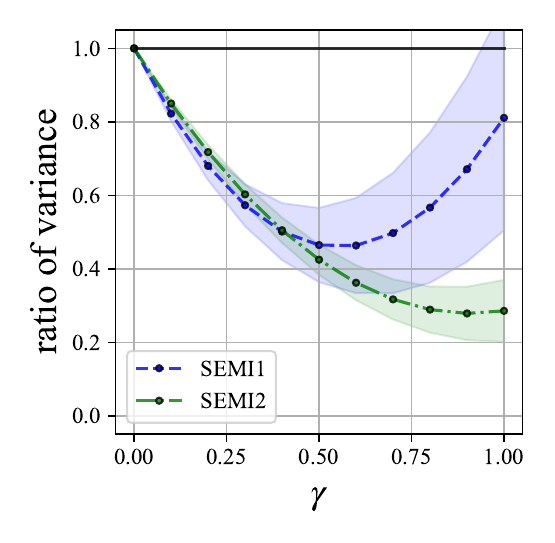}}
    \end{minipage}\hfill
    \begin{minipage}[t]{\subfigwidth}
        \centering
         \subfigure[{\tt census1-5}, IT]{\includegraphics[scale=0.40]{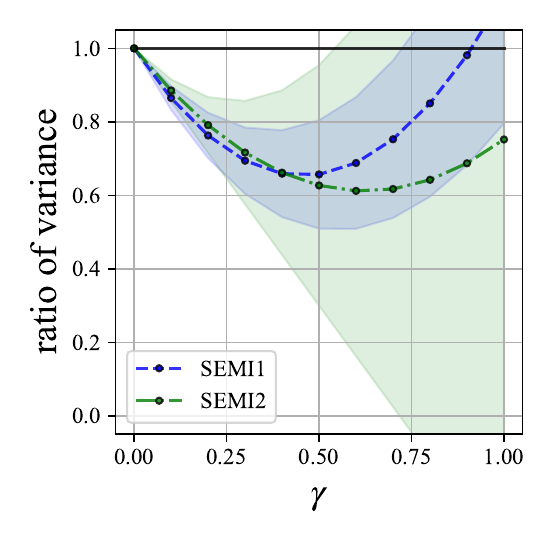}}
    \end{minipage}\hfill
    \begin{minipage}[t]{\subfigwidth}
        \centering
         \subfigure[{\tt census2-5}, IT]{\includegraphics[scale=0.40]{fig/variance_reduction/sv_IT_census2-5.pdf}}
    \end{minipage}\hfill
    \begin{minipage}[t]{\subfigwidth}
        \centering
         \subfigure[{\tt computer1-5}, IT]{\includegraphics[scale=0.40]{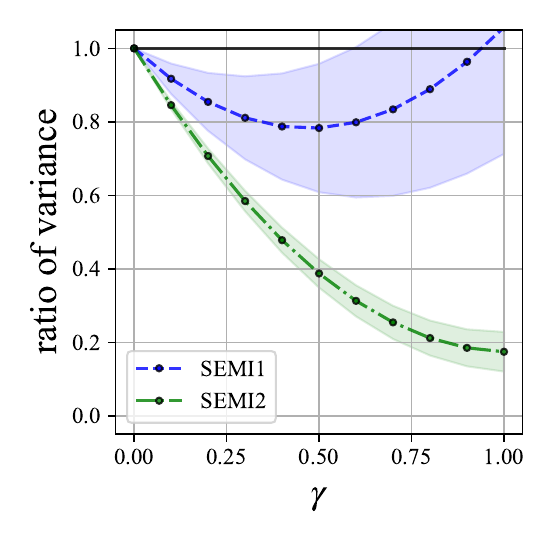}}
    \end{minipage}\hfill
    \begin{minipage}[t]{\subfigwidth}
        \centering
         \subfigure[{\tt computer2-5}, IT]{\includegraphics[scale=0.40]{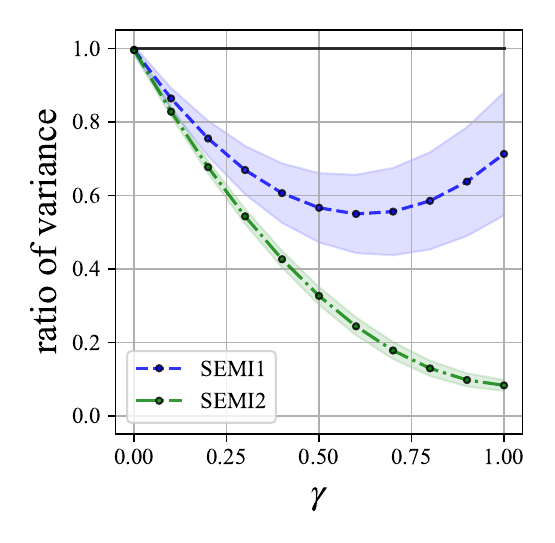}}
    \end{minipage}\hfill
    \begin{minipage}[t]{\subfigwidth}
        \centering
         \subfigure[{\tt fireman-example}, IT]{\includegraphics[scale=0.40]{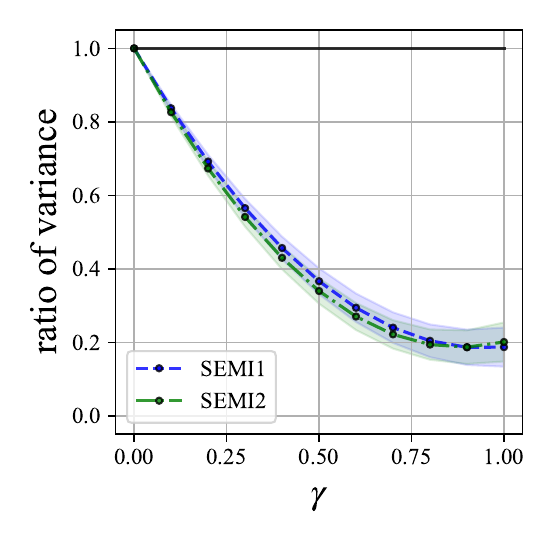}}
    \end{minipage}\hfill
    \begin{minipage}[t]{\subfigwidth}
        \centering
         \subfigure[{\tt kinematics}, IT]{\includegraphics[scale=0.40]{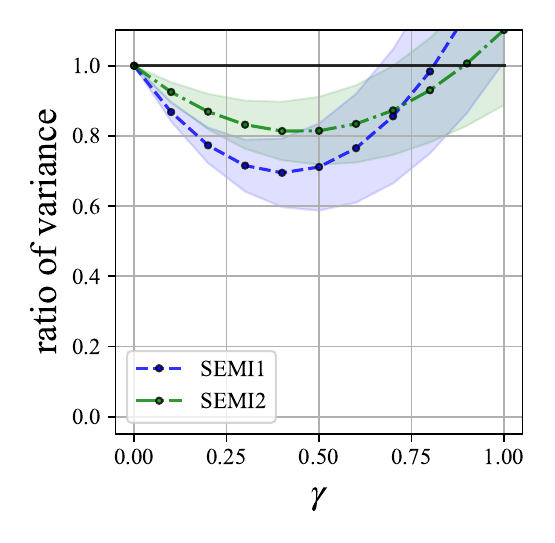}}
    \end{minipage}\hfill
    \begin{minipage}[t]{\subfigwidth}
        \centering
         \subfigure[{\tt lev}, IT]{\includegraphics[scale=0.40]{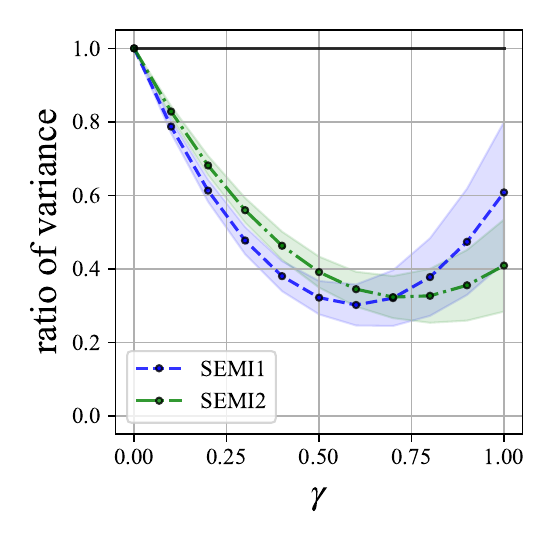}}
    \end{minipage}\hfill
    \begin{minipage}[t]{\subfigwidth}
        \centering
         \subfigure[{\tt swd}, IT]{\includegraphics[scale=0.40]{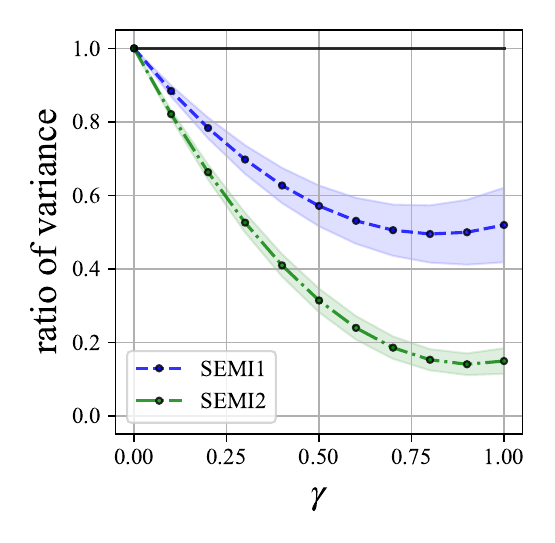}}
    \end{minipage}\hfill
    \begin{minipage}[t]{\subfigwidth}
        \centering
         \subfigure[{\tt toy}, IT]{\includegraphics[scale=0.40]{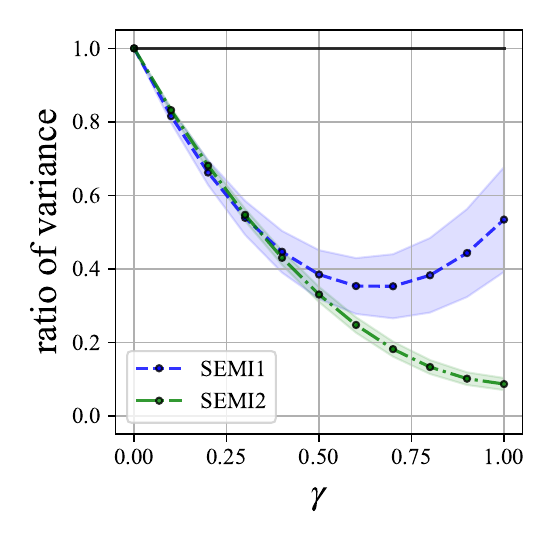}}
    \end{minipage}\hfill
    \end{center}
    \caption{The ratio between the variance of the empirical semi-supervised risk to the supervised risk when we adopted IT as the task surrogate loss and used a linear-in-input model.}
    \label{fig:variance_reduction_IT_linear-in-input}
\end{figure*}

\begin{figure*}[htb]
    \begin{center}
    \setlength{\subfigwidth}{.32\linewidth}
    \addtolength{\subfigwidth}{-.32\subfigcolsep}
    \begin{minipage}[t]{\subfigwidth}
        \centering
         \subfigure[{\tt abalone}, LS]{\includegraphics[scale=0.40]{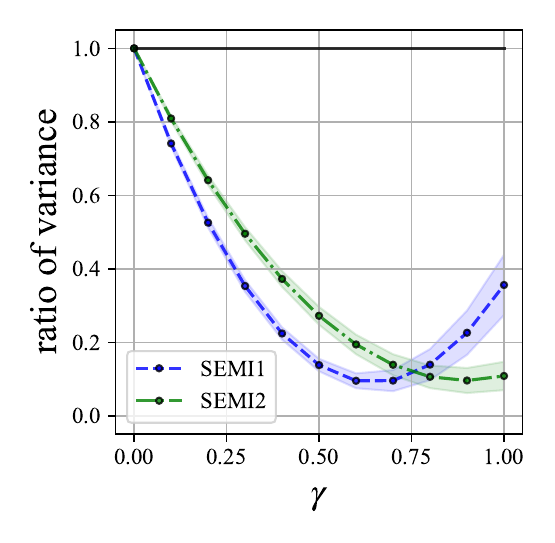}}
    \end{minipage}\hfill
    \begin{minipage}[t]{\subfigwidth}
        \centering
         \subfigure[{\tt bank1-5}, LS]{\includegraphics[scale=0.40]{fig/variance_reduction/sv_LS_bank1-5.pdf}}
    \end{minipage}\hfill
    \begin{minipage}[t]{\subfigwidth}
        \centering
         \subfigure[{\tt bank2-5}, LS]{\includegraphics[scale=0.40]{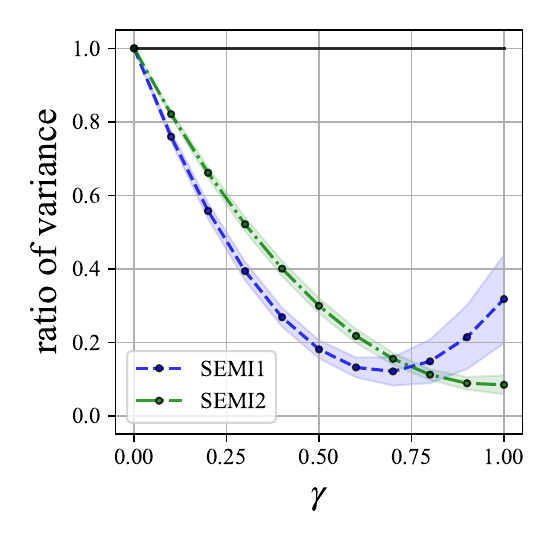}}
    \end{minipage}\hfill
    \begin{minipage}[t]{\subfigwidth}
        \centering
         \subfigure[{\tt census1-5}, LS]{\includegraphics[scale=0.40]{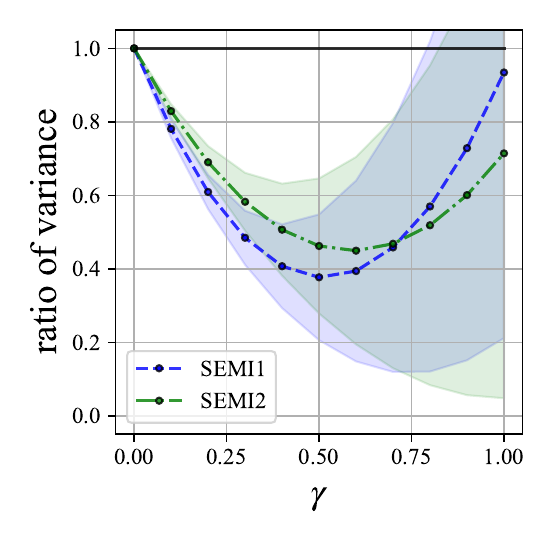}}
    \end{minipage}\hfill
    \begin{minipage}[t]{\subfigwidth}
        \centering
         \subfigure[{\tt census2-5}, LS]{\includegraphics[scale=0.40]{fig/variance_reduction/sv_LS_census2-5.pdf}}
    \end{minipage}\hfill
    \begin{minipage}[t]{\subfigwidth}
        \centering
         \subfigure[{\tt computer1-5}, LS]{\includegraphics[scale=0.40]{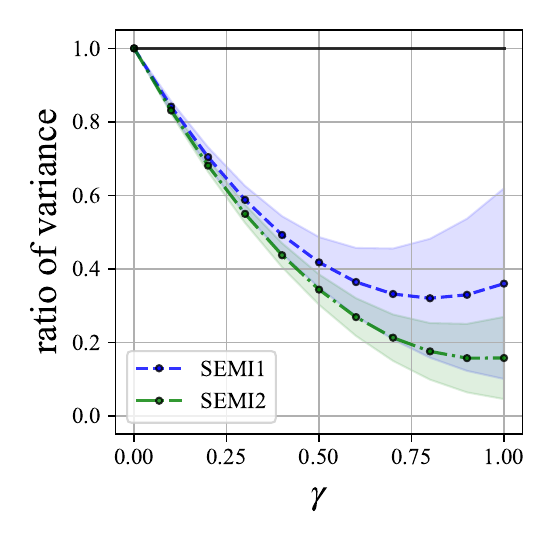}}
    \end{minipage}\hfill
    \begin{minipage}[t]{\subfigwidth}
        \centering
         \subfigure[{\tt computer2-5}, LS]{\includegraphics[scale=0.40]{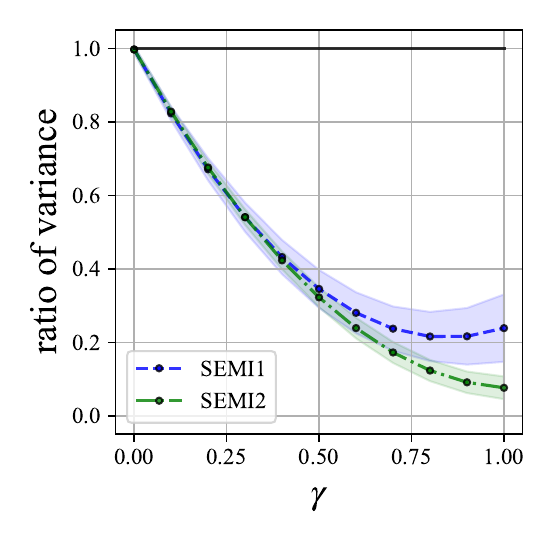}}
    \end{minipage}\hfill
    \begin{minipage}[t]{\subfigwidth}
        \centering
         \subfigure[{\tt fireman-example}, LS]{\includegraphics[scale=0.40]{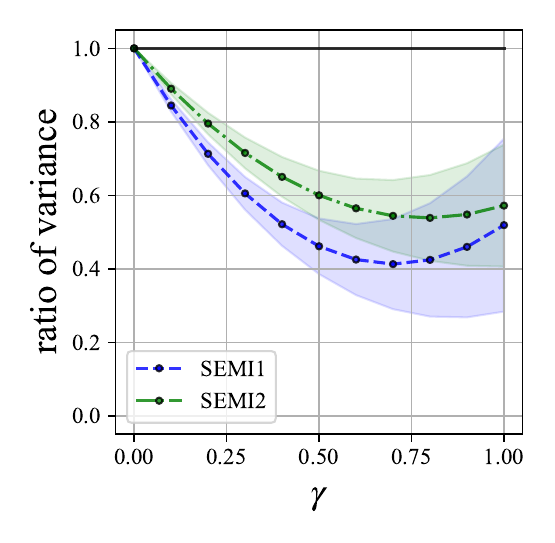}}
    \end{minipage}\hfill
    \begin{minipage}[t]{\subfigwidth}
        \centering
         \subfigure[{\tt kinematics}, LS]{\includegraphics[scale=0.40]{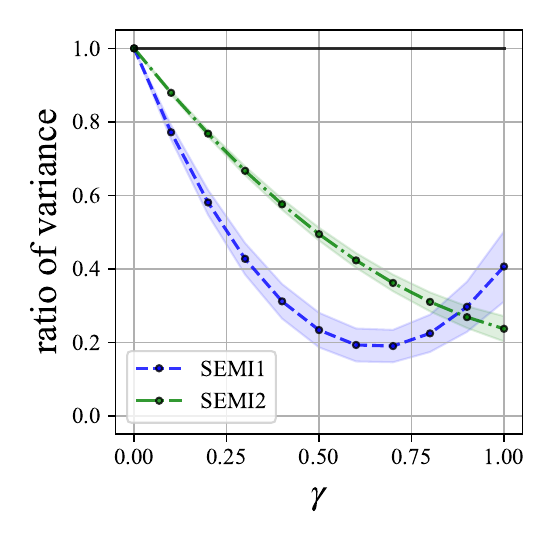}}
    \end{minipage}\hfill
    \begin{minipage}[t]{\subfigwidth}
        \centering
         \subfigure[{\tt lev}, LS]{\includegraphics[scale=0.40]{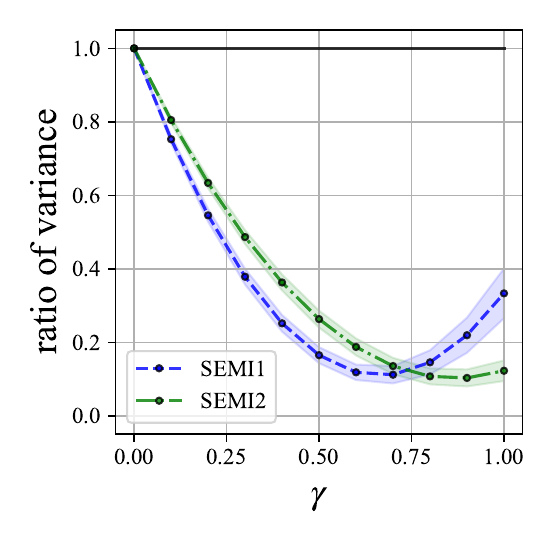}}
    \end{minipage}\hfill
    \begin{minipage}[t]{\subfigwidth}
        \centering
         \subfigure[{\tt swd}, LS]{\includegraphics[scale=0.40]{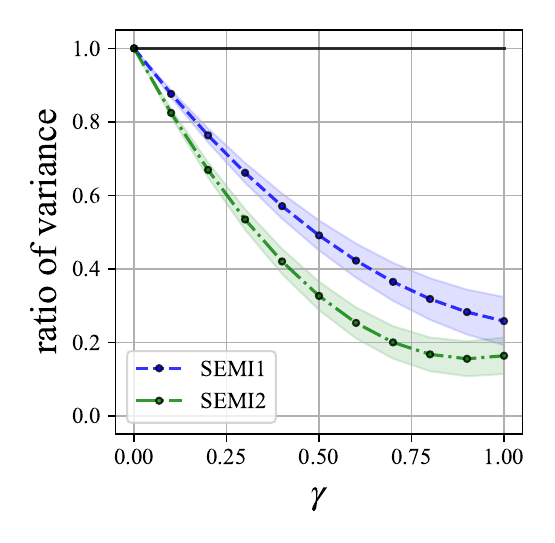}}
    \end{minipage}\hfill
    \begin{minipage}[t]{\subfigwidth}
        \centering
         \subfigure[{\tt toy}, LS]{\includegraphics[scale=0.40]{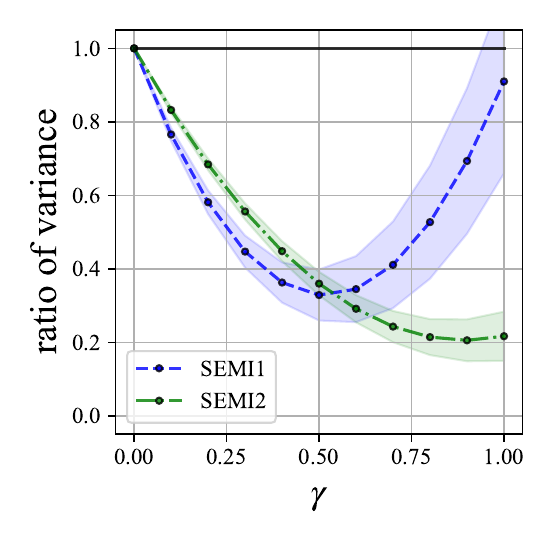}}
    \end{minipage}\hfill
    \end{center}
    \caption{The ratio between the variance of the empirical semi-supervised risk to the supervised risk when we adopted LS as the task surrogate loss and used a linear-in-input model.}
    \label{fig:variance_reduction_LS_linear-in-input}
\end{figure*}


\begin{figure*}[htb]
    \begin{center}
    \setlength{\subfigwidth}{.30\linewidth}
    \addtolength{\subfigwidth}{-.30\subfigcolsep}
    \begin{minipage}[t]{\subfigwidth}
        \centering
         \subfigure[{\tt abalone}, AT]{\includegraphics[scale=0.32]{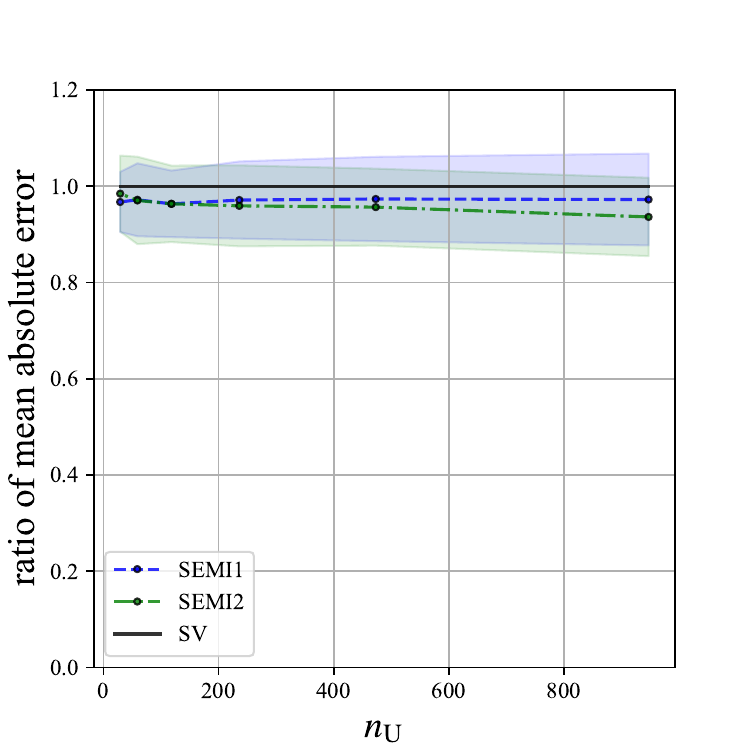}}
    \end{minipage}\hfill
    \begin{minipage}[t]{\subfigwidth}
        \centering
         \subfigure[{\tt bank1-5}, AT]{\includegraphics[scale=0.32]{fig/nu_change/bank1-5_AT.pdf}}
    \end{minipage}\hfill
    \begin{minipage}[t]{\subfigwidth}
        \centering
         \subfigure[{\tt bank2-5}, AT]{\includegraphics[scale=0.32]{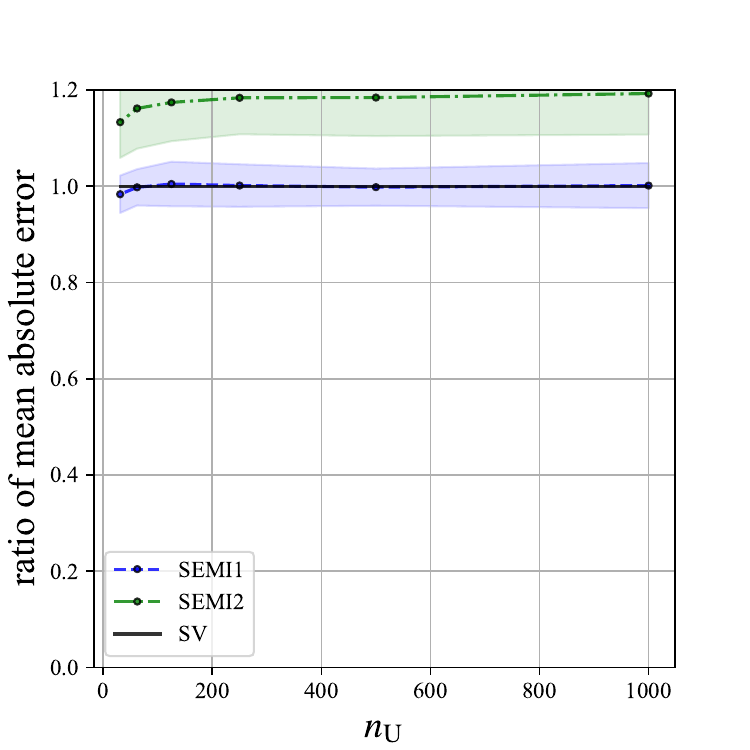}}
    \end{minipage}\hfill
    \vspace{-7pt}  
    \begin{minipage}[t]{\subfigwidth}
        \centering
         \subfigure[{\tt census1-5}, AT]{\includegraphics[scale=0.32]{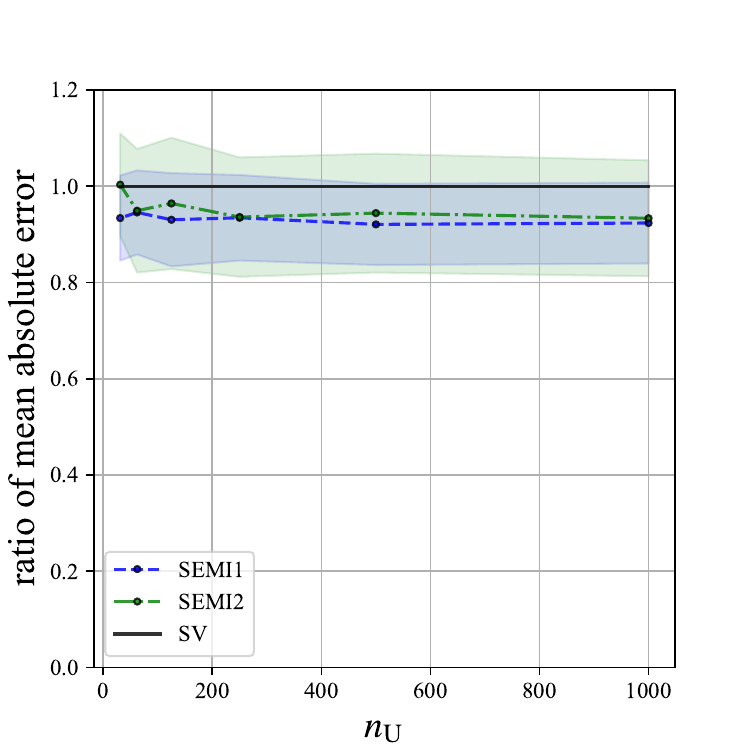}}
    \end{minipage}\hfill
    \begin{minipage}[t]{\subfigwidth}
        \centering
         \subfigure[{\tt census2-5}, AT]{\includegraphics[scale=0.32]{fig/nu_change/census2-5_AT.pdf}}
    \end{minipage}\hfill
    \begin{minipage}[t]{\subfigwidth}
        \centering
         \subfigure[{\tt computer1-5}, AT]{\includegraphics[scale=0.32]{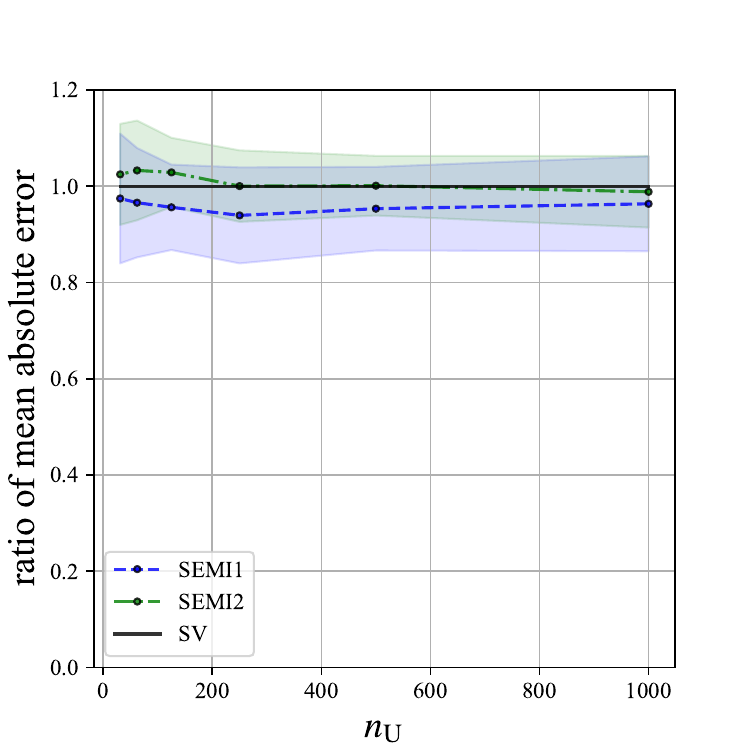}}
    \end{minipage}\hfill
    \vspace{-7pt}  
    \begin{minipage}[t]{\subfigwidth}
        \centering
         \subfigure[{\tt computer2-5}, AT]{\includegraphics[scale=0.32]{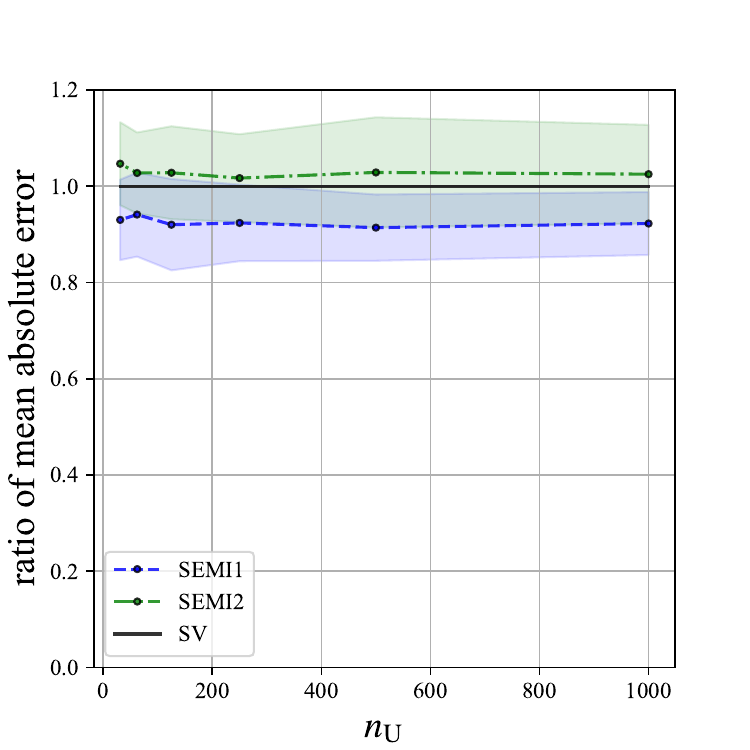}}
    \end{minipage}\hfill
    \begin{minipage}[t]{\subfigwidth}
        \centering
         \subfigure[{\tt fireman-example}, AT]{\includegraphics[scale=0.32]{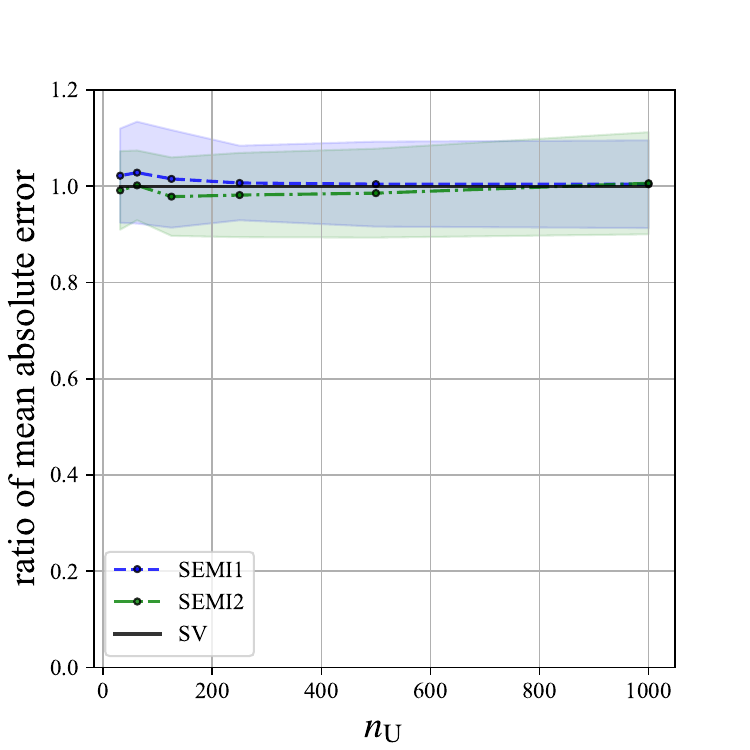}}
    \end{minipage}\hfill
    \begin{minipage}[t]{\subfigwidth}
        \centering
         \subfigure[{\tt kinematics}, AT]{\includegraphics[scale=0.32]{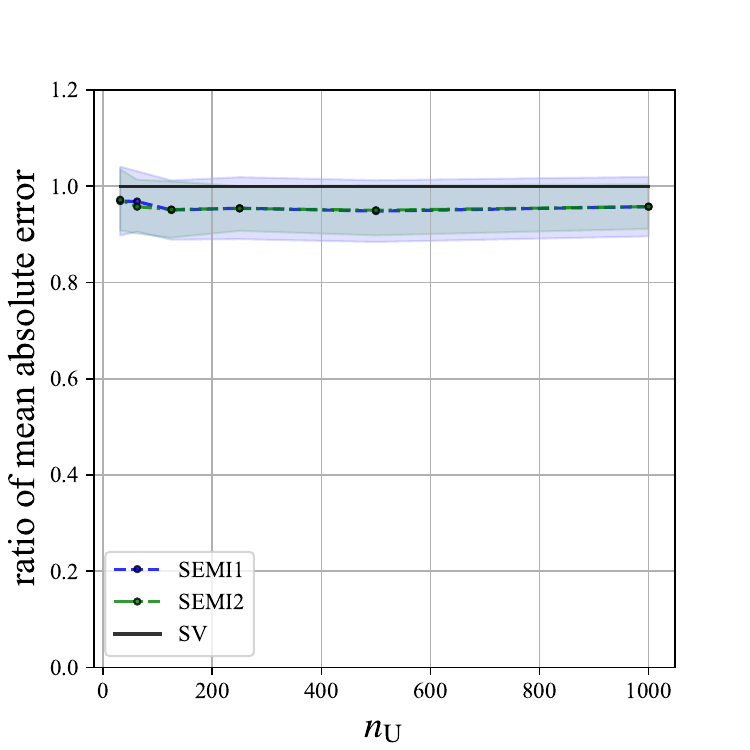}}
    \end{minipage}\hfill
    \vspace{-7pt}  
    \begin{minipage}[t]{\subfigwidth}
        \centering
         \subfigure[{\tt lev}, AT]{\includegraphics[scale=0.32]{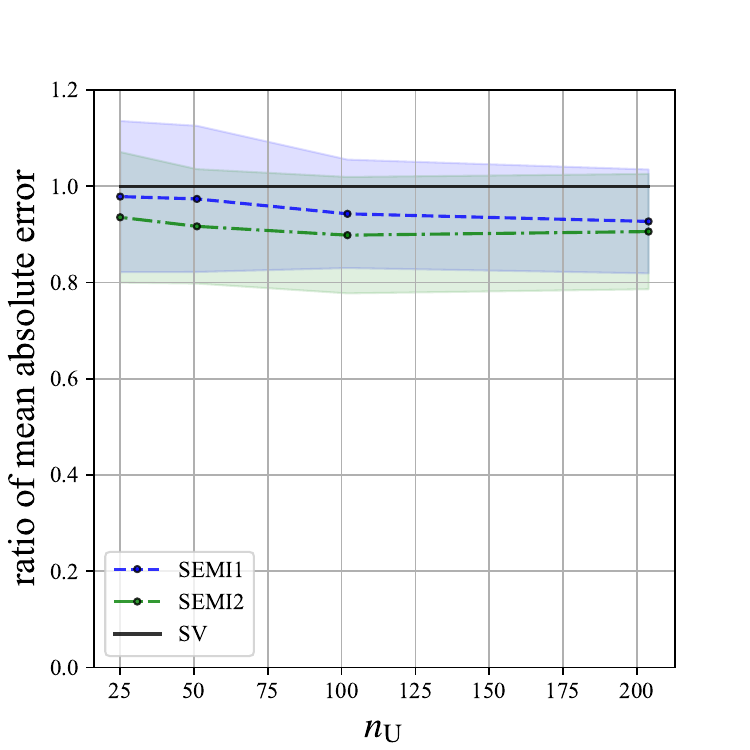}}
    \end{minipage}\hspace{50pt}
    \begin{minipage}[t]{\subfigwidth}
        \centering
         \subfigure[{\tt swd}, AT]{\includegraphics[scale=0.32]{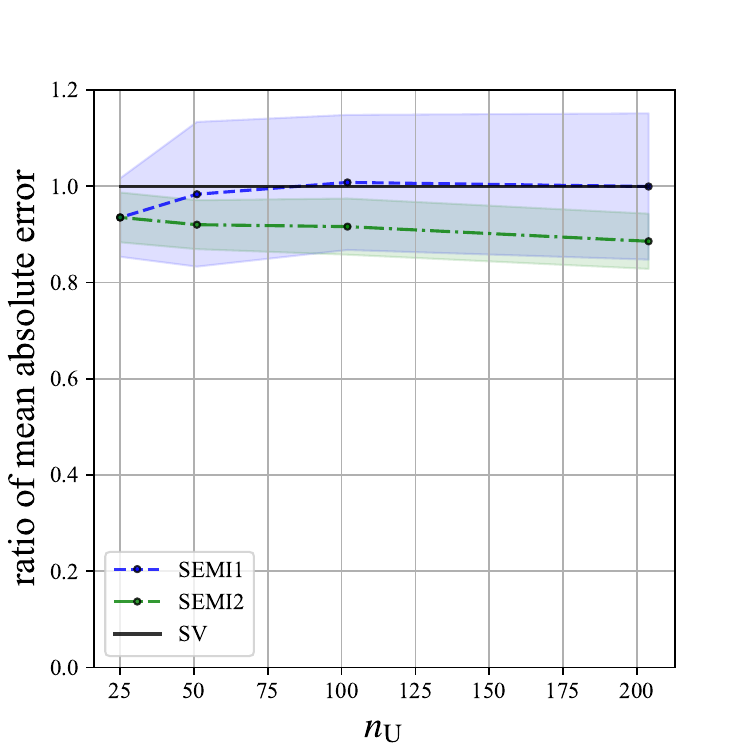}}
    \end{minipage}
    \end{center}
    \caption{
    The ratio between the error of the empirical semi-supervised risk to the supervised risk when we adopted AT as the task surrogate loss and used linear-in-input as a model.
    The ratio below 1.0 implies that the error of the semi-supervised method is better than that of the supervised risk.
    }
    \label{fig:nu_change_AT_linear-in-input}
\end{figure*}

\begin{figure*}[htb]
    \begin{center}
    \setlength{\subfigwidth}{.30\linewidth}
    \addtolength{\subfigwidth}{-.30\subfigcolsep}
    \begin{minipage}[t]{\subfigwidth}
        \centering
         \subfigure[{\tt abalone}, IT]{\includegraphics[scale=0.32]{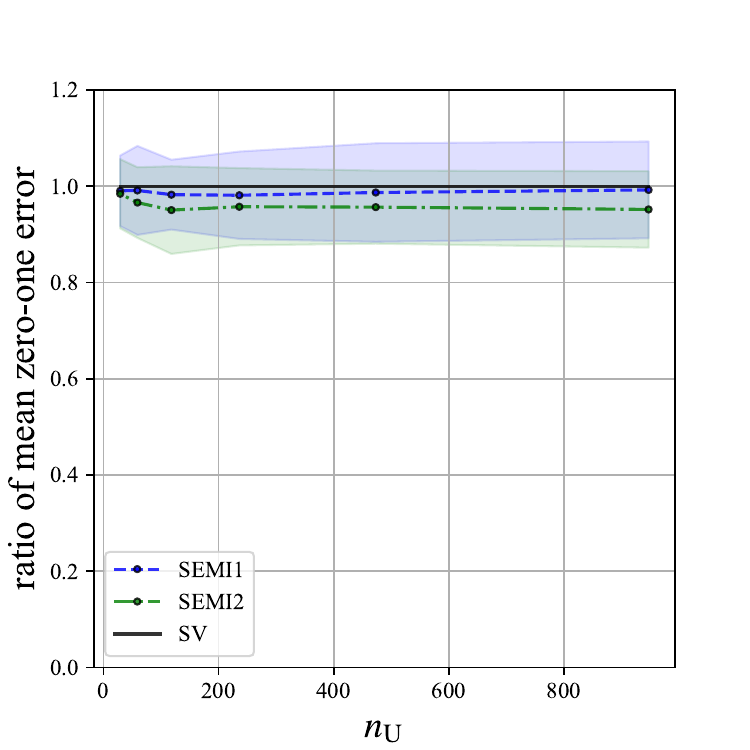}}
    \end{minipage}\hfill
    \begin{minipage}[t]{\subfigwidth}
        \centering
         \subfigure[{\tt bank1-5}, IT]{\includegraphics[scale=0.32]{fig/nu_change/bank1-5_IT.pdf}}
    \end{minipage}\hfill
    \begin{minipage}[t]{\subfigwidth}
        \centering
         \subfigure[{\tt bank2-5}, IT]{\includegraphics[scale=0.32]{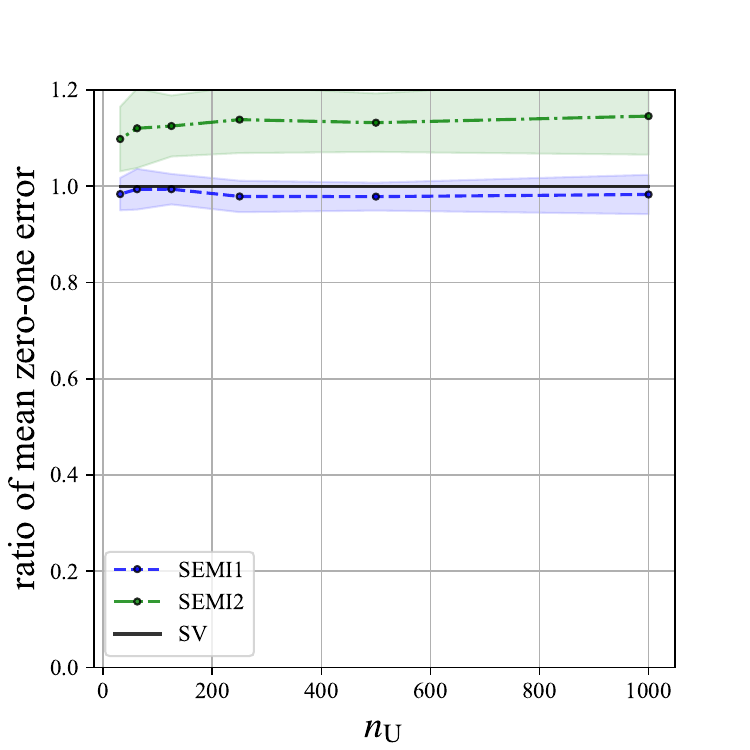}}
    \end{minipage}\hfill
    \vspace{-7pt}  
    \begin{minipage}[t]{\subfigwidth}
        \centering
         \subfigure[{\tt census1-5}, IT]{\includegraphics[scale=0.32]{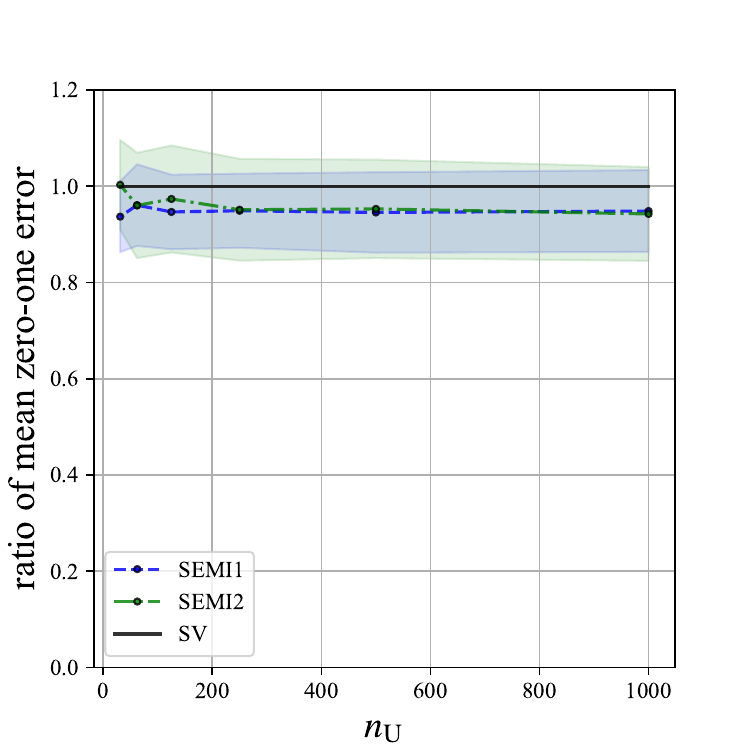}}
    \end{minipage}\hfill
    \begin{minipage}[t]{\subfigwidth}
        \centering
         \subfigure[{\tt census2-5}, IT]{\includegraphics[scale=0.32]{fig/nu_change/census2-5_IT.pdf}}
    \end{minipage}\hfill
    \begin{minipage}[t]{\subfigwidth}
        \centering
         \subfigure[{\tt computer1-5}, IT]{\includegraphics[scale=0.32]{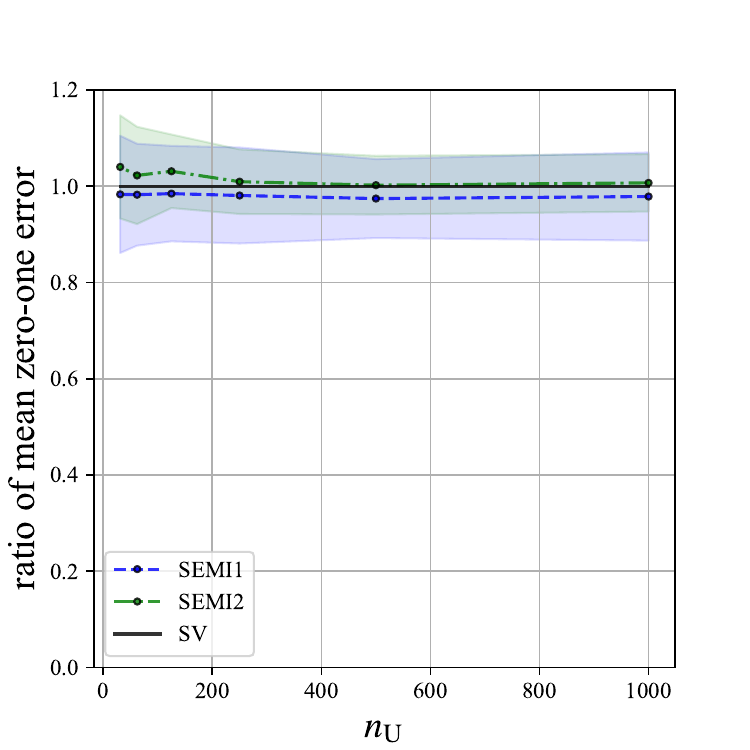}}
    \end{minipage}\hfill
    \vspace{-7pt}  
    \begin{minipage}[t]{\subfigwidth}
        \centering
         \subfigure[{\tt computer2-5}, IT]{\includegraphics[scale=0.32]{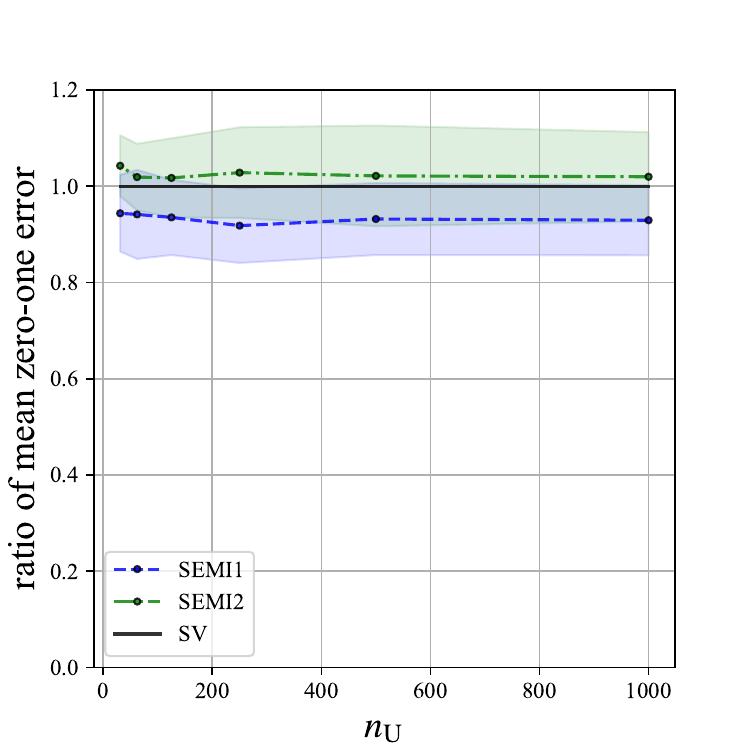}}
    \end{minipage}\hfill
    \begin{minipage}[t]{\subfigwidth}
        \centering
         \subfigure[{\tt fireman-example}, IT]{\includegraphics[scale=0.32]{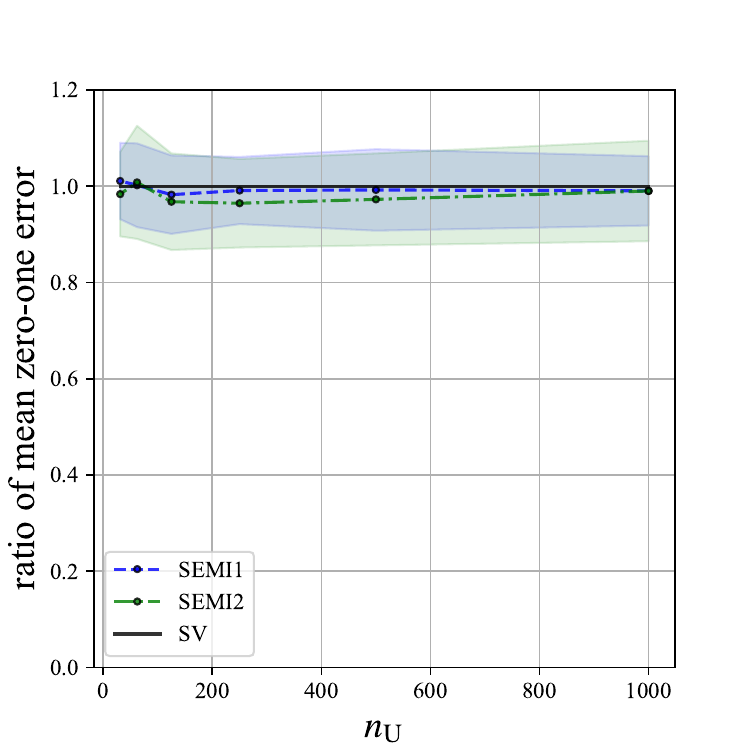}}
    \end{minipage}\hfill
    \begin{minipage}[t]{\subfigwidth}
        \centering
         \subfigure[{\tt kinematics}, IT]{\includegraphics[scale=0.32]{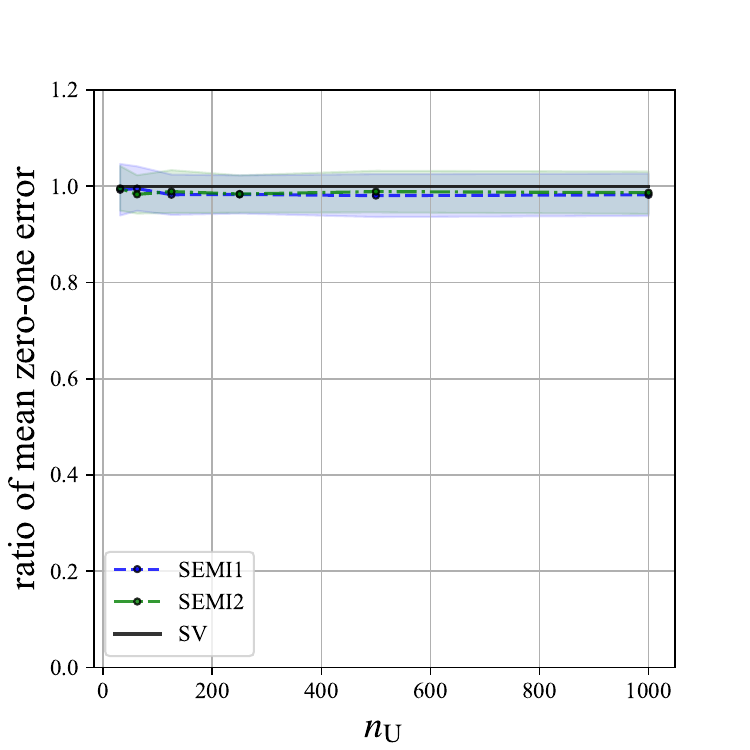}}
    \end{minipage}\hfill
    \vspace{-7pt}  
    \begin{minipage}[t]{\subfigwidth}
        \centering
         \subfigure[{\tt lev}, IT]{\includegraphics[scale=0.32]{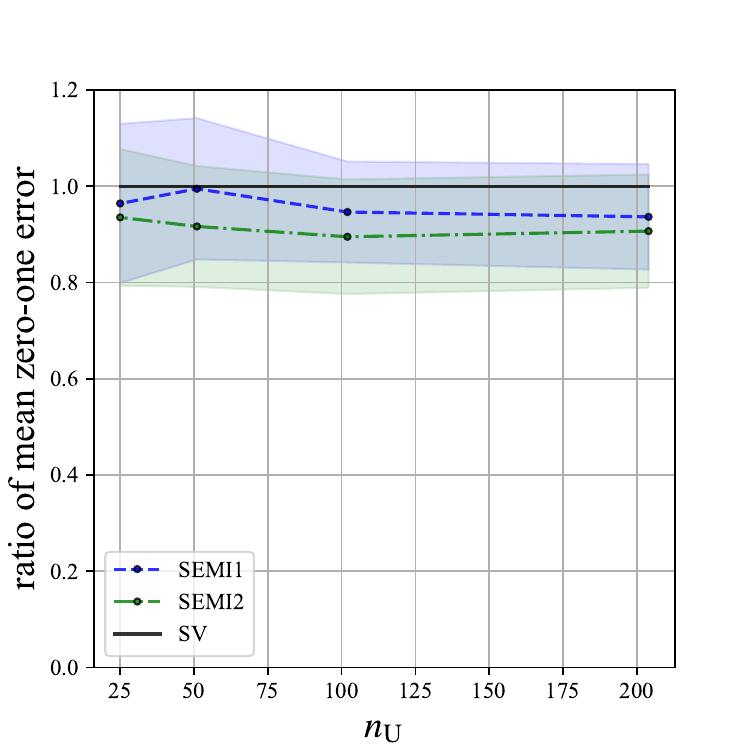}}
    \end{minipage}\hspace{50pt}
    \begin{minipage}[t]{\subfigwidth}
        \centering
         \subfigure[{\tt swd}, IT]{\includegraphics[scale=0.32]{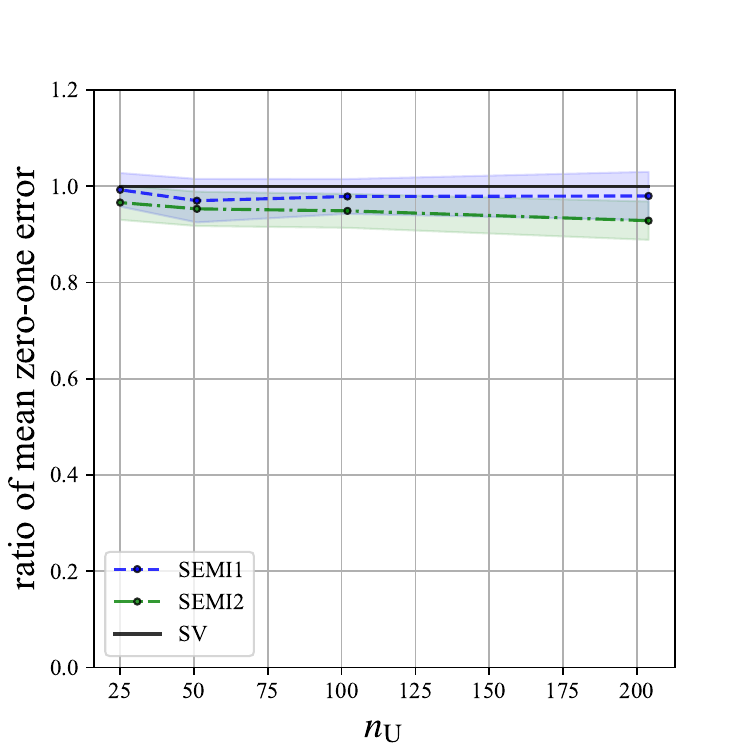}}
    \end{minipage}
    \end{center}
    \caption{
    The ratio between the error of the empirical semi-supervised risk to the supervised risk when we adopted IT as the task surrogate loss and used a linear-in-input model.
    The ratio below 1.0 implies that the error of the semi-supervised method is better than that of the supervised risk.
    }
    \label{fig:nu_change_IT_linear-in-input}
\end{figure*}

\begin{figure*}[htb]
    \begin{center}
    \setlength{\subfigwidth}{.30\linewidth}
    \addtolength{\subfigwidth}{-.30\subfigcolsep}
    \begin{minipage}[t]{\subfigwidth}
        \centering
         \subfigure[{\tt abalone}, LS]{\includegraphics[scale=0.32]{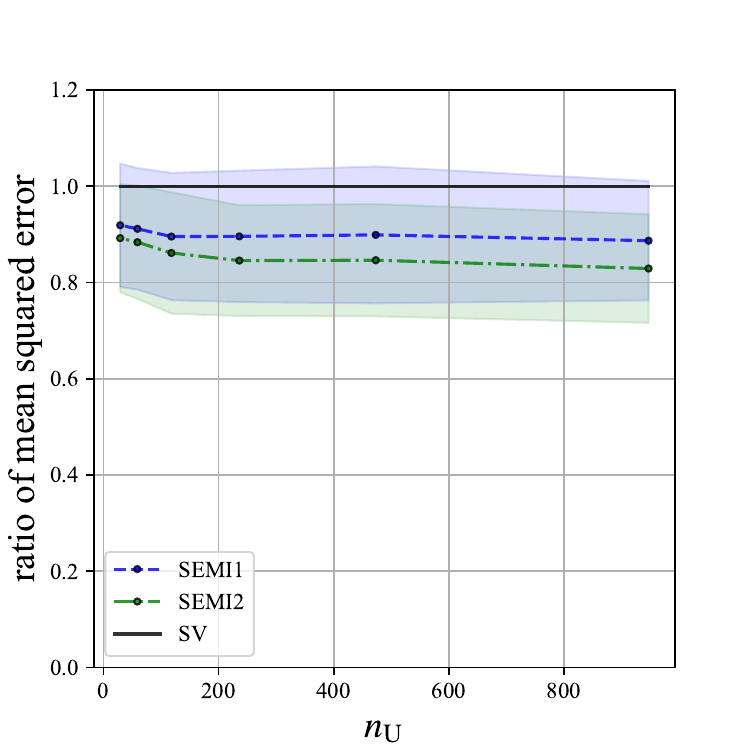}}
    \end{minipage}\hfill
    \begin{minipage}[t]{\subfigwidth}
        \centering
         \subfigure[{\tt bank1-5}, LS]{\includegraphics[scale=0.32]{fig/nu_change/bank1-5_LS.pdf}}
    \end{minipage}\hfill
    \begin{minipage}[t]{\subfigwidth}
        \centering
         \subfigure[{\tt bank2-5}, LS]{\includegraphics[scale=0.32]{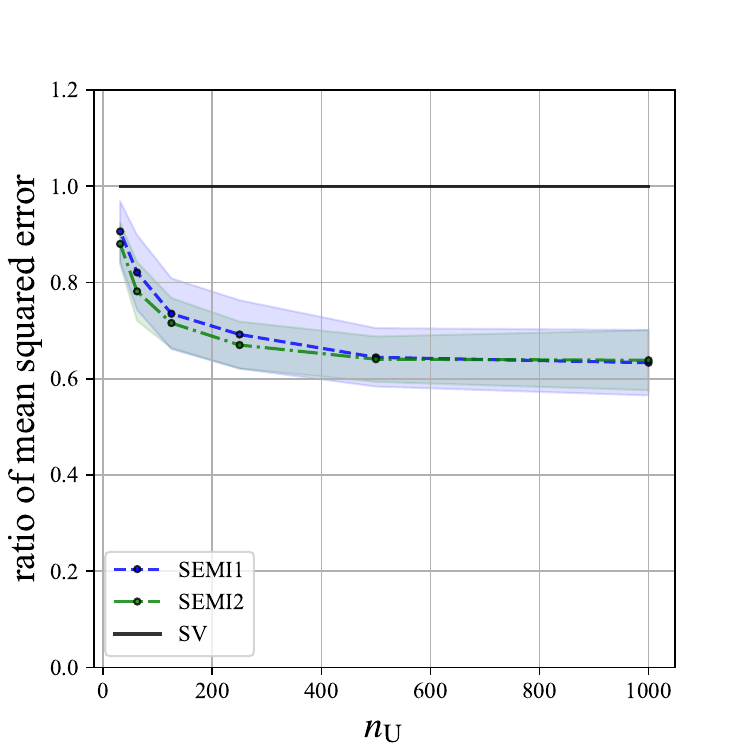}}
    \end{minipage}\hfill
    \vspace{-7pt}  
    \begin{minipage}[t]{\subfigwidth}
        \centering
         \subfigure[{\tt census1-5}, LS]{\includegraphics[scale=0.32]{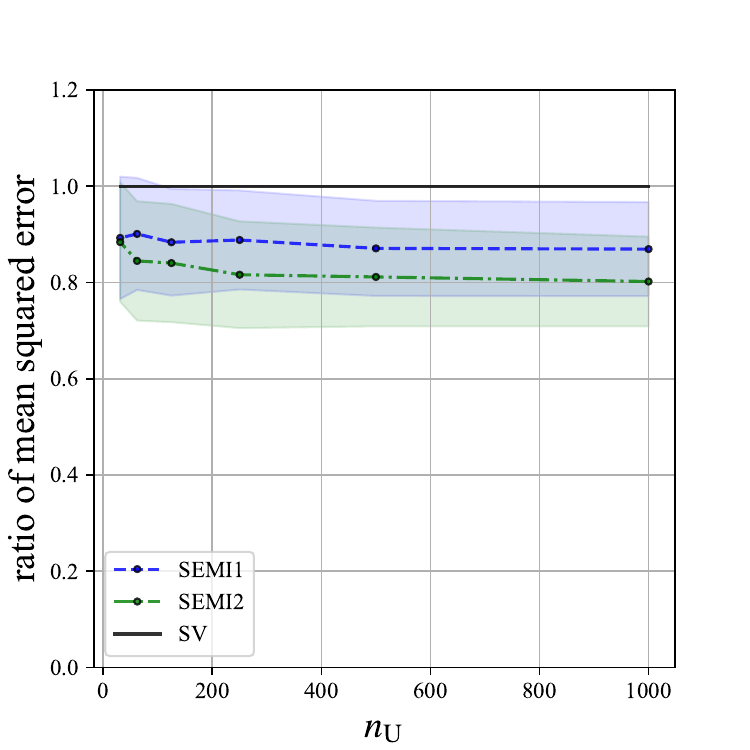}}
    \end{minipage}\hfill
    \begin{minipage}[t]{\subfigwidth}
        \centering
         \subfigure[{\tt census2-5}, LS]{\includegraphics[scale=0.32]{fig/nu_change/census2-5_LS.pdf}}
    \end{minipage}\hfill
    \begin{minipage}[t]{\subfigwidth}
        \centering
         \subfigure[{\tt computer1-5}, LS]{\includegraphics[scale=0.32]{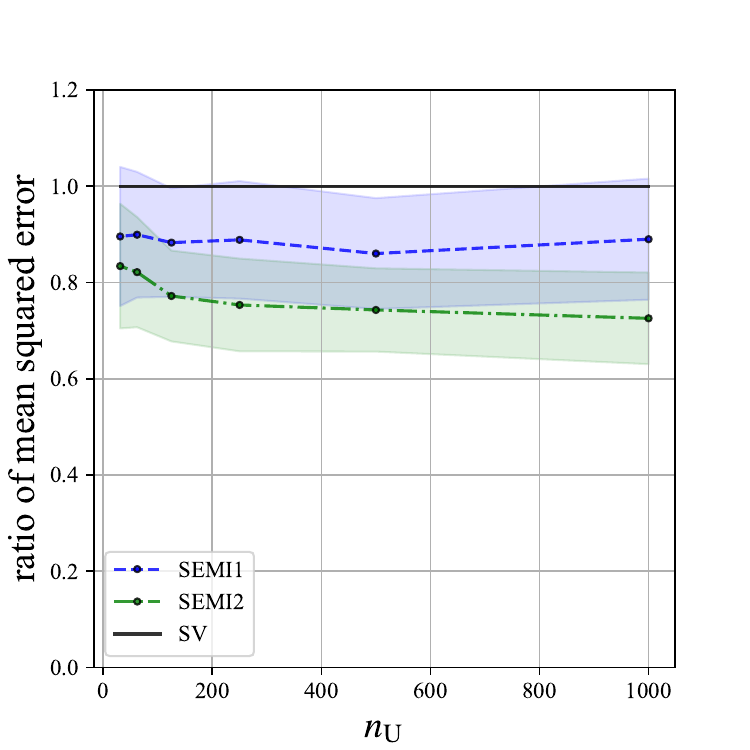}}
    \end{minipage}\hfill
    \vspace{-7pt}  
    \begin{minipage}[t]{\subfigwidth}
        \centering
         \subfigure[{\tt computer2-5}, LS]{\includegraphics[scale=0.32]{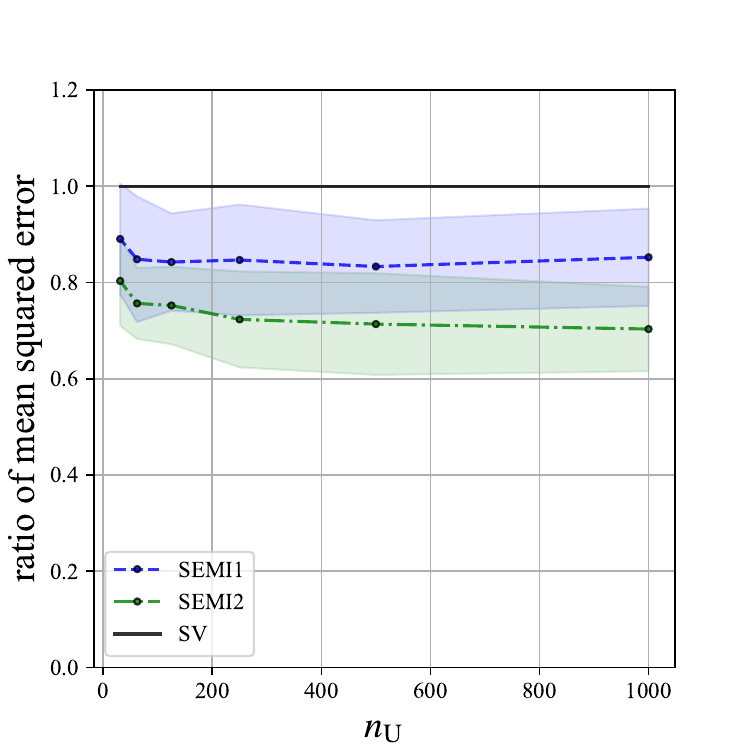}}
    \end{minipage}\hfill
    \begin{minipage}[t]{\subfigwidth}
        \centering
         \subfigure[{\tt fireman-example}, LS]{\includegraphics[scale=0.32]{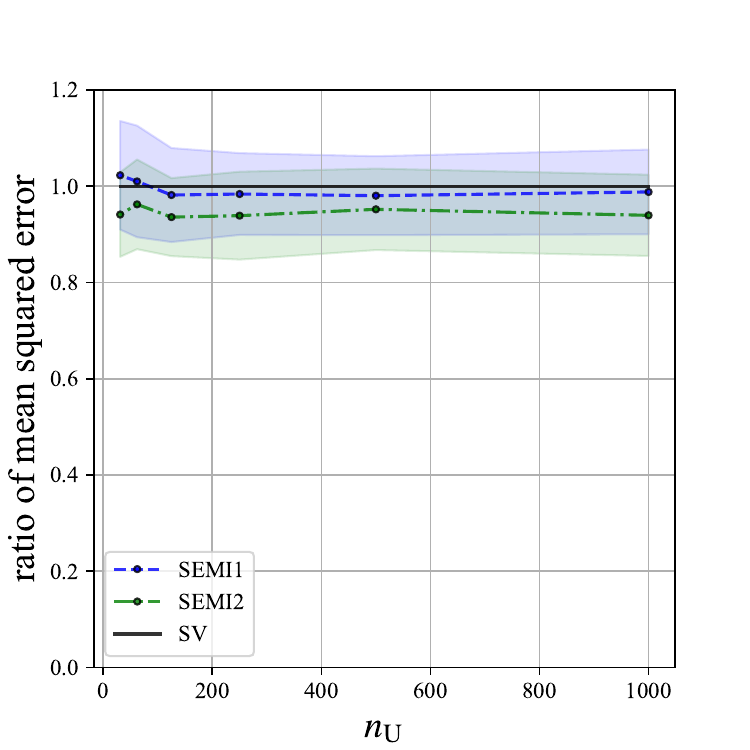}}
    \end{minipage}\hfill
    \begin{minipage}[t]{\subfigwidth}
        \centering
         \subfigure[{\tt kinematics}, LS]{\includegraphics[scale=0.32]{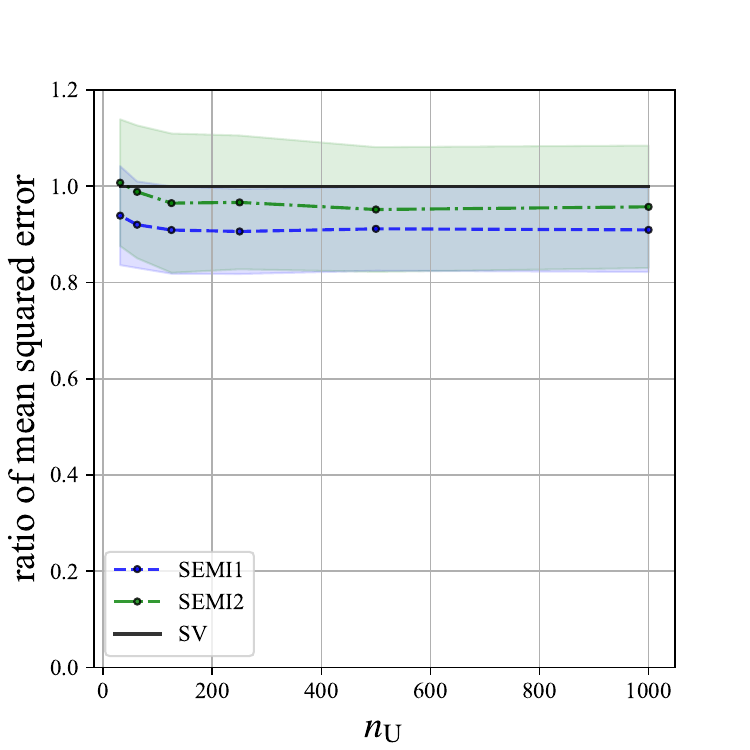}}
    \end{minipage}\hfill
    \vspace{-7pt}  
    \begin{minipage}[t]{\subfigwidth}
        \centering
         \subfigure[{\tt lev}, LS]{\includegraphics[scale=0.32]{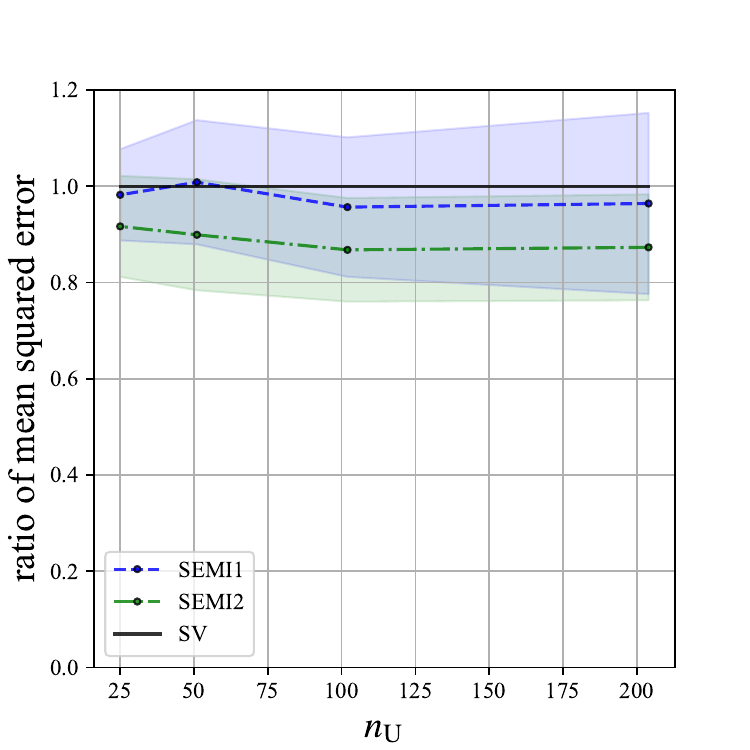}}
    \end{minipage}\hspace{50pt}
    \begin{minipage}[t]{\subfigwidth}
        \centering
         \subfigure[{\tt swd}, LS]{\includegraphics[scale=0.32]{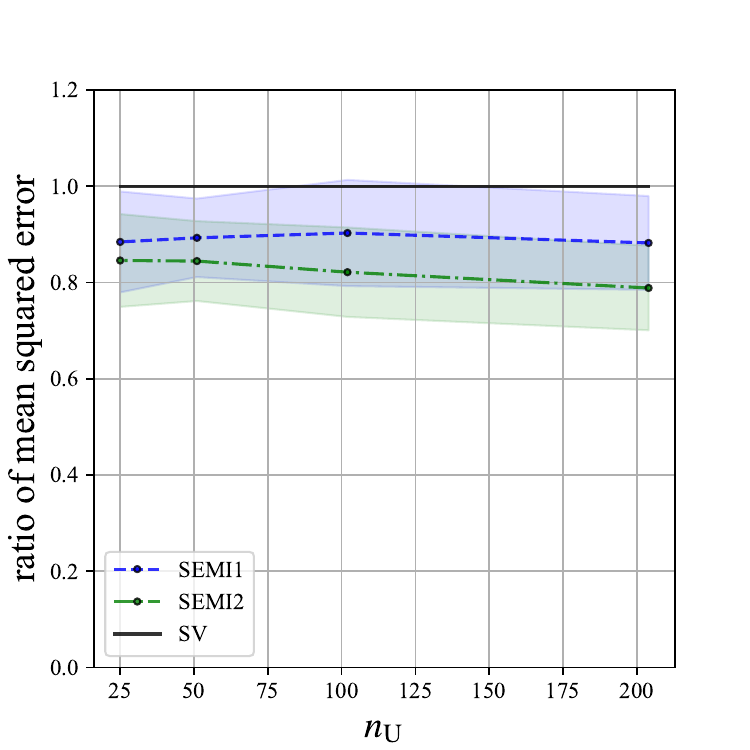}}
    \end{minipage}
    \end{center}
    \caption{
    The ratio between the error of the empirical semi-supervised risk to the supervised risk when we adopted LS as the task surrogate loss and used a linear-in-input model.
    The ratio below 1.0 implies that the error of the semi-supervised method is better than that of the supervised risk.
    }
    \label{fig:nu_change_LS_linear-in-input}
\end{figure*}

\end{document}